\documentclass{article}

\usepackage[final]{neurips_2022}

\usepackage[utf8]{inputenc} %
\usepackage[T1]{fontenc}    %
\usepackage{hyperref}       %
\usepackage{url}            %
\usepackage{booktabs}       %
\usepackage{amsfonts}       %
\usepackage{nicefrac}       %
\usepackage{microtype}      %
\usepackage{xcolor}         %

\usepackage{mathtools} %
\usepackage{booktabs} %
\usepackage{tikz} %

\usepackage{amsmath,amsfonts,bm}

\def\eqref#1{equation~\ref{#1}}
\def\Eqref#1{Equation~\ref{#1}}

\def\1{\bm{1}}

\DeclareMathAlphabet{\mathsfit}{\encodingdefault}{\sfdefault}{m}{sl}
\SetMathAlphabet{\mathsfit}{bold}{\encodingdefault}{\sfdefault}{bx}{n}

\DeclareMathOperator*{\argmax}{arg\,max}

\usepackage{hyperref}
\usepackage{url}

\hypersetup{
    colorlinks=true,
    citecolor=cyan,
    linkcolor=blue,
    filecolor=magenta,      
    urlcolor=cyan
    }

\usepackage{amsmath}        %
\usepackage{amssymb}   
\usepackage{bbm}            %
\usepackage{graphicx}       %
\usepackage{xcolor}         %
\usepackage{subfigure}      %
\usepackage{overpic}

\usepackage[ruled,linesnumbered]{algorithm2e}
\makeatletter
\newcommand{\algorithmstyle}[1]{\renewcommand{\algocf@style}{#1}}
\newcommand{\removelatexerror}{\let\@latex@error\@gobble}
\makeatother

\SetCommentSty{mycommfont}
\SetAlFnt{\footnotesize}
\SetAlCapFnt{\footnotesize}
\SetAlCapNameFnt{\footnotesize}

\usepackage{multirow}
\usepackage{array}

\definecolor{myblue}{HTML}{5576D1}
\definecolor{myred}{HTML}{FF0000}
\definecolor{local}{HTML}{FFa822}
\definecolor{elite}{HTML}{1ac0c6}
\definecolor{sample}{HTML}{134e6f}

\usepackage{wrapfig}

\usetikzlibrary{shapes}
\newcommand{\markerblue}{\raisebox{0pt}{\tikz{\node[scale=0.5,circle,fill=myblue](){};}}}
\newcommand{\markerblack}{\raisebox{0pt}{\tikz{\node[scale=0.5,circle,fill=black](){};}}}
\newcommand{\markeropenblack}{\raisebox{0pt}{\tikz{\node[draw=black,very thick, scale=0.5,circle,fill=white](){};}}}
\newcommand{\markerlocal}{\raisebox{0pt}{\tikz{\node[scale=0.5,isosceles triangle,rotate=90,fill=local](){};}}}
    
\newcommand{\markerelite}{\raisebox{0pt}{\tikz{\node[scale=0.5,circle,fill=elite](){};}}}
\newcommand{\markersample}{\raisebox{0pt}{\tikz{\node[scale=0.5,circle,fill=sample](){};}}}

\usetikzlibrary{shapes.misc}
\tikzset{cross/.style={cross out, draw=elite, minimum size=2*(#1-\pgflinewidth), inner sep=0pt, outer sep=0pt},
cross/.default={1pt}}
\newcommand{\markerstart}{\raisebox{0pt}{\tikz{\node[scale=3.0,cross](){};}}}

\usepackage{pgfplots}
\pgfplotsset{compat=1.16}
\tikzset{
    cross/.pic = {
    \draw[rotate = 45] (-#1,0) -- (#1,0);
    \draw[rotate = 45] (0,-#1) -- (0, #1);
    }
}
\newcommand{\markerglobal}{\raisebox{0pt}{\tikz{ \draw (0,0) pic[rotate=45,red] {cross=2pt} ;}}}

\usepackage{amsthm}

\newtheorem{definition}{Definition}[section]

\newtheorem{assumption}{Assumption}

\usepackage{thmtools, thm-restate}

\title{A Simple Decentralized Cross-Entropy Method}

\author{%
  Zichen Zhang$^{1,}$\thanks{Work partially done during Zichen's internship at Huawei Noah's Ark Lab.} 
  \hspace{1mm}
   Jun Jin$^{2}$
  \hspace{1mm}
  Martin Jagersand$^{1}$
  \hspace{1mm}
   Jun Luo$^{2,\dag}$
  \hspace{1mm}
  Dale Schuurmans$^{1,\dag}$\\
\\
\hspace{-5mm}
$^{1}${University of Alberta}
  \hspace{1mm}
$^{2}$Huawei Noah's Ark Lab
  \hspace{1mm}
$^\dag$equal advising \\
\hspace{-2mm}\texttt{ \{zichen2,mj7,daes\}@ualberta.ca}
\texttt{ \{jun.jin1,jun.luo1\}@huawei.com}
}

\begin{document}

\maketitle

\begin{abstract}
 Cross-Entropy Method (CEM) is commonly used for planning in model-based reinforcement learning (MBRL) where a \textit{centralized} approach is typically utilized to update the sampling distribution based on only the top-$k$ operation's results on samples. In this paper, we show that such a \textit{centralized} approach makes CEM vulnerable to local optima, thus impairing its sample efficiency. To tackle this issue, we propose \textbf{Decent}ralized \textbf{CEM (DecentCEM)}, a simple but effective improvement over classical CEM, by using an ensemble of CEM instances running independently from one another, and each performing a local improvement of its own sampling distribution. We provide both theoretical and empirical analysis to demonstrate the effectiveness of this simple \textit{decentralized} approach. We empirically show that, compared to the classical centralized approach using either a single or even a mixture of Gaussian distributions, our DecentCEM finds the global optimum much more consistently thus improves the sample efficiency. Furthermore, we plug in our DecentCEM in the planning problem of MBRL, and evaluate our approach in several continuous control environments, with comparison to the state-of-art CEM based MBRL approaches (PETS and POPLIN). Results show sample efficiency improvement by simply replacing the classical CEM module with our DecentCEM module, while only sacrificing a reasonable amount of computational cost. Lastly, we conduct ablation studies for more in-depth analysis. Code is available at \url{https://github.com/vincentzhang/decentCEM}.
\end{abstract}

\section{Introduction}
Model-based reinforcement learning (MBRL) uses a model as a proxy of the environment for planning actions in multiple steps.
This paper studies planning in MBRL with a specific focus on the Cross-Entropy Method (CEM) \citep{CEM, mannor2003cross}, which is popular in MBRL due to its ease of use and strong empirical performance \citep{PETS, planet, poplin, zhang2021importance, leggedrobot}. CEM is a stochastic, derivative-free optimization method. It uses a sampling distribution to generate imaginary trajectories of environment-agent interactions with the model. These trajectories are then ranked based on their returns computed from the rewards given by the model.
The sampling distribution is updated to increase the likelihood of producing the top-$k$ trajectories with higher returns.
These steps are iterated and eventually yield an improved distribution over the action sequences to guide the action execution in the real environment.

Despite the strong empirical performance of CEM for planning, it is prone to two problems: (1) lower sample efficiency as the dimensionality of solution space increases, and (2) the Gaussian distribution that is commonly used for sampling may cause the optimization to get stuck in local optima of multi-modal solution spaces commonly seen in real-world problems. Previous works addressing these problems either add gradient-based updates of the samples to optimize the parameters of CEM, or adopt more expressive sampling distributions, such as using Gaussian Mixture Model~\citep{okada2020variational} or masked auto-regressive neural network~\citep{gacem2020}. Nevertheless, all CEM implementations to date are limited to a \textit{centralized} formulation where the ranking step involves \textit{all samples}. As analyzed below and in Section \ref{sec:motivation}, such a centralized design makes CEM vulnerable to local optima and impairs its sample efficiency.

We propose \textbf{Decent}ralized \textbf{CEM (DecentCEM)}, a simple but effective improvement over classical CEM, to address the above problems. 
Rather than ranking \textit{all samples}, as in the \textit{centralized} design,
our method distribute the sampling budget across an ensemble of CEM instances. These instances run independently from one another, and each performs a local improvement of its own sampling distribution based on the ranking of its generated samples. The best action is then aggregated by taking an $\argmax$ among the solution of the instances. It recovers the conventional CEM when the number of instance is one.
\begin{wrapfigure}{r}{0.46\textwidth}
  \centering
\subfigure[Centralized CEM]{\includegraphics[ trim=0 0 0 0cm, clip=True, width=0.49\linewidth]{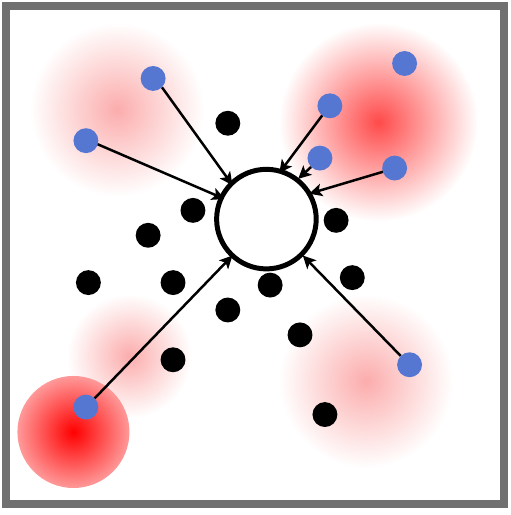}}    
  \label{fig:centralized}
\subfigure[Decentralized CEM]{\includegraphics[trim=0 0 0 0cm,clip,width=0.49\linewidth]{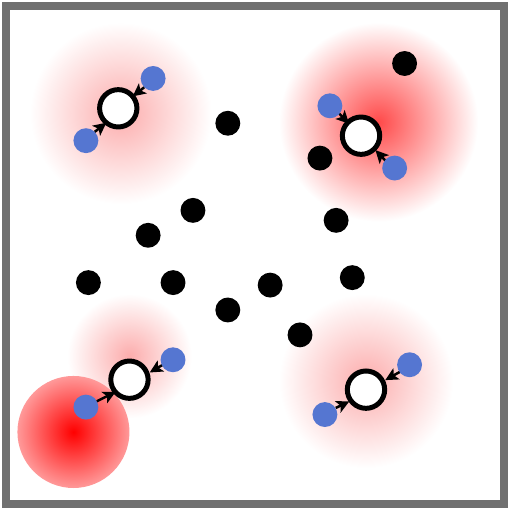}}
\label{fig:decentralized}
    \caption{\footnotesize
    Illustration of CEM approaches in optimization.
    Shades of red indicate relative value of the 2D optimization landscape: \textcolor{myred}{redder} is better. 
    and optimal solutions are near bottom left corner of the solution space. \textcolor{myblue}{Blue dots} \protect\markerblue \ are top-$k$ samples, and black dots \protect\markerblack \ are other samples.  Open dots \protect\markeropenblack \, represent the sampling distributions whose sizes indicate the number of generated samples. \normalsize
 }
    \label{fig:decentralized-comp}
\end{wrapfigure}
We hypothesize that by shifting to this \textit{decentralized} design, CEM can be less susceptible to premature convergence caused by the \textit{centralized} ranking step. As illustrated in Fig.~\ref{fig:decentralized-comp}, the \textit{centralized} sampling distribution exhibits a bias toward the sub-optimal solutions near top right, due to the \textit{global} top-$k$ ranking. This bias would occur regardless of the family of distributions used. In comparison, a \textit{decentralized} approach could maintain enough diversity thanks to its \textit{local} top-$k$ ranking in each sampling instance.

Through a one-dimensional multi-modal optimization problem in Section \ref{sec:motivation}, we show that DecentCEM empirically finds the global optimum more consistently than centralized CEM approaches that use either a single or a mixture of Gaussian distributions. Also we show that DecentCEM is theoretically sound that it converges almost surely to a local optimum.
We further apply it to sequential decision making problems and use neural networks to parameterize the sampling distributions. %
Empirical results in several continuous control environments suggest that DecentCEM offers an effective mechanism to improve the sample efficiency over the baseline CEM under the same sample budget for planning.

\section{Preliminaries}
We consider a Markov Decision Process (MDP) specified by ($S$,$A$,$R$,$P$,$\gamma$,$d_0$,$T$). 
$S \subset \mathbb{R}^{d_s}$ is the state space, $A \subset \mathbb{R}^{d_a}$ is the action space. $d_s, d_a$ are scalars denoting the dimensionality.
$R: S \times A \rightarrow \mathbb{R} $ is the reward function that maps a state and action pair to a real-valued reward.
$P(s'|s,a): S \times A \times S \rightarrow \mathbb{R}^+ $ is the transition probability from a state and action pair $s,a$ to the next state $s'$.
$\gamma \in [0, 1]$ is the discount factor.
$d_0$ denotes the distribution of the initial state $s_0$.
At time step $t$, the agent receives a state $s_t$
\footnote{We assume full observability, i.e. agent has access to the state.} and takes an action $a_t$ according to a policy $\pi(\cdot|s)$ that maps the state to a probability distribution over the action space. 
The environment transitions to the next state $s_{t+1} \sim P( \cdot |s_t, a_t)$ and gives a reward $r_t = R(s_t, a_t)$ to the agent. 
Following the settings from CEM-based MBRL papers (Sec.~\ref{sec:benchmark}), we assume that the reward function is deterministic (a mild assumption \citep{agarwal2019reinforcement}) and known. Note that they are not fundamental limitations of CEM and are adopted here so as to be consistent with the literature.
The return $G_t = \sum_{i=0} ^{T} \gamma^i r_{t+i}$, is the sum of discounted reward within an episode length of $T$.
The agent aims to find a policy $\pi$ that maximizes the expected return.
We denote the learned model in MBRL as $f_\omega(\cdot | s,a)$, which is parameterized by $\omega$ and approximates $P(\cdot|s,a)$.

\paragraph{\textbf{Planning with the Cross Entropy Method}}
\label{sec:CEM}

Planning in MBRL is about leveraging the model to find the best action in terms of its return. 
Model-Predictive-Control (MPC) performs decision-time planning at each time step up to a horizon to find the optimal action sequence:
\begin{equation}
    \pi_\text{MPC}(s_t) = \argmax_{a_{t:t+H-1}} \mathbb{E} [\Sigma_{i=0}^{H-1} \gamma^i r(s_{t+i}, a_{t+i}) + \gamma^H V(s_H) ] 
    \label{eq:MPC}
\end{equation}
where $H$ is the planning horizon, $a_{t:t+H-1}$ denotes the action sequence from time step $t$ to $t+H-1$, and $V(s_H)$ is the terminal value function at the end of the horizon. The first action in this sequence is executed and the rest are discarded. The agent then re-plans at the next time step.

\begin{figure*}[htb]
    \centering
    \includegraphics[width=.9\textwidth]{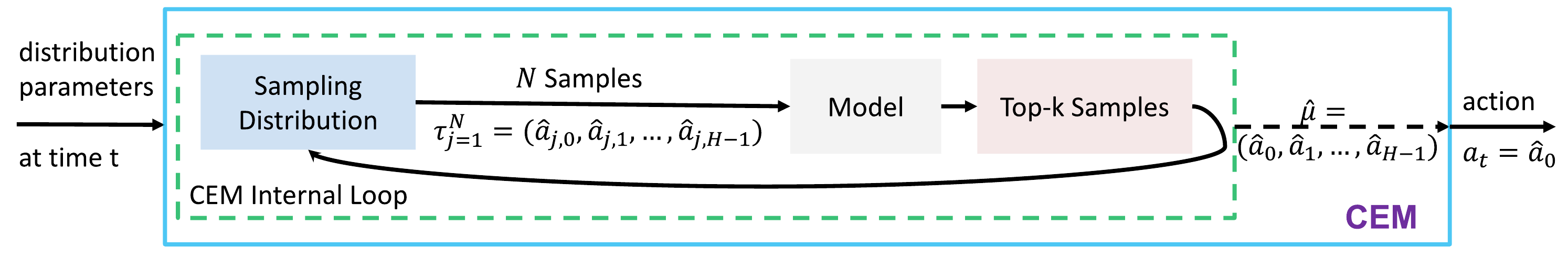}
    \caption{\footnotesize Cross Entropy Method (CEM) for Planning in MBRL \normalsize}
    \label{fig:cem-workflow}
\end{figure*}

The Cross-Entropy Method (CEM) is a gradient-free optimization method that can be used for solving Eq. (\ref{eq:MPC}).
The workflow is shown in Fig. \ref{fig:cem-workflow}.
CEM planning starts by generating $N$ samples $\{\tau_j\}_{j=1}^N = \{(\hat{a}_{j,0}, \hat{a}_{j,1}, \cdots, \hat{a}_{j,H-1}) \}_{j=1}^N$ from an initial sampling distribution $g_\phi(\tau)$ parameterized by $\phi$, where each sample $\tau_j$ is an action sequence from the current time step up to the planning horizon $H$. The domain of $g_\phi(\tau)$ has a dimension of ${d_\tau} = d_a H$.

Using a model $f$, CEM generates imaginary rollouts based on the action sequence $\{\tau_j\}$ (in the case of a stochastic model) and estimate the associated value $v(\tau_j)=\mathbb{E}[\Sigma_{i=0}^{H-1} \gamma^i r(s_{j,i}, a_{j,i})]$ where $s_{j,0}$ is the current state $s$ and $s_{j,i+1} \sim f(\cdot|s_{j,i}, a_{j,i})$.
The terminal value $\gamma^H V(s_{j,H})$ is omitted here following the convention in the CEM planning literature but the MPC performance can be further improved if paired with an accurate value predictor~\citep{bertsekas2005dynamic,POLO}.
The sampling distribution is then updated by fitting to the current top-$k$
samples in terms of their value estimates $v(\tau_j)$, using the Maximum Likelihood Estimation (MLE) which solves:
\begin{equation}
\phi' = \argmax_\phi \sum_{j=1}^N    \mathbbm{1} (v(\tau_j) \geq v_{\text{th}}) \log g_\phi(\tau_j)
\label{eq:MLE}
\end{equation}
where $v_{\text{th}}$ is the threshold equal to the value of the $k$-th best sample and $\mathbbm{1}(\cdot)$ is the indicator function.
In practice, the update to the distribution parameters are smoothed by
$\phi^{l+1}=\alpha \phi' + (1-\alpha) \phi^{l}$ where $\alpha \in [0, 1]$ is a smoothing parameter that balances between the solution to Eq. (\ref{eq:MLE}) and the parameter at the current internal iteration $l$.
CEM repeats this process of sampling and distribution update in an inner-loop, until it reaches the stopping condition. In practice, it is stopped when either a maximum number of iterations has been reached or the parameters have not changed for a few iterations.
The output of CEM is an action sequence, typically set as the 
expectation\footnote{Other options are discussed in Appendix \ref{ap:cem-output}} of the most recent sampling distribution for uni-modal distributions such as Gaussians $\hat{\mu} = \mathbb{E}(g_\phi) = (\hat{a}_{0}, \hat{a}_{1}, \cdots, \hat{a}_{H-1})$. %

\paragraph{\textbf{Choices of Sampling Distributions in CEM:}}
\label{sec:CEM-GMM}
A common choice is a multivariate Gaussian distribution under which Eq.(\ref{eq:MLE}) has an analytical solution. %
But the uni-modal nature of Gaussian makes it inadequate in solving multi-modal optimization that often occurs in MBRL.
To increase the capacity of the distribution,
a Gaussian Mixture Model (GMM) can be used \citep{okada2020variational}. 
We denote such an approach as \textit{CEM-GMM}. Going forward, we use \textit{CEM} to refer to the vanilla version that employs a Gaussian distribution. %
Computationally, the major difference between \textit{CEM} and \textit{CEM-GMM} is that the distribution update in \textit{CEM-GMM} is more computation-intensive since it solves for more parameters. %
Detailed steps can be found in \citet{okada2020variational}. %

\section{Decentralized CEM}
\label{sec:motivation}

In this section, we first introduce the formulation of the proposed  \textit{decentralized} approach  called the \textbf{Decent}ralized \textbf{CEM (DecentCEM)}. Then we illustrate the intuition behind the proposed approach using a one-dimensional synthetic multi-modal optimization example where we show the issues of the existing CEM methods and how they can be addressed by DecentCEM.

\paragraph{\textbf{Formulation of DecentCEM}}
\label{sec:decentcem}
DecentCEM is composed of an ensemble of $M$ CEM instances indexed by $i$, each having its own sampling distributions $g_{\phi_i}$.
They can be described by a set of distribution parameters $\Phi = \{{\phi_i}\}_{i=1}^M$.
Each instance $i$ manages its own sampling and distribution update by the steps described in Section \ref{sec:CEM}, independently from other instances.

Note that the $N$ samples and $k$ elites are evenly split among the $M$ instances. The top-$\frac{k}{M}$ sample sets are decentralized and managed by each instance independently whereas the centralized approach only keeps one set of top-$k$ samples regardless of the distribution family used.

After the stopping condition is reached for all instances, the final sampling distribution is taken as the best distribution in the set $\Phi$ according to (the $\argmax$ uses a deterministic tie-breaking):
\begin{equation}
    \phi_\text{DecentCEM} =
    \argmax_{\phi_i \in \Phi} \mathbb{E}_{\phi_i} [v(x)] \approx \argmax_{\phi_i \in \Phi}
    \sum_{j=1}^{\frac{N}{M}} v(\tau_{i,j})
    \label{eq:decentcem-argmax}
\end{equation}
where $\mathbb{E}_{\phi_i}[ v(x)]$ denotes the expectation with respect to the distribution $g_{\phi_i}$, approximated by the sample mean of $\frac{N}{M}$ samples.
When $M=1$, it recovers the conventional CEM.

\paragraph{\textbf{Motivational Example}}
\label{sec:optimization-1d}
   
\vspace{-5mm}
\begin{figure}[h]
  \centering
    \begin{tabular}{@{}cc}
		{\includegraphics[trim=0.2cm 0 0.25cm 0, clip,height=4cm]{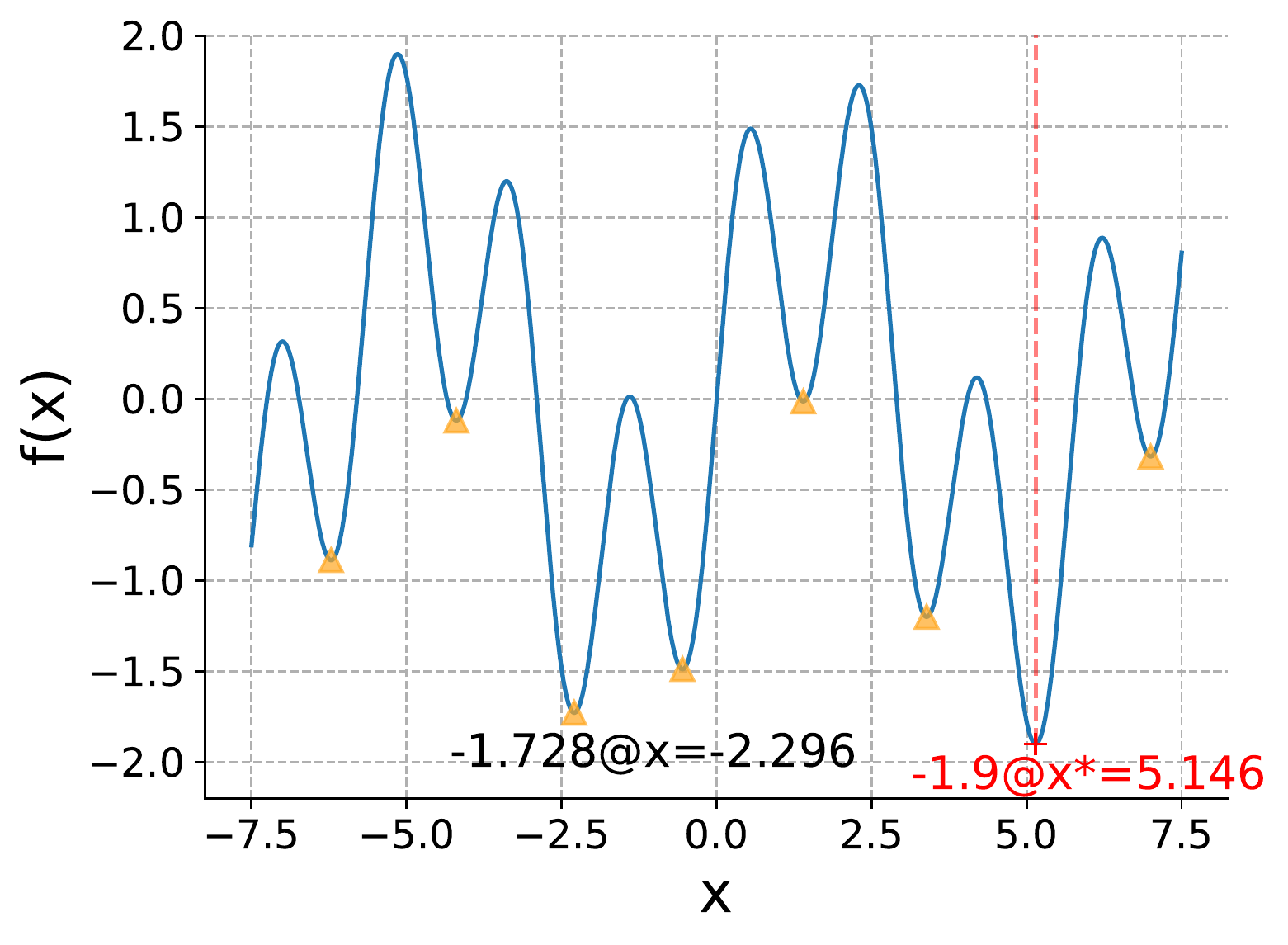}}
		&{\includegraphics[trim=0.2cm 0 0.3cm 0, clip, height=4cm]{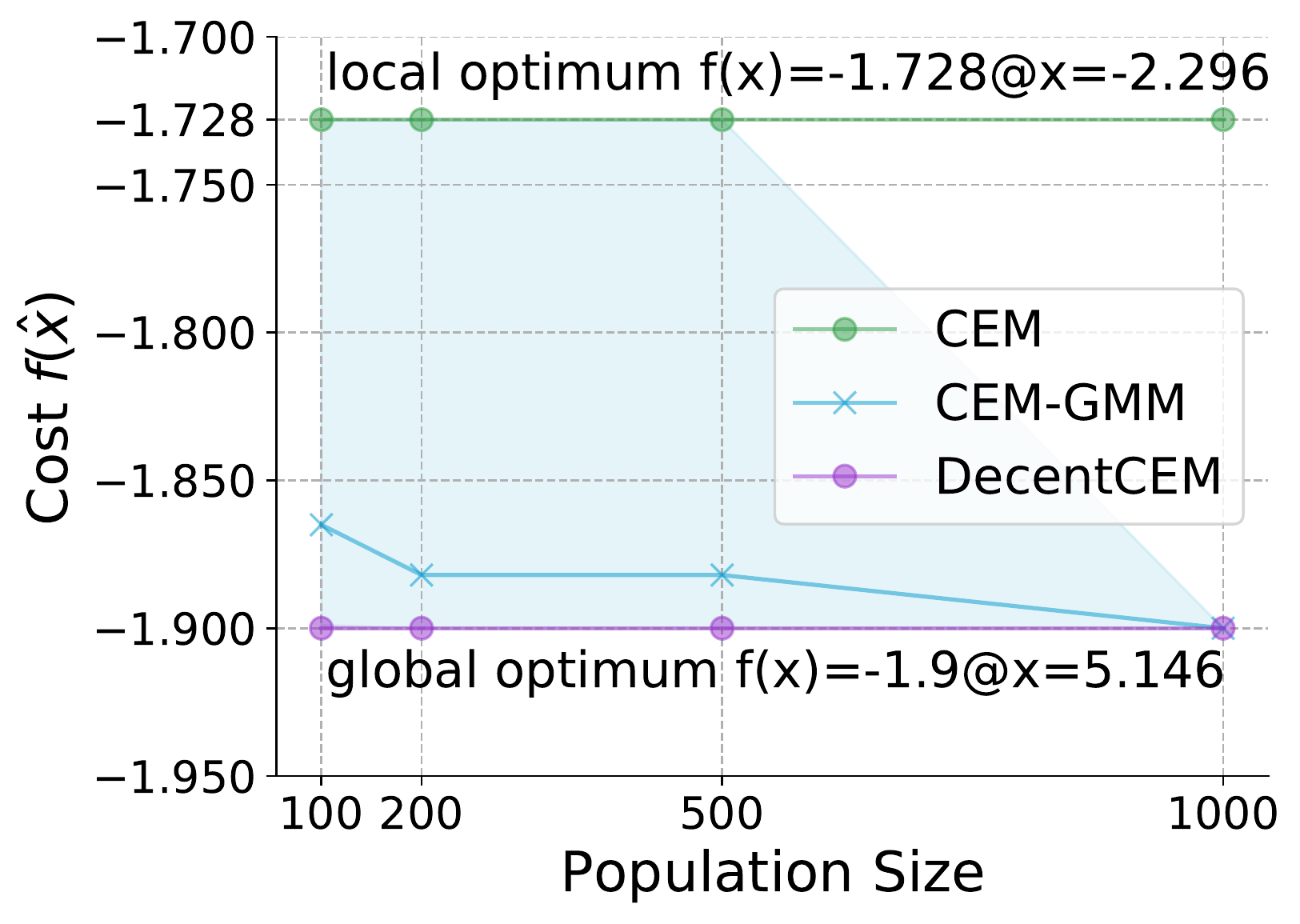}}
	\end{tabular}
 \vspace{-2mm}
    \caption{\footnotesize Left: The objective function in a 1D optimization task. 
    Right: Comparison of the proposed \textit{DecentCEM} method to \textit{CEM} and \textit{CEM-GMM}, 
     wherein the line and the shaded region denote the mean and the min/max cost from 10 independent runs. $\hat{x}$: solution of each method.\normalsize}
    \label{fig:motivation-1d}
   
\end{figure}

Consider a one dimensional multi-modal optimization problem shown in Fig.~\ref{fig:motivation-1d} (Left): $\arg\min_x \sin(x) + \sin(10 x/3), -7.5 \leq x \leq 7.5$.
There are eight local optima, including one global optimum $f(x^*) = -1.9$ where $x^{*}=5.146$. This objective function mimics the RL value landscape that has many local optima, as shown by \cite{poplin}. 
This optimization problem is ``easy'' in the sense that a grid search over the domain can get us a solution close to the global optimum.
However, only our proposed \textit{DecentCEM} method successfully converges to the global optimum consistently under varying population size (i.e., number of samples) and random runs, as shown in Fig.~\ref{fig:motivation-1d} (Right). For a fair comparison, hyperparameter search has been conducted on all methods for each population size (Appendix \ref{ap:motivation}).

\begin{figure}[bth]
  \centering
     \includegraphics[width=0.32\linewidth]{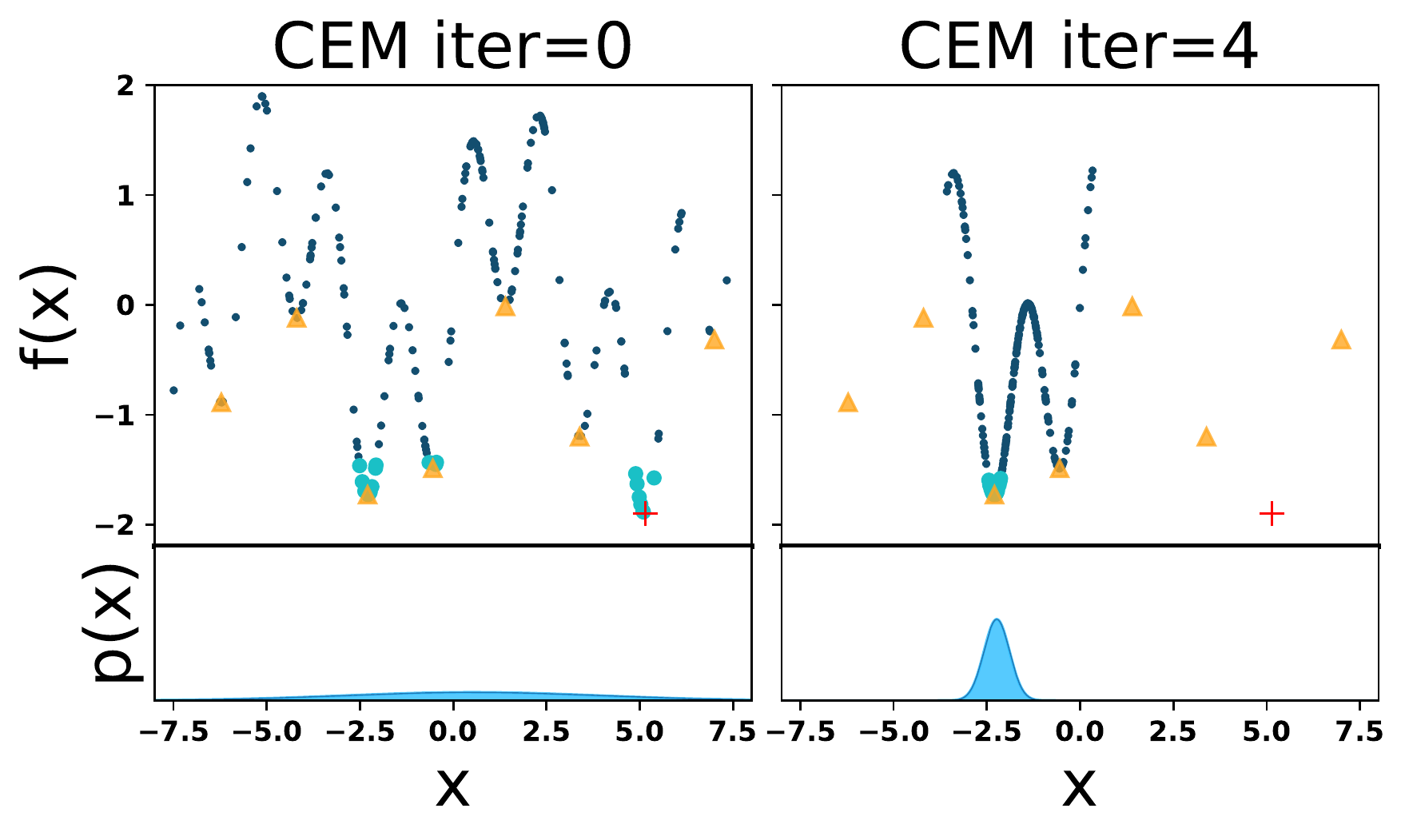}
    \includegraphics[width=0.32\linewidth]{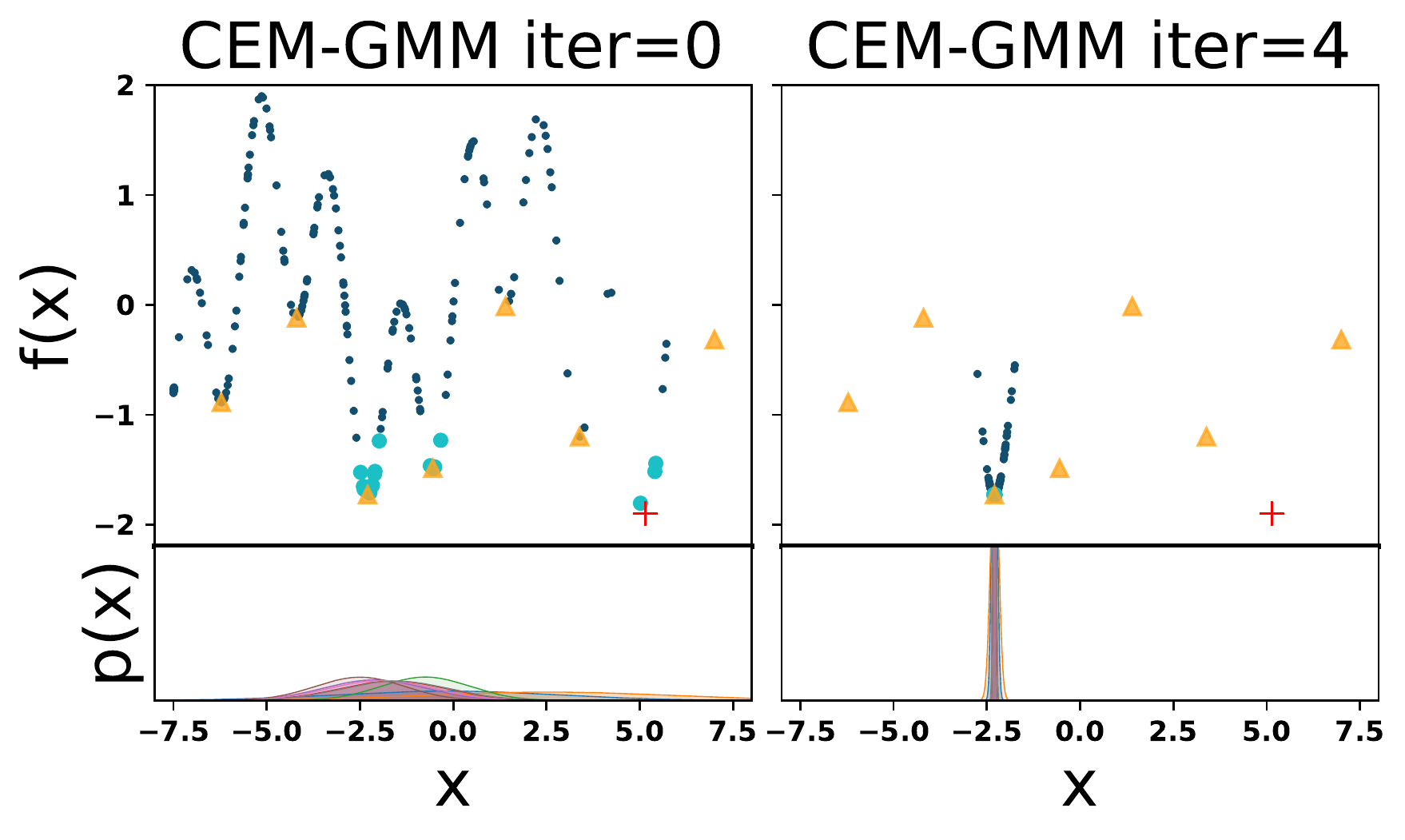}
    \includegraphics[width=0.32\linewidth]{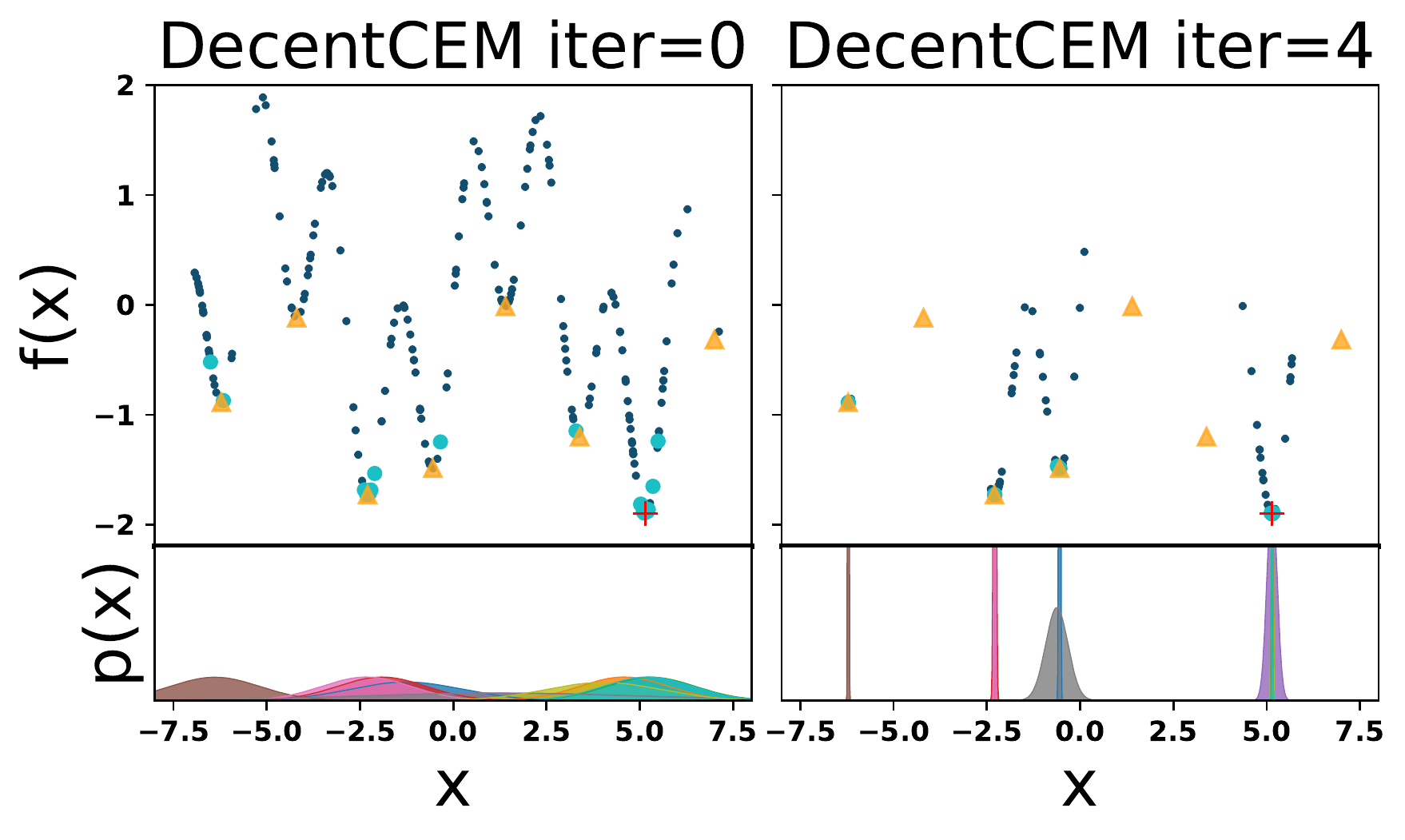}
    \vspace{-5pt}
  \caption{\footnotesize How the sampling distributions evolve in the 1D optimization task, \textit{after} the specified iteration. Symbols include samples \protect\markersample, elites \protect\markerelite, local optima \protect\markerlocal, global optimum \protect\markerglobal.
  2nd row in each figure shows the weighted p.d.f of individual distribution. Population size: 200.\normalsize}
  \label{fig:vis1d}
\vspace{-5pt}
\end{figure}

Both \textit{CEM-GMM} and the proposed \textit{DecentCEM} are equipped with multiple sampling distributions. 
The fact that \textit{CEM-GMM} is outperformed by \textit{DecentCEM} may appear surprising.
To gain some insights, 
we illustrate in Fig.~\ref{fig:vis1d} how the sampling distribution evolves during the iterative update (more details in Fig.~\ref{fig:motivation_1D_full_compare} in Appendix).
\textit{CEM} updated the unimodal distribution toward a local optimum despite seeing the global optimum.
\textit{CEM-GMM} appears to have a similar issue. 
During MLE on the top-$k$ samples, it moved most distribution components towards the same local optimum which quickly led to mode collapse. On the contrary, \textit{DecentCEM} successfully escaped the local optima thanks to its independent distribution update over \textit{decentralized} top-$k$ samples and was able to maintain a decent diversity among the distributions.

GMM suits density estimation problems like distribution-based clustering where the samples are drawn from a \textit{fixed} \textit{true} distribution that can be represented by multi-modal Gaussians. 
However, in CEM for optimization, exploration is coupled with density estimation: the sampling distribution in CEM is \textit{not fixed} but rather gets updated iteratively toward the top-$k$ samples.
And the ``true'' distribution in optimization puts uniform non-zero densities to the global optima and zero densities everywhere else. When there is a unique global optimum, it degenerates into a Dirac measure that assigns the entire density to the optimum. 
Density estimation of such a distribution only needs one Gaussian but the exploration is challenging.
In other words, the conditions for GMM to work well are not necessarily met when used as the sampling distribution in \textit{CEM}.
\textit{CEM-GMM} is subject to mode collapse during the iterative top-$k$ greedification, causing premature convergence, as observed in Fig~\ref{fig:vis1d}.
In comparison, our proposed decentralized approach takes care of the exploration aspect by running multiple CEM instances independently, each performing its own local improvement. This is shown to be effective from this optimization example and the benchmark results in Section \ref{sec:exp}.
\textit{CEM-GMM} only consistently converge to the global optimum when we increase the population size to the maximum 1,000 which causes expensive computations.
Our proposed \textit{DecentCEM} runs more than 100 times faster than \textit{CEM-GMM} at this population size, shown in Table \ref{tb:motivation-runningtime} in Appendix.

\paragraph{Convergence of DecentCEM}
\label{sec:DecentCEM-theory}
The convergence of DecentCEM requires the following assumption: 
\begin{assumption}
Let $M$ be the number of instances in DecentCEM and each instance has a sample size of $\frac{N_k}{M}$ where $N_k$ is the total number of samples that satisfies $N_k=\Theta(k^\beta)$ where $\beta > \max\{0, 1-2\lambda\}$. $M$ and $\lambda$ are some positive constant and $0 < M < N_k \ \forall\ k$.
\label{ap:a1}
\end{assumption}

Here the assumption of the total sample size follows the same one as in the CEM literature.
Other such standard assumptions are summarized in Appendix \ref{ap:convergence-analysis}.

\begin{restatable}[Convergence of DecentCEM]{theorem}{convergence}
If a CEM instance described in Algorithm \ref{alg:CEM} converges, and we decentralize it by evenly dividing its sample size $N_k$ into $M$ CEM instances which satisfies the Assumption \ref{ap:a1}, then the resulting 
DecentCEM converges almost surely to the best solution of the individual instances.
\label{thm:main-convergence}
\end{restatable}
\begin{proof}(sketch)
We show that the previous convergence result of CEM  \citep{hu2011stochastic} extends to DecentCEM under the same sample budget.
The key observation is that the convergence property of each CEM instance still holds since the number of samples in each instance is only changed by a constant factor, i.e., number of instances. Each CEM instance converges to a local optimum. The convergence of DecentCEM to the best solution comes from the $\argmax$ operator and applying the strong law of large numbers.
The detailed proof is left to Appendix \ref{ap:convergence-analysis}.
\end{proof}

\section{DecentCEM for Planning}
\label{sec:decentcem-planning}
In this section, we develop two instantiations of DecentCEM for planning in MBRL where the sampling distributions are parameterized by policy networks. 
For the dynamics model learning, we adopt the ensemble of probabilistic neural networks proposed in \citep{PETS}. Each network predicts the mean and diagonal covariance of a Gaussian distribution and is trained by minimizing the negative log likelihood.  
Our work focuses on the planning aspect and refer to \citet{PETS} for further details on the model.

\paragraph{\textbf{CEM Planning with a Policy Network}}
\label{sec:cem-policynet}
In MBRL, CEM is applied to every state separately to solve the optimization problem stated in Eq. (\ref{eq:MPC}).
The sampling distribution is typically initialized to a fixed distribution at the beginning of every episode \citep{okada2020variational, pinneri2020sampleefficient}, or more frequently at every time step \citep{planet}. 
Such initialization schemes are sample inefficient since there is no mechanism that allows the information of the high-value region in the value space of one state to generalize to nearby states.
Also, the information is discarded after the initialization.
It is hence difficult to scale the approach to higher dimensional solution spaces, present in many continuous control environments.
\cite{poplin} proposed to use a policy network in CEM planning that helped to mitigate the issues above.

They developed two methods: \textit{POPLIN-A} that plans in the action space, and \textit{POPLIN-P} that plans in the parameter space of the policy network. 
In \textit{POPLIN-A}, the policy network is used to learn to output the mean of a Gaussian sampling distribution of actions. %
In \textit{POPLIN-P}, the policy network parameters serve as the initialization of the mean of the sampling distribution of parameters.
The improved policy network can then be used to generate an action.
They show that when compared to the vanilla method of using a fixed sampling distribution in the action space, both modes of CEM planning with such a learned distribution perform better. 
The same principle of combining a policy network with CEM can be applied to the DecentCEM approach as well, which we describe next.

\paragraph{\textbf{DecentCEM Planning with an Ensemble of Policy Networks}}
\label{sec:DecentCEM-Ensemble}
\begin{figure}[th]
    \vspace{-2mm}
    \centering
    \includegraphics[width=0.7\linewidth]{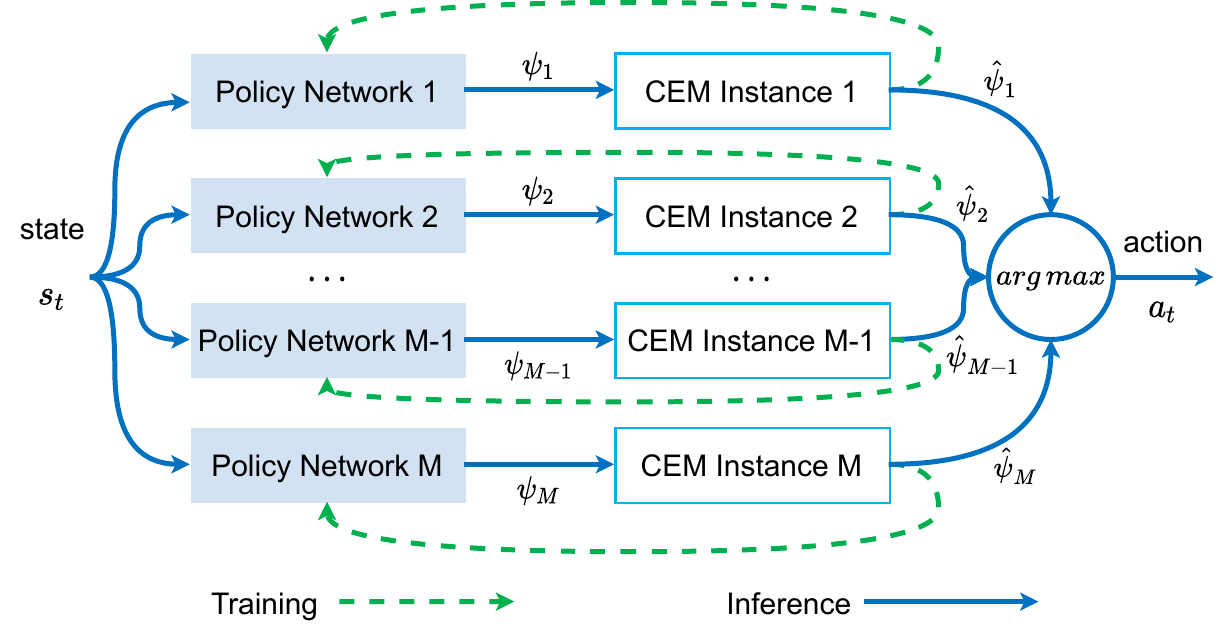}
    \caption{\footnotesize The architecture of DecentCEM planning with $M$ CEM instances. $\psi_i=\phi_i$ for planning in action space and $\psi_i=\theta_i$ for planning in policy network parameter space with the instance index $i\in\{1,\cdots,M\}$. \normalsize}
    \label{fig:decentcem-a-architecture}
    \vspace{-2mm}
\end{figure}

For better sample efficiency in MBRL setting, we extend DecentCEM to use an ensemble of policy networks to learn the sampling distributions in the CEM instances. 
Similar to the \textit{POPLIN} variations, we develop two instantiations of \textit{DecentCEM}, namely \textit{DecentCEM-A} and \textit{DecentCEM-P}. The architecture of the proposed algorithm is illustrated in Fig.~\ref{fig:decentcem-a-architecture}. 
\textit{DecentCEM-A} plans in the action space.
It consists of an ensemble of policy networks followed by CEM instances.
Each policy network takes the current state $s_t$ as input, outputs the parameters $\phi_i$ of the sampling distribution for CEM instance $i$ in the action space. There is no fundamental difference from the DecentCEM formulation in Section~\ref{sec:decentcem} except that the initialization of sampling distributions is learned by the policy networks rather than a fixed distribution.

The second instantiation \textit{DecentCEM-P} plans in the parameter space of the policy network. The output of each policy network is the network parameter $\theta_i$. The initial sampling distribution of CEM instance $i$ is a Gaussian distribution over the policy parameter space with the mean at $\theta_i$. In the $\argmax$ operation in Eq.~(\ref{eq:decentcem-argmax}), the sample $\tau_{i,j}$ denotes the $j$-th parameter sample from the distribution for CEM instance $i$. Its value is approximated by the model-predicted value of the action sequence generated from the policy network with the parameters $\tau_{i,j}$.

The ensemble of policy networks in both instantiations \textit{DecentCEM-A} and \textit{DecentCEM-P} are initialized with random weights, which is empirically found to be adequate to ensure that the output of the networks do not collapse into the same distribution (Sec.~\ref{sec:ab} and Appendix \ref{ap:ab}).

\paragraph{\textbf{Training the Policy Networks in DecentCEM}}
\label{sec:train-policy}
When planning in action space, the policy networks are trained by behavior cloning, similar to the scheme in \textit{POPLIN}~\citep{poplin}.
Denote the first action in the planned action sequence at time step $t$ by the $i$-th CEM instance as $\hat{a}_{t,i}$,
the $i$-th policy network is trained to mimic $\hat{a}_{t,i}$ and the training objective is
  $  \min_{\theta_i}
   \mathbb{E}_{s_t,\hat{a}_{t,i} \sim D_i}
   \|a_{\theta_i}(s_t) - \hat{a}_{t,i} \|^2    $
where $D_i$ denotes the replay buffer with the state and action pairs $(s_t,\hat{a}_{t,i})$.
$a_{\theta_i}(s_t)$ is the action prediction at state $s_t$ from the policy network parameterized by $\theta_i$.

While the above training scheme can be applied to both planning in action space and parameter space, we follow the \textit{setting parameter average} (AVG)~\citep{poplin} training scheme when planning in parameter space.
The parameter is updated as 
 $\theta_i = \theta_i +
 \frac{1}{|D_i|} \sum_{{\delta}_{i} \in D_i} {\delta}_{i}$
where $D_i = \{{\delta}_{i}\}$ is a dataset of policy network parameter updates planned from the $i$-th CEM instance previously. 
It is more effective than behavior cloning based on the experimental result reported by \cite{poplin} and our own preliminary experiments.

Note that each policy network in the ensemble is trained independently from the data observed by its corresponding CEM instance rather than from the aggregated result after taking the $\argmax$. This allows for enough diversity among the instances. More importantly, it increases the size of the training dataset for the policy networks compared to the approach taken in \textit{POPLIN}. For example, with an ensemble of $M$ instances, there would be $M$ training data samples available from one real environment interaction, compared to the one data sample in \textit{POPLIN-A/P}.
As a result, \textit{DecentCEM} is able to train larger policy networks than is otherwise possible, shown in Sec. \ref{sec:ab} and Appendix \ref{ap:ab}.

\section{Related Work} 

We limit the scope of related works to CEM planning methods, which is one of the broad class of planning algorithms in MBRL. For a review of different families of planning algorithms, the readers are referred to \citet{mbbl}.
Vanilla CEM planning in action space with a single Gaussian distribution has been adopted as the planning method for both simulated and real-world robot control \citep{PETS,finn2017deep,ebert2018visual,planet,leggedrobot,zhang2021importance}.
Previous attempts to improving the performance of CEM-based planning can be grouped into two types of approaches.
\textbf{The first type} includes CEM in a hybrid of CEM+X where ``X'' is some other component or algorithm.
POPLIN \citep{poplin} is a prominent example where ``X'' is a policy network that learns a state conditioned distribution that initializes the subsequent CEM process.
Another common choice of ``X'' is gradient-based adjustment of the samples drawn in CEM.
GradCEM \citep{gradcem2020} adjusts the samples in each iteration of CEM by taking gradient ascent of the return estimate w.r.t the actions.
CEM-RL \citep{cemrl2019} also combines CEM with gradient based updates from RL algorithms but the samples are in the parameter space of the actor network. 
To improve computational efficiency, \cite{asyncCEM2020} proposes an asynchronous version of CEM-RL where each CEM instance updates the sampling distribution asynchronously.

\textbf{The second type} of approach aims at improving CEM itself. 
\cite{DCEM2020} proposes a fully-differentiable version of CEM called \textit{DCEM}.
The key is to make the top-$k$ selection in CEM differentiable such that the entire CEM module can be trained in an end-to-end fashion.
Despite cutting down the number of samples needed in CEM, this method does not beat the vanilla CEM in benchmark test.
\textit{GACEM} \citep{gacem2020} increase the capacity of the sampling distribution by replacing the Gaussian distribution with an auto-regressive neural network.
This change allows CEM to perform search in multi-modal solution space but it is only verified in toy examples and its computation seems too high to be scaled to MBRL tasks.
Another method that increases the capacity of the sampling distribution is \textit{PaETS} \citep{okada2020variational} that uses a GMM with CEM. It is the approach that we followed for our \textit{CEM-GMM} implementation.
The running time results in the optimization task in Sec.\ref{sec:optimization-1d} shows that it is computationally heavier than the \textit{CEM} and \textit{DecentCEM} methods, limiting its use in complex environments.
\cite{macua2015distributed} proposed a ``distributed'' CEM that is similar in spirit to our method in that they used multiple sampling distributions and applied the top-$k$ selection locally to samples from each instance. 
However, their instances are cooperative as opposed to being independent as in our work. 
They applied “collaborative smoothed projection steps” to update each sampling distribution as an average of its neighboring instances including itself. The updating procedure is more complicated than our proposed method and proper network topology of the instances is needed: a naive approach of updating according to all instances will lead to mode collapse since the resulting sampling distributions will be identical. The method was tested in toy optimization examples only.
Overall, this second type of approach did not outperform vanilla CEM, a phenomenon that motivated our move to a decentralized formulation.

\vspace{-1mm}
\section{Experiments}
\label{sec:exp}
We evaluate the proposed \textit{DecentCEM} methods in simulated environments with continuous action space.
The experimental evaluation is mainly set up to understand if \textit{DecentCEM} improves the performance and sample efficiency over conventional CEM approaches.

\subsection{Benchmark Setup}
\label{sec:benchmark}

\textbf{{\em Environments}}
We run the benchmark in a set of OpenAI Gym~\citep{brockman2016openai} and MuJoCo~\citep{todorov2012mujoco} environments commonly used in the MBRL literature:
Pendulum, InvertedPendulum, Cartpole, Acrobot, FixedSwimmer\footnote{a modified version of the original Gym Swimmer environment where the velocity sensor on the neck is moved to the head. This fix was proposed by \cite{poplin}}, Reacher, Hopper, Walker2D, HalfCheetah, PETS-Reacher3D, PETS-HalfCheetah, PETS-Pusher, Ant.
The three environments prefixed by ``PETS'' are proposed by \citet{PETS}.
Note that MBRL algorithms often make different assumptions about the dynamics model or the reward function. Their benchmark environments are often modified from the original OpenAI gym environments such that the respective algorithm is runnable.
Whenever possible, we inherit the same environment setup from that of the respective baseline methods. This is so that the comparison against the baselines is fair. 
More details on the environments and their reward functions are in Appendix~\ref{ap:env}.

\textbf{\em Algorithms}
The baseline algorithms are \textit{PETS}~\citep{PETS} %
and \textit{POPLIN}~\citep{poplin}.
\textit{PETS} uses CEM with a single Gaussian distribution for planning.
The \textit{POPLIN} algorithm combines a single policy network with CEM. As described in Sec.~\ref{sec:cem-policynet}, 
\textit{POPLIN} comes with two modes: \textit{POPLIN-A} and \textit{POPLIN-P} with the suffix ``A'' denotes planning in action space and ``P'' for the network parameter space.
We reuse default hyperparameters for these algorithms from the original papers if not mentioned specifically.
For our proposed methods, we include two variations \textit{DecentCEM-A} and \textit{DecentCEM-P} 
as described in Sec.~\ref{sec:DecentCEM-Ensemble} where the suffix carries the same meaning as in \textit{POPLIN-A/P}. 
The ensemble size of \textit{DecentCEM-A/P} as well as detailed hyperparameters for all algorithms are listed in the Appendix \ref{ap:hyperparam}.
We also included Decentralized \textit{PETS}, denoted by \textit{DecentPETS}.
All MBRL algorithms studied in this benchmark use the same ensemble networks proposed by \cite{PETS} for model learning. And a model-free RL baseline SAC~\citep{SAC2018} was included.

\textbf{\em Evaluation Protocol}
The learning curve shows the mean and standard error of the test performance out of 5 independent training runs.
The test performance is an average return of 5 episodes of the \emph{evaluation} environment, evaluated at every training episode.
At the beginning of each training run, the evaluation environment is initialized with a fixed random seed such that the \emph{evaluation} environments are consistent across different methods and multiple runs to make it a fair comparison.
All experiments were conducted using Tesla V100 Volta GPUs.

\subsection{Results}
\begin{figure}[bt]
    \centering
    	\begin{tabular}{@{}ccc}
		{\includegraphics[height=3.55cm]{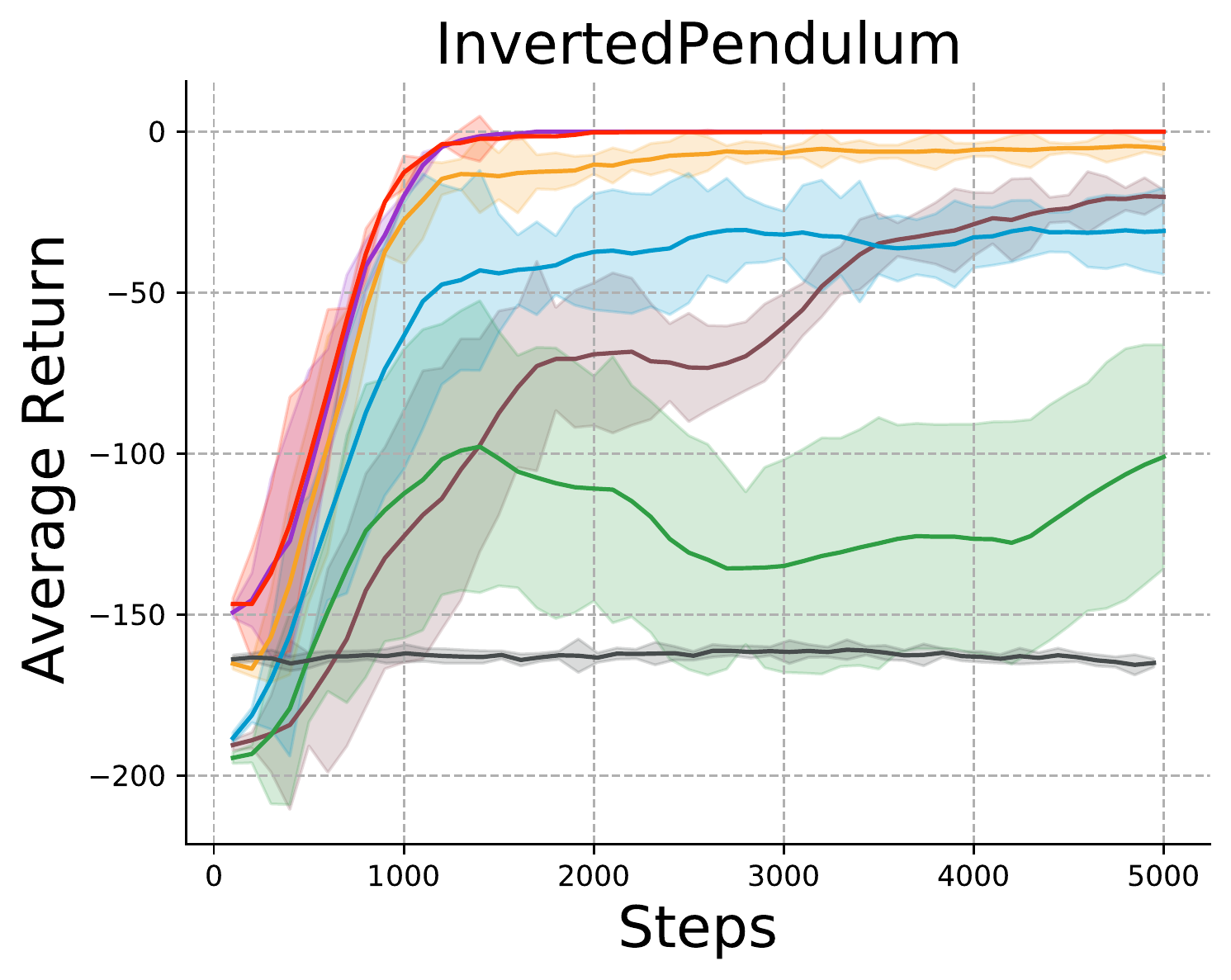}}
		&{\includegraphics[height=3.55cm]{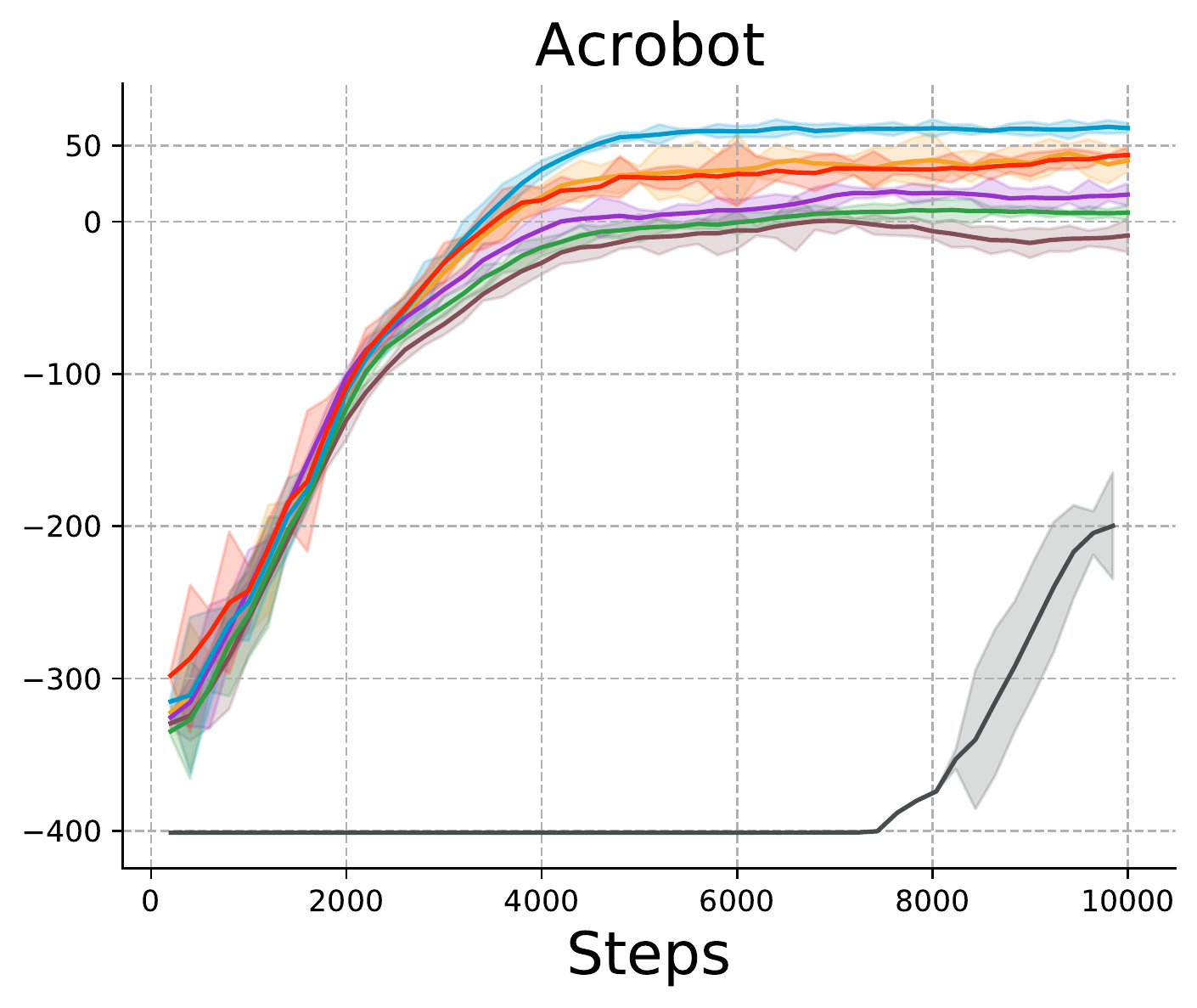}}
		&{\includegraphics[height=3.55cm]{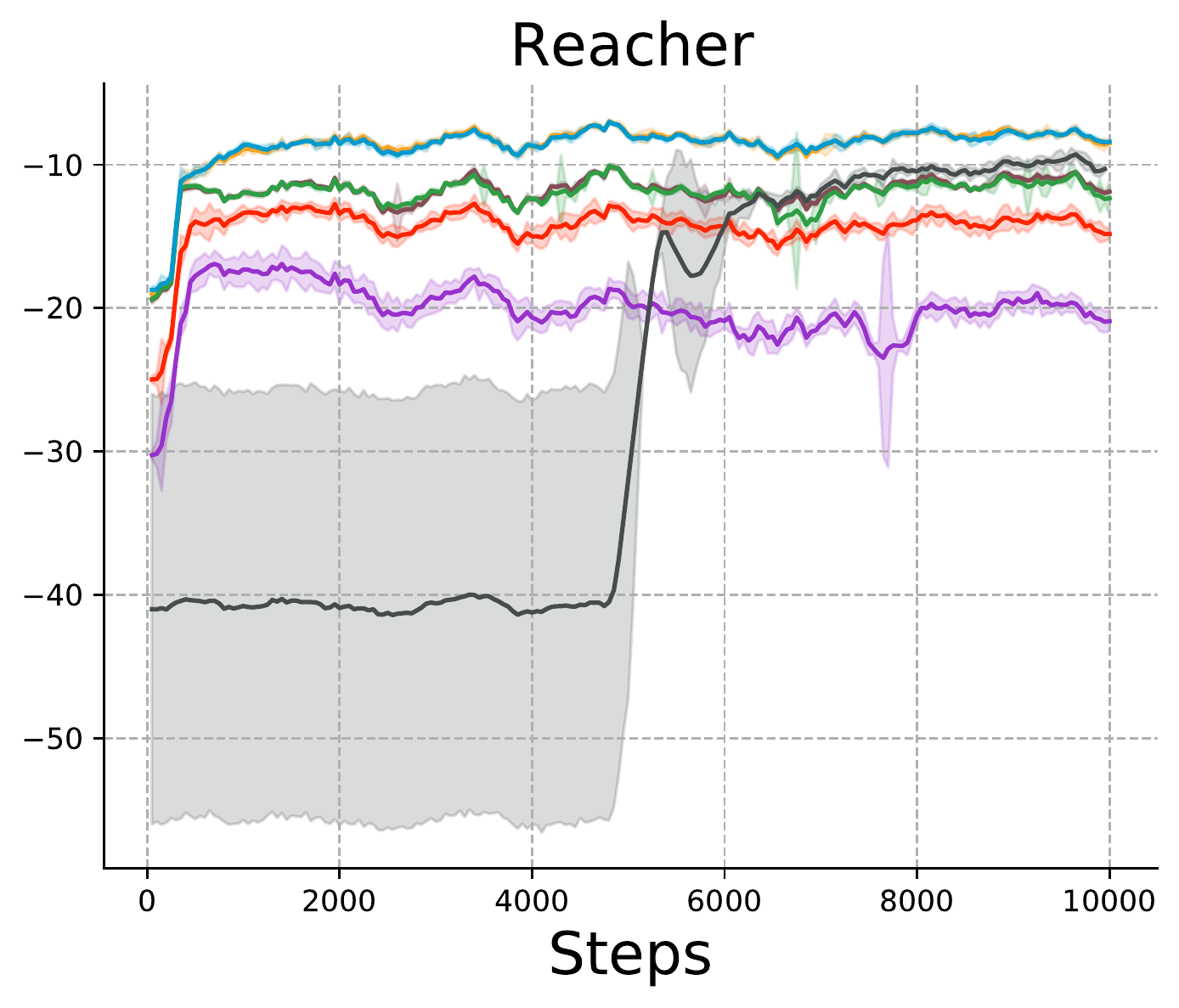}} \\
		{\includegraphics[height=3.55cm]{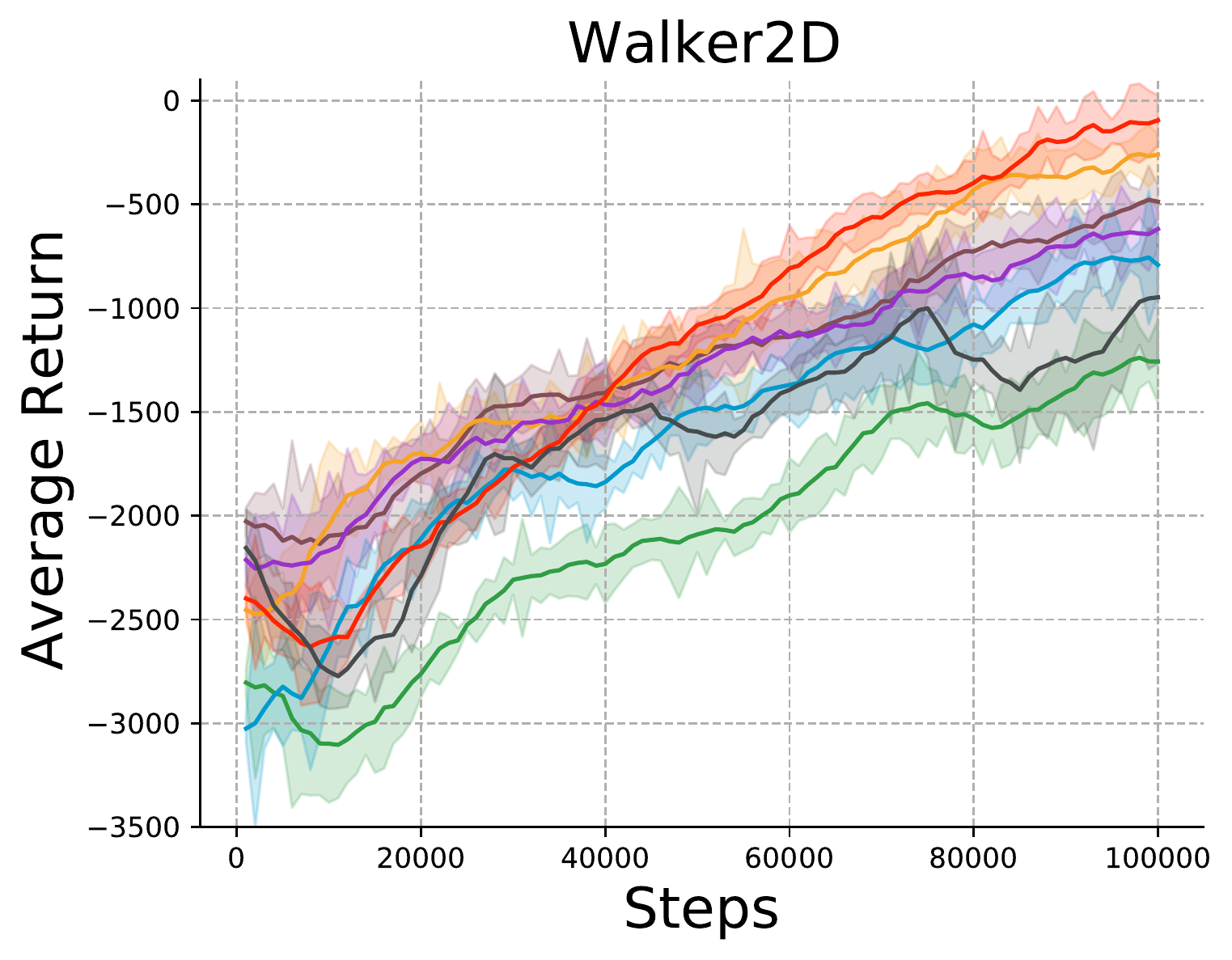}} &
		{\includegraphics[height=3.55cm]{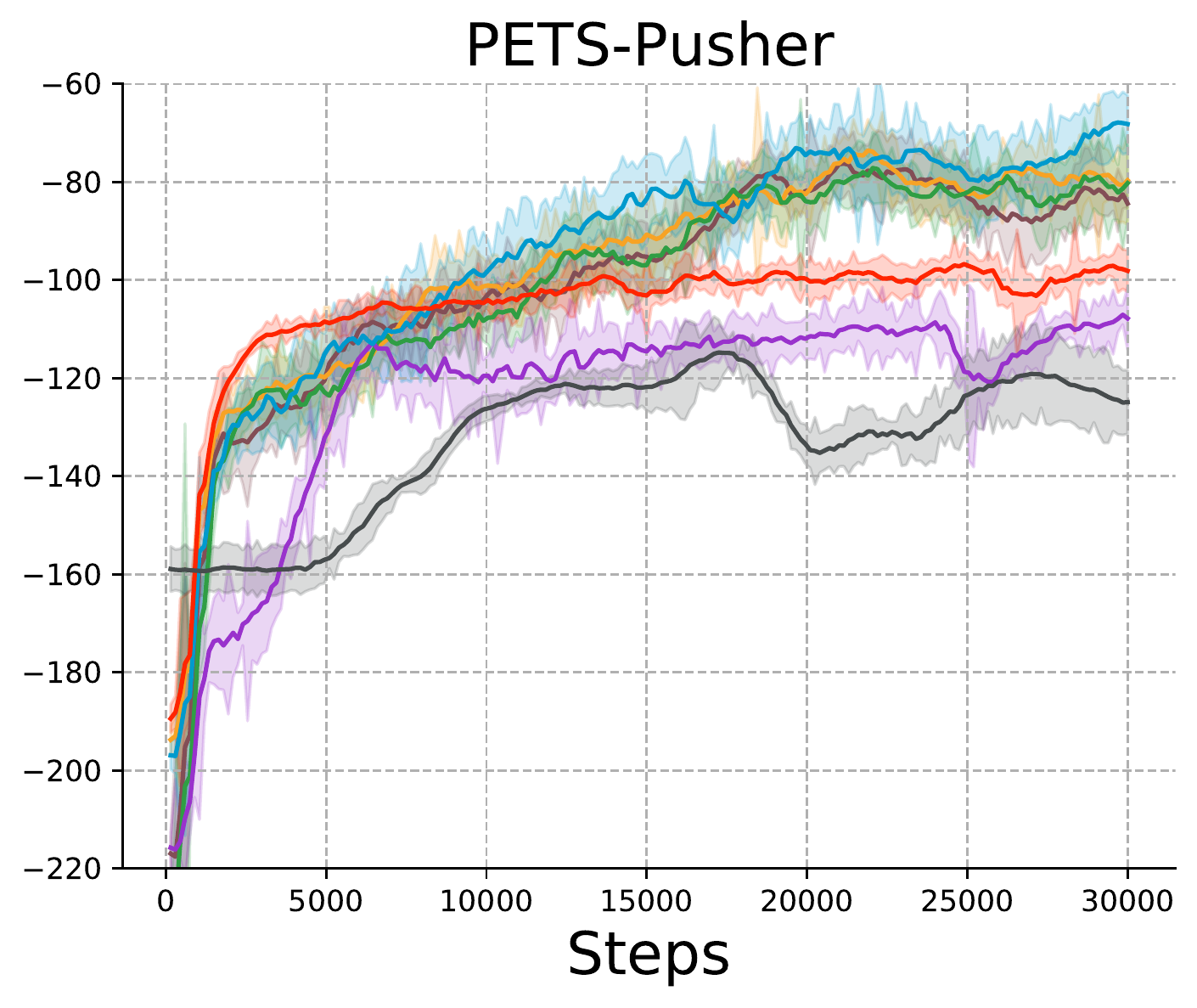}} &
		{\includegraphics[height=3.55cm]{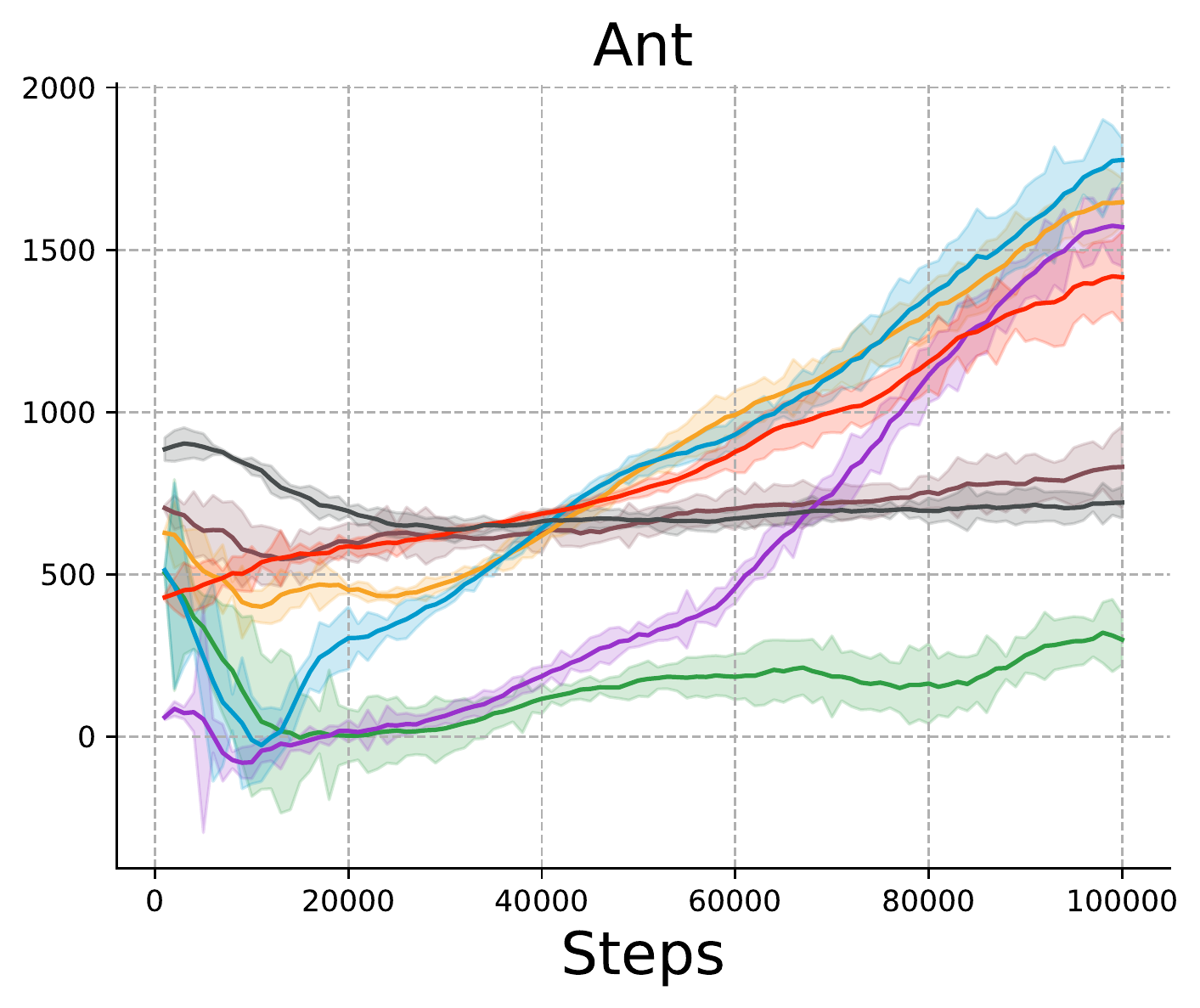}} 
	\end{tabular}
		\includegraphics[width=0.95\textwidth]{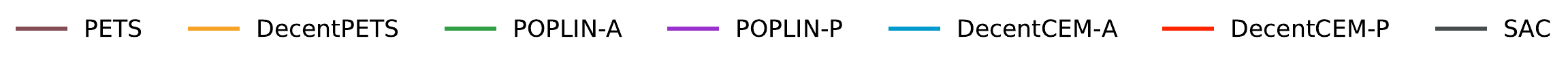}
		\vspace{-5pt}
    \caption{\footnotesize Learning curves of the proposed \textit{DecentCEM} methods and the baselines on continuous control environments. The line and the shaded region shows the mean and standard error of evaluation results from 5 training runs with different random seeds.
    \normalsize } %
    \label{fig:learningcurve}
    \vspace{-10pt}
\end{figure}

\paragraph{\textbf{Learning Curves}}
The learning curves of the algorithms are shown in Fig.~\ref{fig:learningcurve} for InvertedPendulum, Acrobot, Reacher, Walker2D, PETS-Pusher and Ant, sorted by the state and action dimensionality of task. 
The curves are smoothed with a sliding window of 10 data points.
The full results for all environments are included in Appendix \ref{ap:results}.
\textit{We can observe two main patterns from the results.}
One pattern was that 
in most environments, the \textit{decentralized} methods \textit{DecentPETS}, \textit{DecentCEM-A/P} either matched or outperformed their counterpart that took a \textit{centralized} approach.
In fact, the former can be seen as a generalization of the later, by an additional hyperparameter that controls the ensemble size with size one recovering the centralized approach.
The optimal ensemble size depends on the task. 
This additional hyperparameter offers flexibility in fine-tuning CEM for individual domains.
For instance in Fig.~\ref{fig:learningcurve}, for the P-mode where the planning is performed in policy parameter space, an ensemble size of larger than one works better in most environments while an ensemble size of one works better in Ant.
The other pattern was that using policy networks to learn the sampling distribution in general helped improving the performance of \textit{centralized} CEM but not necessarily in \textit{decentralized} formulation.  
This is perhaps due to that the added exploration from multiple instances makes it possible to identify good solutions in some environments.
Using a policy net in such case may hinder the exploration due to overfitting to previous actions.

\paragraph{\textbf{Ablation Study}}
\label{sec:ab}

A natural question to ask about the \textit{DecentCEM-A/P} methods is whether the increased performance is from the larger number of neural network parameters.
We added two variations of the \textit{POPLIN} baselines where a bigger policy network was used. 
The number of the network parameters was equivalent to that of the ensemble of policy networks in  \textit{DecentCEM-A/P}. We show the comparison in Reacher(2) and PETS-Pusher(7) (action dimension in parenthesis) in Fig.~\ref{fig:bignet_pc}.
In both action space and parameter space planning, 
a bigger policy network in \textit{POPLIN} either did not help or significantly impaired the performance (see the \textit{POPLIN-P} results in reacher and PETS-Pusher). This is expected since unlike \textit{DecentCEM}, the training data in \textit{POPLIN} do not scale with the size of the policy network, as explained at the end of Sec.~\ref{sec:train-policy}. 

\begin{figure}[h]
\vspace{-2mm}
	\centering
    	\begin{tabular}{m{4.8cm}m{4.8cm}m{4.8cm}}
	{\includegraphics[height=4cm, trim=0.3cm 0 0.32cm 0, clip]{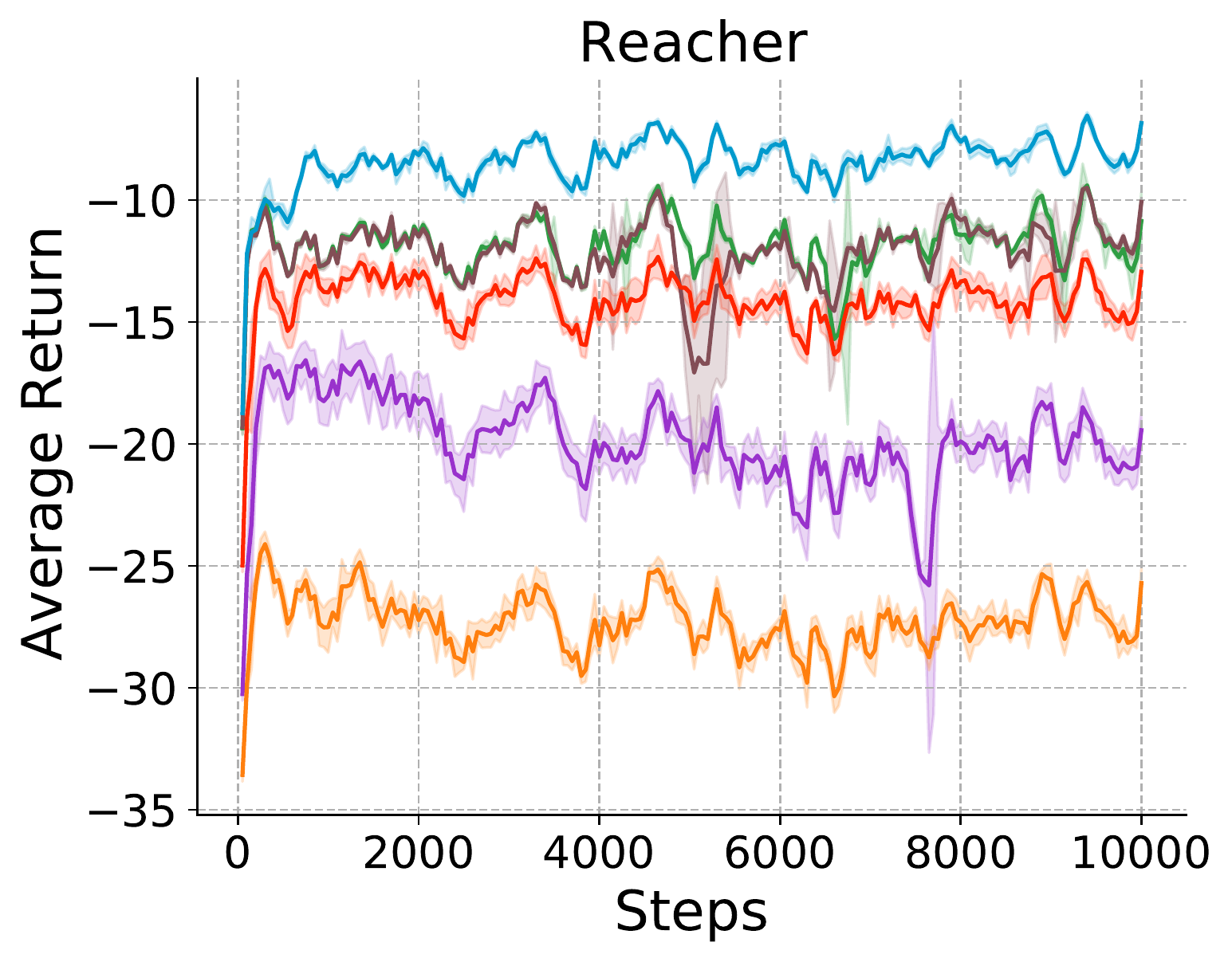}}	
    & {\includegraphics[height=4cm,trim=0.3cm 0 0.3cm 0, clip]{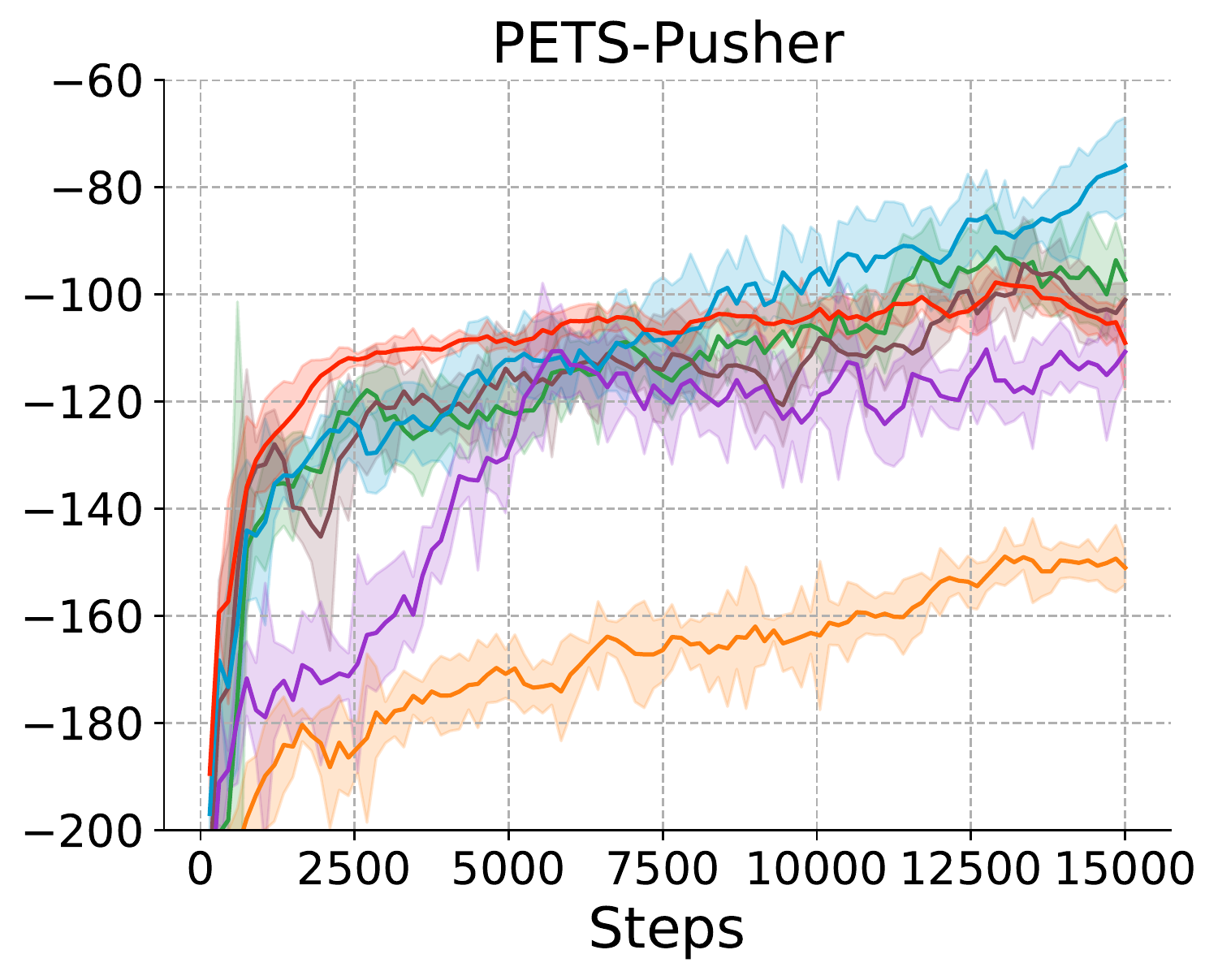}}
    & {\includegraphics[height=1.95cm,trim=0.6cm 0 0.4cm 0, clip]{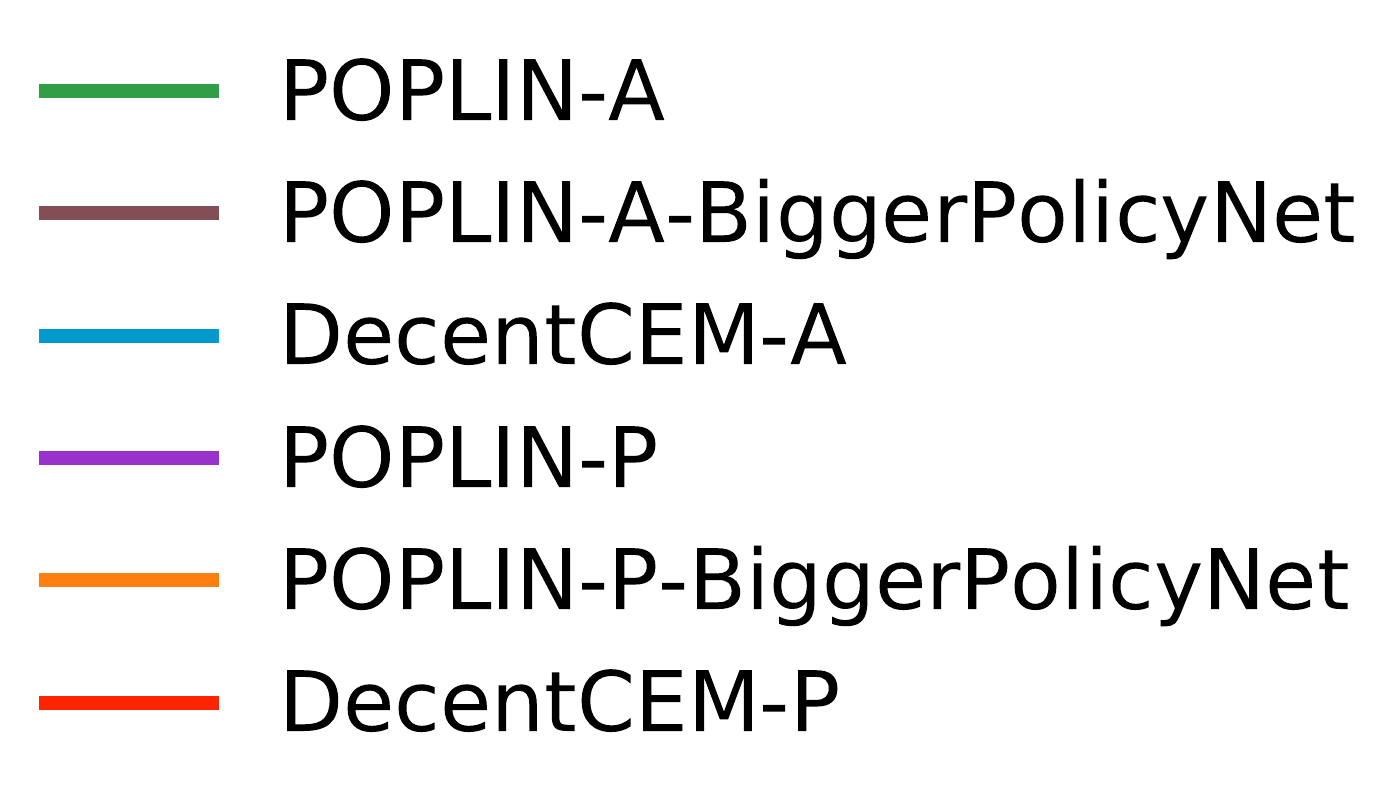}}
	    \end{tabular}
	\vspace{-4mm}
	\caption{\footnotesize Ablation study on the policy network size where
	\textit{POPLIN-A\&P} have a bigger policy network equivalent in the total number of network weights to their \textit{DecentCEM} counterparts. \normalsize} 
	\vspace{-1mm}
	\label{fig:bignet_pc}
\end{figure}

Next, we look into the impact of ensemble size. Fig.~\ref{fig:ensemble-size} shows the learning curves of different ensemble sizes in the Reacher environments for action space planning (left), parameter space planning (middle) and planning without policy networks (right).
Since we fix the total number of samples the same across the methods, the larger the ensemble size is, the fewer samples that each instance has access to. As the ensemble size goes up, we are trading off the accuracy of the sample mean for better exploration by more instances. 
Varying this hyperparameter allows us to find the best trade-off,
as can be observed from the figure %
where increasing the ensemble size beyond a sweet spot yields diminishing or worse returns.
\vspace{-5pt}
\begin{figure}[h]
	\centering
    	\begin{tabular}{ccc}
	{\includegraphics[height=3.7cm,trim=0.27cm 0 0.28cm 0, clip]{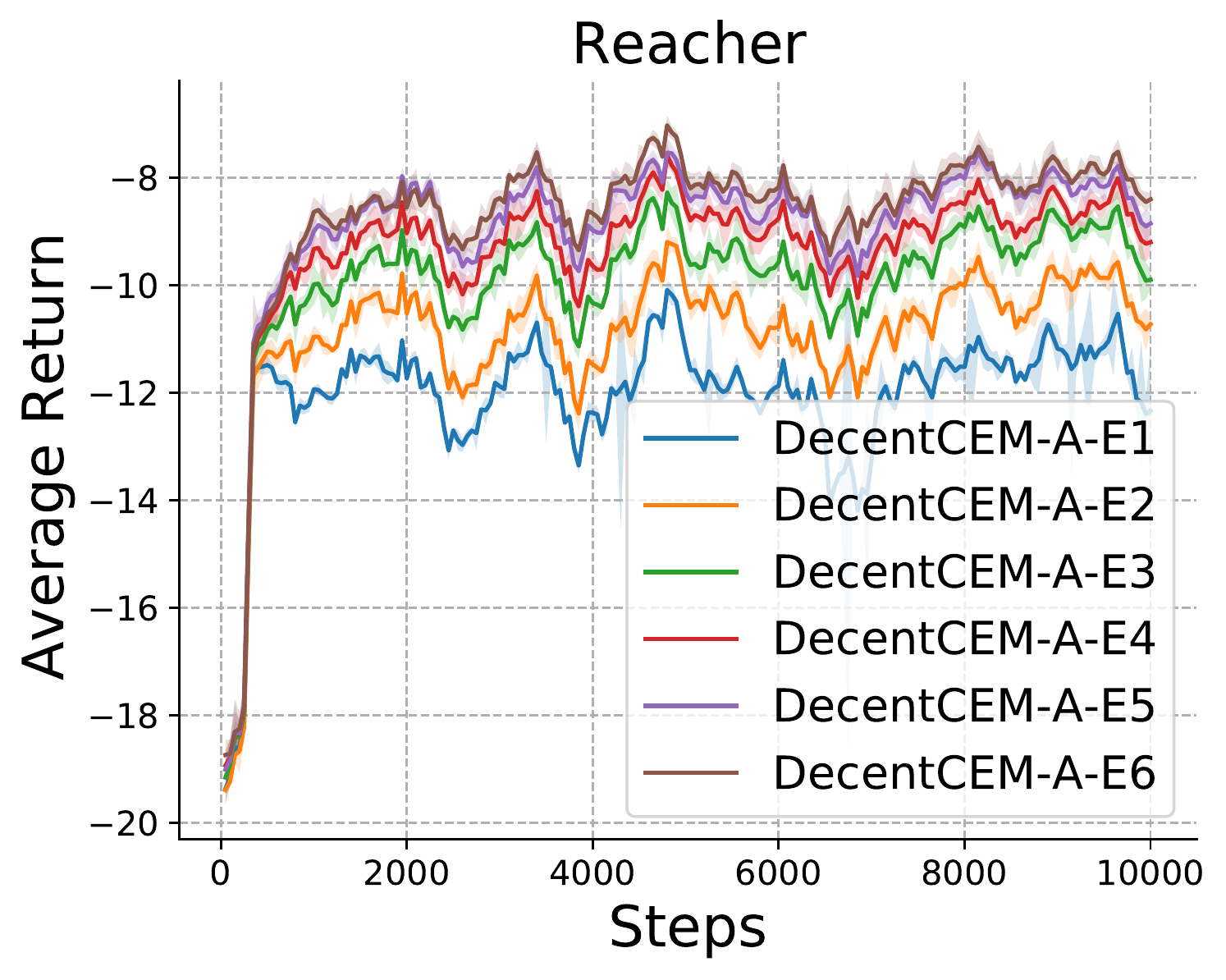}}	
    & {\includegraphics[height=3.7cm,trim=0.35cm 0 0.28cm 0, clip]{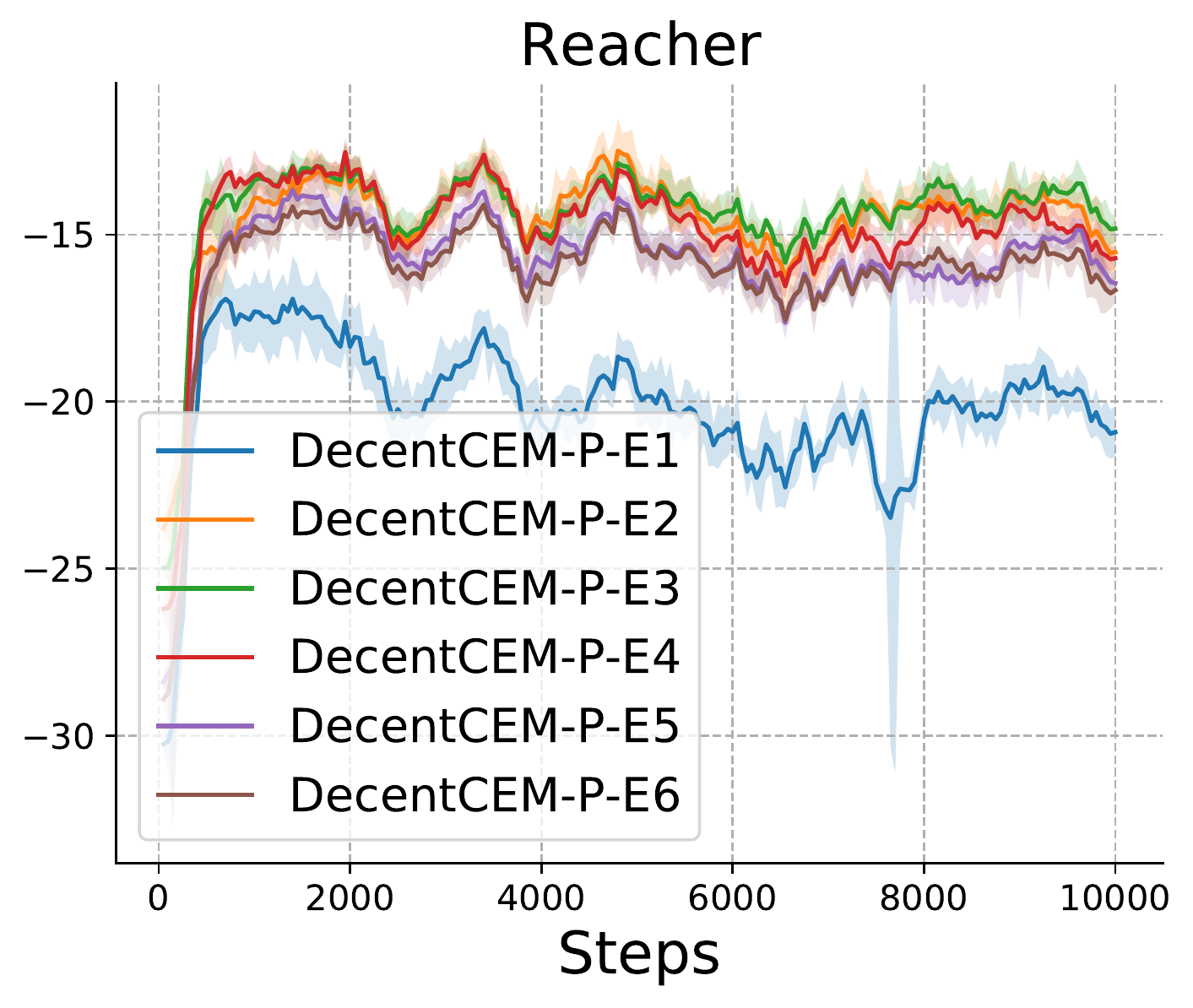}}
    & {\includegraphics[height=3.7cm,trim=0.35cm 0 0.28cm 0, clip]{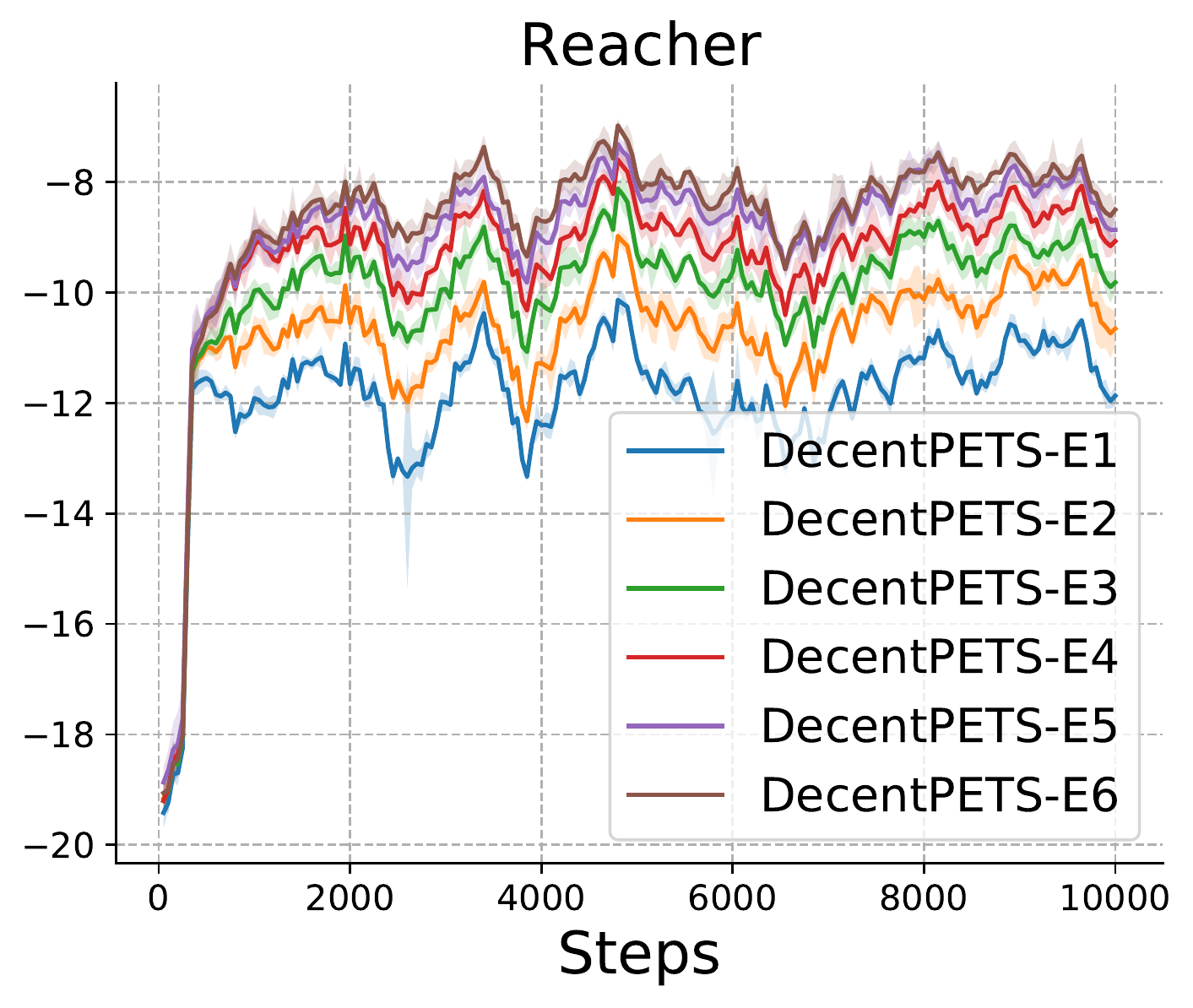}}
	    \end{tabular}
	 \vspace{-2mm}
	\caption{\footnotesize Ablation study on the ensemble size in \emph{DecentCEM-A/P} and \emph{DecentPETS}, e.g. E2 denotes an ensemble with 2 instances. The total number of samples during planning are the same across all variations.\normalsize} 
	\label{fig:ensemble-size}
 	\vspace{-5pt}
\end{figure}

Figure~\ref{fig:selection-ratio} shows the cumulative selection ratio of each CEM instance during training of \textit{DecentCEM-A} with an ensemble size of 5. It suggests that the random initialization of the policy network is sufficient to avoid mode collapse. 
We also plot the actions and pairwise action distances of the instances in
Figure~\ref{fig:action-stats} and \ref{fig:action-pair}.
For visual clarity, we show a time segment toward the end of the training rather than all the 10k steps. %
\emph{DecentCEM-A} has maintained enough diversity in the instances even toward the end of the training.
\textit{DecentCEM-P} and \emph{DecentPETS} share similar trends. The plots of their results along with other ablation results are included in Appendix \ref{ap:ab}.
\begin{figure}
  \centering
\subfigure[Selection Ratios]{\includegraphics[height=3.6cm]{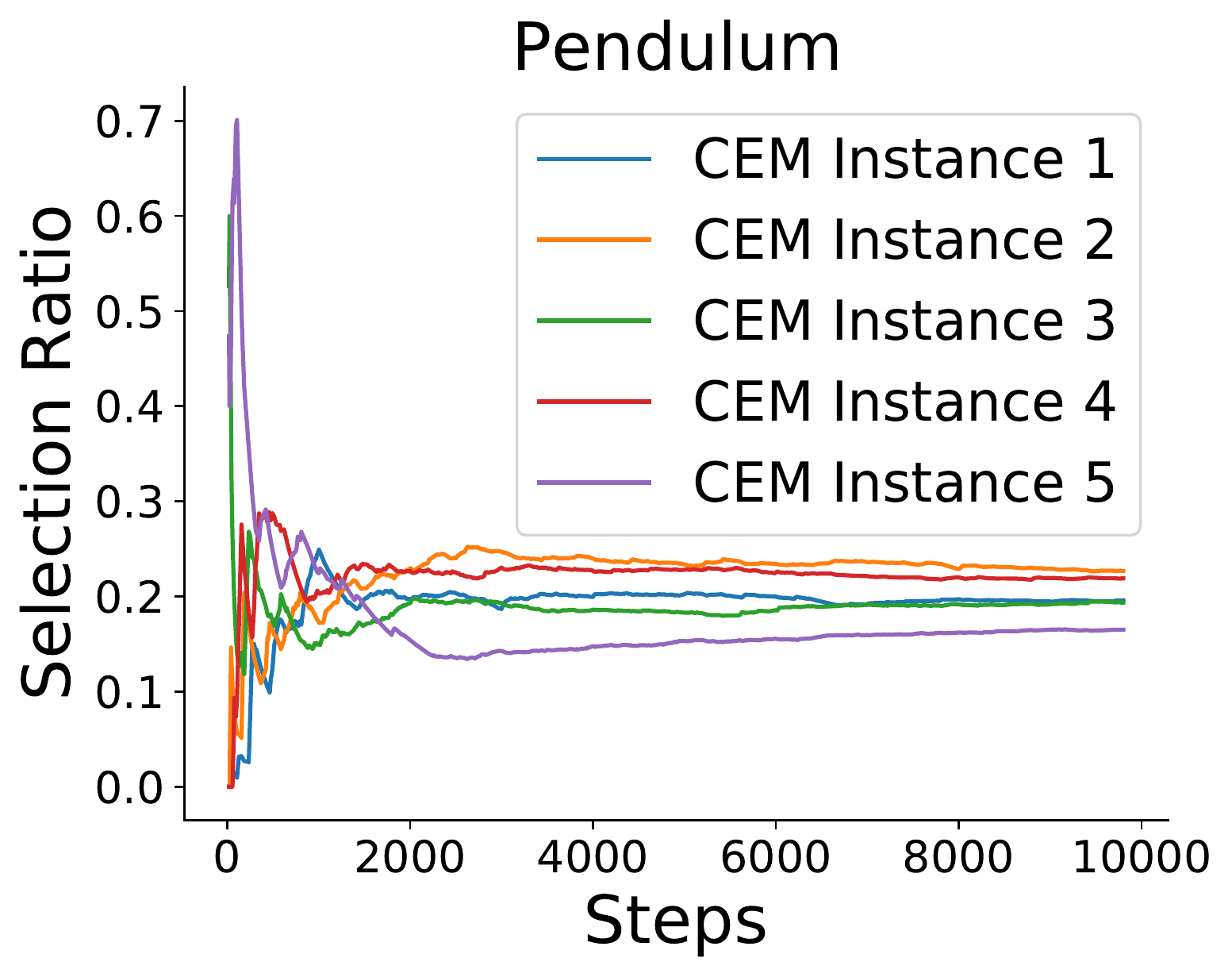}\label{fig:selection-ratio}}    
\subfigure[Action Statistics]{\includegraphics[height=3.6cm]{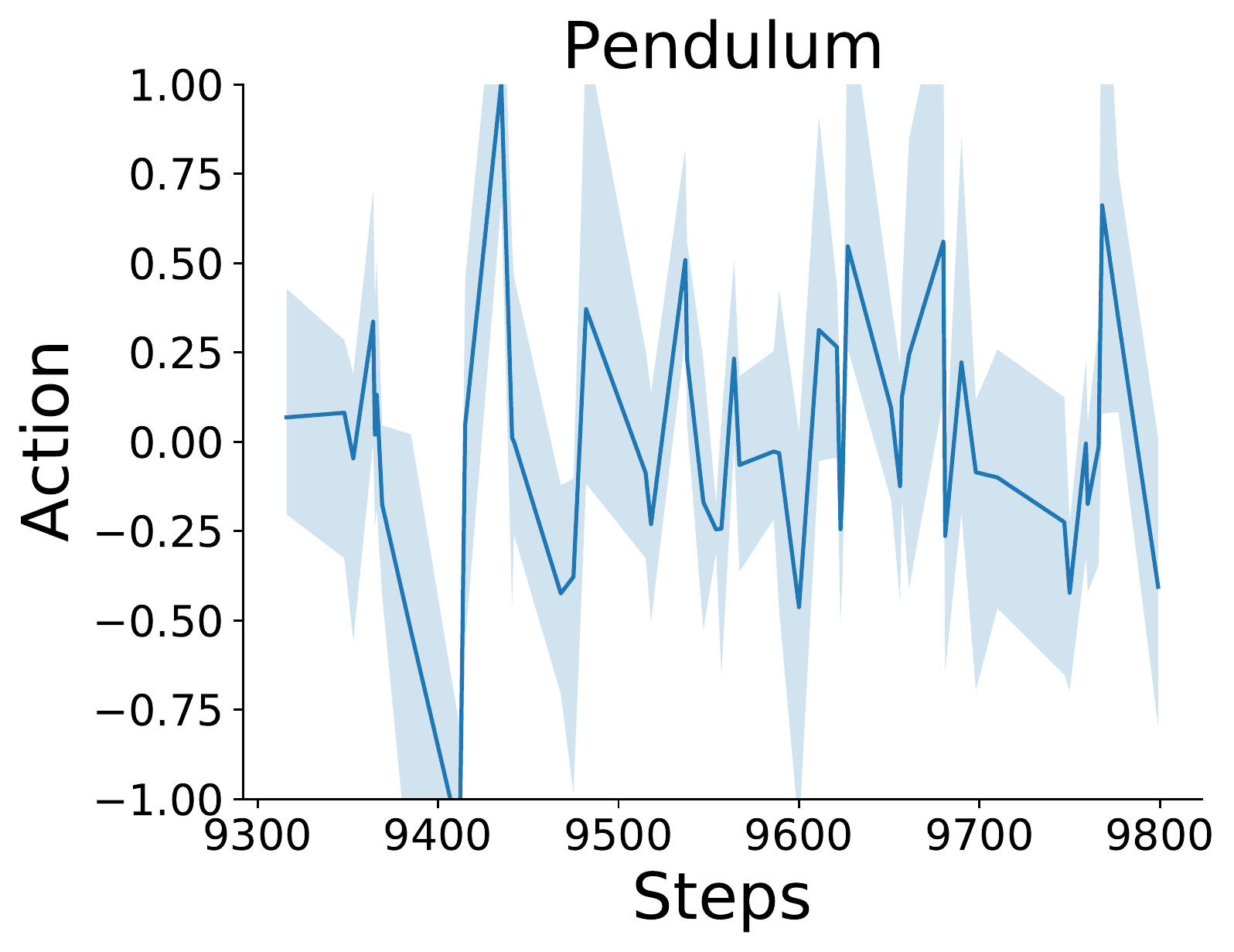}\label{fig:action-stats}}
\subfigure[Action Distance Statistics]{\includegraphics[height=3.6cm]{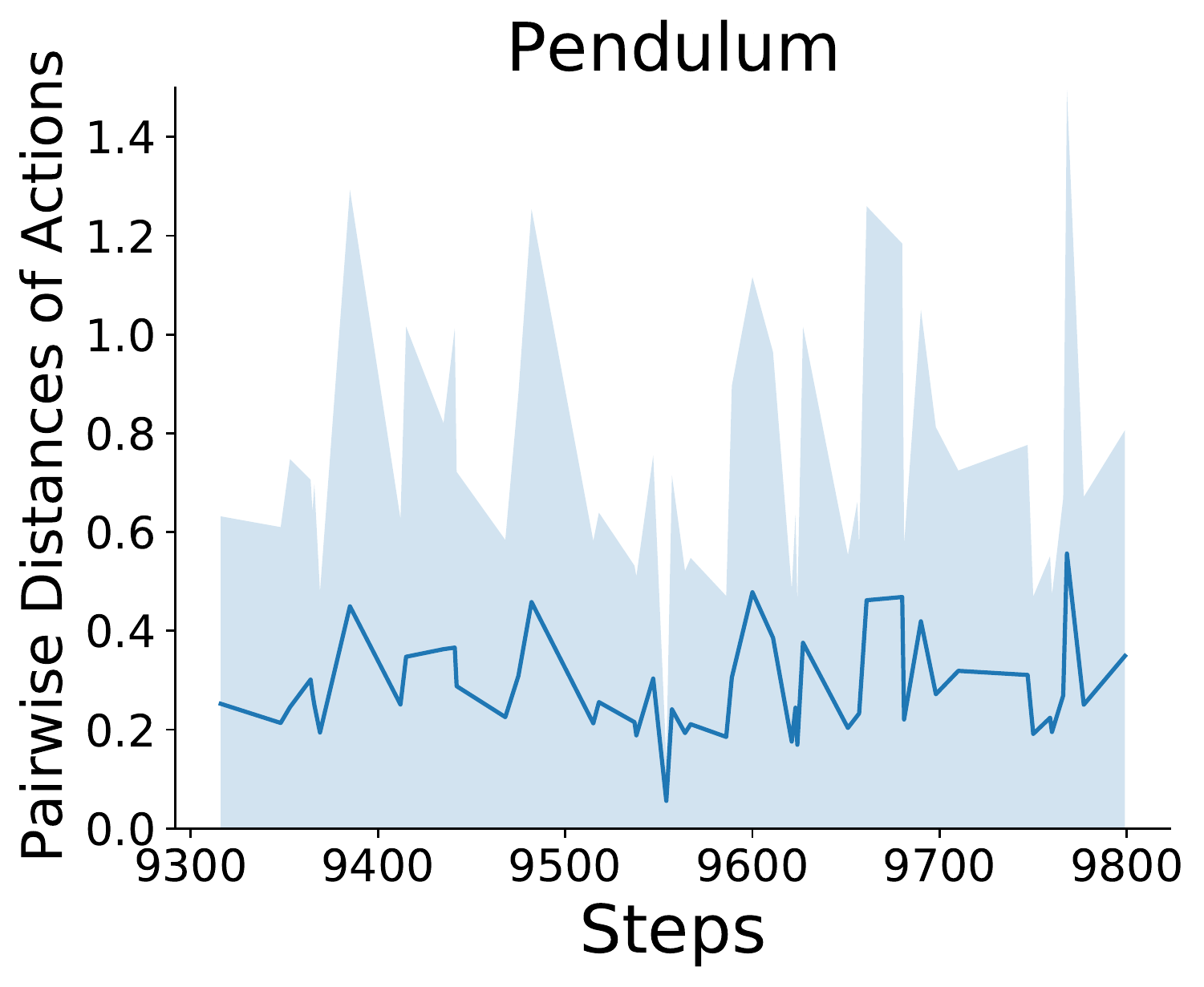} \label{fig:action-pair}}
\caption{\footnotesize Ablation of ensemble diversity in Pendulum during training of \textit{DecentCEM-A} with 5 instances. (a) Cumulative selection ratio of each CEM instance. (b)(c) Statistics of the actions and pairwise action distances of the instances, respectively. The line and shaded region represent the mean and min/max.\normalsize}
 \label{fig:ensemble-diversity}
\end{figure}

\section{Conclusion and Future Work}
\label{sec:conclusion}

In this paper, we study CEM planning in the context of continuous-action MBRL.
We propose a novel \textit{decentralized} formulation of CEM named \textit{DecentCEM}, which generalizes CEM to run multiple independent instances and recovers the conventional CEM when the number of instances is one.
We illustrate the strengths of the proposed DecentCEM approach in a motivational one-dimensional optimization task and show how it fundamentally differs from the CEM approach that uses a Gaussian or GMM.
We also show that DecentCEM has almost sure convergence to a local optimum.
We extend the proposed approach to MBRL by plugging in the \textit{decentralized} CEM into three previous CEM-based methods: \textit{PETS}, \textit{POPLIN-A}, \textit{POPLIN-P}.
We show their efficacy in benchmark control tasks and ablations studies.

There is a gap between the convergence result and practice where the theory assumes that the number of samples grow polynomially with the iterations whereas a constant sample size is commonly used in practice including our work. 
Investigating the convergence properties of \textit{CEM} under a constant sample size makes an interesting direction for future work. Another interesting direction to pursue is finite-time analysis of CEM under both \textit{centralized} and \textit{decentralized} formulations.
In addition, the implementation has room for optimization: the instances currently run serially but can be improved by a parallel implementation to take advantage of the parallelism of the ensemble.

\begin{ack}

We would like to thank the anonymous reviewers for their time and valuable suggestions.
Zichen Zhang would like to thank Jincheng Mei for helpful discussions in the convergence analysis and Richard Sutton for raising a question about the model learning.
This work was partially done during Zichen Zhang's internship at Huawei Noah's Ark Lab.
Zichen Zhang gratefully acknowledges the financial support by an NSERC CGSD scholarship and an Alberta Innovates PhD scholarship. He is thankful for the compute resources generously provided by Digital Research Alliance of Canada (and formerly Compute Canada), which is sponsored through the accounts of Martin Jagersand and Dale Schuurmans.
Dale Schuurmans gratefully acknowledges funding from the Canada CIFAR AI Chairs Program, Amii and NSERC.

\end{ack}

\bibliography{rl}
\bibliographystyle{abbrvnat}

\medskip

\small

\section*{Checklist}

\begin{enumerate}

\item For all authors...
\begin{enumerate}
  \item Do the main claims made in the abstract and introduction accurately reflect the paper's contributions and scope?
    \answerYes{}
  \item Did you describe the limitations of your work?
    \answerYes{} Limitations are discussed in Sec.~\ref{sec:conclusion} and Appendix~\ref{ap:overhead}.
  \item Did you discuss any potential negative societal impacts of your work?
    \answerNA{}
  \item Have you read the ethics review guidelines and ensured that your paper conforms to them?
    \answerYes{}
\end{enumerate}

\item If you are including theoretical results...
\begin{enumerate}
  \item Did you state the full set of assumptions of all theoretical results?
    \answerYes{}
	\item Did you include complete proofs of all theoretical results?
    \answerYes{}
\end{enumerate}

\item If you ran experiments...
\begin{enumerate}
  \item Did you include the code, data, and instructions needed to reproduce the main experimental results (either in the supplemental material or as a URL)?
    \answerYes{}
  \item Did you specify all the training details (e.g., data splits, hyperparameters, how they were chosen)?
    \answerYes{} See Section \ref{sec:exp}, Appendix A, B and D.
	\item Did you report error bars (e.g., with respect to the random seed after running experiments multiple times)?
    \answerYes{} See Section \ref{sec:exp} and Appendix A.
	\item Did you include the total amount of compute and the type of resources used (e.g., type of GPUs, internal cluster, or cloud provider)?
    \answerYes{} See section \ref{sec:benchmark}.
\end{enumerate}

\item If you are using existing assets (e.g., code, data, models) or curating/releasing new assets...
\begin{enumerate}
  \item If your work uses existing assets, did you cite the creators?
    \answerYes{}
  \item Did you mention the license of the assets?
    \answerNA{} The assets (Gym, mujoco) are cited and well known.%
  \item Did you include any new assets either in the supplemental material or as a URL?
    \answerNo{}
  \item Did you discuss whether and how consent was obtained from people whose data you're using/curating?
    \answerNA{}
  \item Did you discuss whether the data you are using/curating contains personally identifiable information or offensive content?
    \answerNA{}
\end{enumerate}

\item If you used crowdsourcing or conducted research with human subjects...
\begin{enumerate}
  \item Did you include the full text of instructions given to participants and screenshots, if applicable?
    \answerNA{}
  \item Did you describe any potential participant risks, with links to Institutional Review Board (IRB) approvals, if applicable?
    \answerNA{}
  \item Did you include the estimated hourly wage paid to participants and the total amount spent on participant compensation?
    \answerNA{}
\end{enumerate}

\end{enumerate}

\clearpage
\newpage
\appendix
\numberwithin{table}{section}

\section*{Appendix}
\numberwithin{table}{section}

\section{Details of the Motivational Example}
\label{ap:motivation}

\subsection{Setup and Running Time}
For a fair comparison of the three methods \textit{CEM}, \textit{CEM-GMM} and \textit{DecentCEM}, we performed a hyperparameter search for all.
The list of hyperparmeters are summarized in Table \ref{tb:motivation-hyperparams} and the best performing hyperparameters for each method under each population size are shown in Table \ref{tb:motivation-best-hyperparam}.
These hyperparameters were what the data in Fig.~\ref{fig:motivation-1d} were based on.
Note that the
top percentage of samples ``Elite Ratio'' (in Table A.1) was used in the implementation instead of top-$k$ but they are equivalent.
The running time are included in Table \ref{tb:motivation-runningtime}.

\begin{table*}[htbp]
 \footnotesize
  \caption{Hyperparameters}
  \label{tb:motivation-hyperparams}
  \centering
 \begin{tabular}{lll}
    \toprule
    Algorithm & Parameter & Value \\
    \midrule
Shared   &  Total Sample Size & {100, 200, 500, 1000}\\
Parameters  &    Elite Ratio & 0.1 \\
 &  $\alpha$: Smoothing Ratio   & 0.1 \\
 &   $\epsilon$: Minimum Variance Threshold  & 1e-3 \\
 & Maximum Number of Iterations & 100 \\
    \midrule
CEM-GMM &        $M$: Number of Mixture Components & {3,5,8,10}\\
   &  $\kappa$: Weights of Entropy Regularizer  & 0.25, 0.5 \\
 &    $r$: Return Mode  & `s': mean of the mixture component \\
 & & \ \ \ \ \ \ \  sampled  based on their weights\\
 & & `m': mean of the component that  \\
 & & \ \ \ \ \ \ \ achieves the minimum cost \\
    \midrule 
DecentCEM &    $E$: Number of Instances in the Ensemble & {3,5,8,10} \\
    \bottomrule
  \end{tabular}
  \normalsize
\end{table*}

\begin{table}[htbp]
 \footnotesize
  \caption{Best Hyper-Parameter}
  \label{tb:motivation-best-hyperparam}
  \centering
 \begin{tabular}{lllll}
    \toprule
     & \multicolumn{4}{c}{Total Sample Size } \\
     & 100 & 200 & 500 & 1000 \\ %
    \midrule
    CEM-GMM & $M=10$ & $M=8$ & $M=8$ & $M=8$\\
    & $\kappa=0.25$, & $\kappa=0.5$ & $\kappa=0.25$ & $\kappa=0.5$ \\
    & $r=$`m' & $r=$`m' & $r=$`m'& $r=$`s' \\
    \midrule 
    DecentCEM & $E=10$ & $E=10$ & $E=10$ & $E=8$ \\
    \bottomrule
  \end{tabular}
  \normalsize
\end{table}

\begin{table}[htbp]
  \footnotesize
  \caption{Total Time of 10 Runs (in seconds)}
  \label{tb:motivation-runningtime}
  \centering
 \begin{tabular}{lllll}
    \toprule
     & \multicolumn{4}{c}{Total Sample Size } \\
     & 100 & 200 & 500 & 1000 \\ %
    \midrule
    CEM  & 0.079 & 0.093 & 0.165 & 0.318 \\
    CEM-GMM  & 7.322 & 11.500 & 24.431 & 59.844 \\
    DecentCEM  & 0.407 & 0.420 & 0.506 & 0.545 \\
    \bottomrule
  \end{tabular}
  \normalsize
\end{table}

\subsection{Output of CEM Approaches}
\label{ap:cem-output}
In terms of the output of CEM approaches, there exist different options in the literature. 
The most common option is to return the sample in the domain that corresponds to the highest probability density in the final sampling distribution. It is the mean in the case of Gaussian and the mode with the highest probability density in the case of GMM. 
One can also draw a sample from the final sampling distribution \citep{okada2020variational} and return it.
Another option is to return the best sample observed \citep{pinneri2020sampleefficient}.
The best option among the three may be application dependent. It has been observed that in many applications, the sequence of sampling distributions numerically converges to a deterministic one \citep{CEM}, in which case the first two options are identical.

\begin{table*}[tbh]
\footnotesize
  \caption{The setup of the environments. The number in the bracket in the ``Environment'' column denotes the source of this environment: [1] refers to the benchmark paper from \citet{mbbl}; [2] denotes PETS \citep{PETS}.
  In the reward functions,
  $\mathbf{d}_t$ denotes the vector between the end effector to the target position. $z_t$ denotes the height of the robot.
  $\| \mathbf{v} \|_1 $ and $\| \mathbf{v} \|_2 $ denote the 1-norm and 2-norm of vector $\mathbf{v}$, respectively.
  In PETS-Pusher, $\mathbf{d}_{1,t}$ is the vector between the object position and the goal and $\mathbf{d}_{2,t}$ denotes the vector between the object position and the end effector.}
  \label{tb:env}
 \centering
  \begin{tabular}{lllll}
    \toprule 
    Environment & $|S|$ & $|A|$ & Episode Length & Reward Function \\
    \midrule 
    Pendulum & 3 & 1 & 200 &  $\theta_t ^2 + 0.1 \dot{\theta_t} ^2 + 0.001  a_t ^2 $  \\
    InvertedPendulum [1] & 4 & 1 & 100 & $-\theta_t ^2$    \\
    Cartpole [1] & 4 & 1 & 200 &  $\cos \theta_t - 0.01 x_t^2$    \\
    Acrobot [1] & 6 & 1 & 200 &  $-\cos\theta_{1,t} - \cos(\theta_{1,t} + \theta_{2,t})$  \\
    FixedSwimmer [1] & 9 & 2 & 1000 & $\dot{x}_t - 0.0001 \| \mathbf{a}_t\|_2^2 $    \\
    Reacher [1] & 11 & 2 & 50 &  $ -\|\mathbf{d}_t\| - \| \mathbf{a}_t\|_2^2 $   \\
    Hopper [1] & 11 & 3 & 1000  &   $ \dot{x}_t - 0.1\| \mathbf{a}_t\|_2^2 - 3(z_t - 1.3)^2 $   \\
    Walker2d [1] & 17 & 6 & 1000 &  $ \dot{x}_t - 0.1\| \mathbf{a}_t\|_2^2 - 3(z_t - 1.3)^2 $    \\
    HalfCheetah [1] & 17 & 6 & 1000  &  $ \dot{x}_t - 0.1\| \mathbf{a}_t\|_2^2$  \\
    PETS-Reacher3D [2] & 17 & 7 & 150  &  $ - \|\mathbf{d}_t\|^2_2 - 0.01\| \mathbf{a}_t\|_2^2$  \\
    PETS-HalfCheetah [2] & 18 & 6 & 1000  &  $ \dot{x}_t - 0.1\| \mathbf{a}_t\|_2^2$  \\
    PETS-Pusher [2] & 20 & 7 & 150  & $ -1.25 \| \mathbf{d}_{1,t} \|_1  - 0.5 \| \mathbf{d}_{2,t} \|_1  - 0.1 \| \mathbf{a_t}\|_2^2$   \\
    Ant [1] & 27 & 8 & 1000 & $ \dot{x}_t - 0.1\| \mathbf{a}_t\|_2^2 -3(z_t-0.57)^2 $    \\
    \bottomrule
  \end{tabular}
\normalsize
\end{table*}

\section{Benchmark Environment Details}
\label{ap:env}In this section, we go over the details of the benchmark environments used in the experiments.

Table~\ref{tb:env} lists the environments along with their properties, including the dimensionality of the state $|S|$ and action spaces $|A|$, the maximum episode length.
as well as the reward function.
Whenever possible, we reuse the original implementations from the literature as noted in Table~\ref{tb:env} so as to avoid confusion.
The environments that start with ``PETS'' are from the PETS paper \citep{PETS}, which is one of the baseline methods. Most other environments are from \cite{mbbl} where the dynamics are the same as the OpenAI gym version and the reward function in Table~\ref{tb:env} is exposed to the agent.
For more details of the environments, the readers are referred to the original paper.

Note that FixedSwimmer is a modifed version of the original Gym Swimmer environment where the velocity sensor on the neck is moved to the head. This fix was originally proposed by \cite{poplin}. For the Pendulum environment, we use the OpenAI Gym version. The modified version in \citep{mbbl} uses a different reward function which we have found to be incorrect.

\vspace{-5pt}
\section{Algorithms}
\label{ap:alg}
In this section, we give the pseudo-code of the proposed algorithms DecentCEM-A and DecentCEM-P in Algorithm \ref{alg:DecentCEM-A} and \ref{alg:DecentCEM-P} respectively.
We only show the training phase. The algorithm at inference time is simply the same process without the data saving and network update. For the internal process of CEM, we refer the readers to \cite{CEM, poplin}.

\begin{algorithm}[hp]
\DontPrintSemicolon
\SetNoFillComment
 Initialize the policy networks $p_i$ with $\theta_i, i = {1,2,\cdots, M}$ where $M$ is the ensemble size. Planning horizon $H$. 
 Initialize the dynamics network $f_\omega$ parameterized by $\omega$. 
 Empty Datasets $D_m$ and $D_p$\\
  
 \tcp {Episode 1, warmup phase}
 Rollout using a random policy, fill the dataset $D_m$ with the transition data $\{(s_t, a_t, s_{t+1})\}$ \\
Update $\omega$ using $D_m$ by Mean-Squared Loss \tcp* {Train the dynamics network with $D_m$}
\tcp {Episode 2 onwards}
  \Repeat{bored} 
  {
    $t=0$, $D_p=\{\}$ \tcp*{Each episode, reset time and dataset}
    \Repeat{Either reached the maximum episode length or terminal state} 
    {
        \ForEach{ policy network $p_i$ in the ensemble}  
            {
            generate reference mean of action sequence distribution $\mu_i$ using $p_i$ and the model $f_\omega$.\\
            \tcp {Apply CEM to refine the action distribution.} 
            \tcp {$\hat{\mu}_i$, $v_i$ are the mean action sequence of the refined distribution and its expected value}
            $\hat{\mu}_i, v_i = \text{CEM}(\mu_i) $\\
            $\hat{a}_{t,i} = \hat{\mu}_i[0]$
            
           }
           $a_t = \argmax_{\hat{a}_{t,i}} v_i$\tcp*{Pick best distribution}
        $s_{t+1} = step(a_t)$ 
        \tcp*{Execute the first action in the mean sequence}
        Append the transition $(s_t, a_t, s_{t+1})$ to $D_m$ \\
        Append the data $\{(s_t, \hat{a}_{t,i})\}_{i=1}^{M}$ to $D_p$ \\
        $ t = t+1 $ \tcp*{Update time step}
    }
    Update the model parameter $\omega$ using dataset $D_m$ \\
    Update the policy network weights $\{\theta_i\}_{i=1}^{M}$  using dataset $D_p$ by the behavior cloning objective \\%Eq.~\eqref{eq:bc} \\
 }
 \caption{DecentCEM-A Training}
\label{alg:DecentCEM-A}
\end{algorithm}
\begin{algorithm}[hp]
\DontPrintSemicolon
\SetNoFillComment

  Initialize the policy networks $p_i$ with $\theta_i, i = {1,2,\cdots, M}$ where $M$ is the ensemble size. Planning horizon $H$.
 Initialize the dynamics network $f_\omega$ parameterized by $\omega$. 
 Empty Datasets $D_m$ and $D_p$\\
  
\tcp {Episode 1, warmup phase}
 Rollout using a random policy, fill the dataset $D_m$ with the transition data $\{(s_t, a_t, s_{t+1})\}$ \\
Update $\omega$ using $D_m$ by Mean-Squared Loss
\tcp* {Train the dynamics network with $D_m$}

\tcp {Episode 2 onwards}
  \Repeat{bored} 
  {
    $t=0$, $D_p=\{\}$ \tcp*{Each episode, reset time and dataset}
    \Repeat{either reached the maximum episode length or terminal state} 
    {
        \ForEach{ policy network $p_i$ in the ensemble}  
            {
            \tcp {Apply CEM to refine the distribution of the neural network weight.} 
            \tcp {$\hat{\mu}_i$, $v_i$ are the mean of the refined weight distribution sequence and its expected value}
            $\hat{\mu}_i, v_i = \text{CEM}(\theta_i) $ \\
            $\delta_i = \hat{\mu}_i[0]$  \tcp*{Keep the weight at the first step and discard the rest}
           }
        $\theta_t= \argmax_{\theta_i + { \delta}_i} v_i$  \tcp*{Pick the best distribution of weight sequence}
        $a_t = p_{\theta_t}(s_t)$ \\
        $s_{t+1} = step(a_t)$ \tcp*{Execute the action returned by the policy network $p_{\theta_t}$}
        Append the transition $(s_t, a_t, s_{t+1})$ to $D_m$ \\
        Append the data $\{ {\delta}_i\}_{i=1}^{M}$ to $D_p$ \\
        $ t = t+1 $ \tcp*{Update time step}
    }
    Update the model parameter $\omega$ using dataset $D_m$ \\
    Update the policy network weights $\{\theta_i\}_{i=1}^{M}$  using dataset $D_p$ by the AVG training objective\\%Eq.~\eqref{eq:avg} \\
 }

\caption{DecentCEM-P Training}
\label{alg:DecentCEM-P}
\end{algorithm}

\vspace{-5pt}
\section{Implementation Details}
\label{ap:implementation}

\subsection{Reproducibility}
\label{ap:random}
Our implementation is fully reproducible by identifying the sources of randomness and controlling the random seeds as summarized in Table~\ref{tb:randomseed}. The seeds are set once at the beginning of the experiments.

\begin{table}[htbp]
 \footnotesize
  \caption{Random Seed. The set \{1,2,3,4,5\} refers to the seeds for five runs. Note that we control the random seed for the environments since there is a random number generator in openai gym environments independent from other sources}
  \label{tb:randomseed}
 \centering
 \begin{tabular}{ll}
    \toprule
    Source of randomness & Random Seed \\
    \midrule
     Tensorflow & \multirow{3}*{\{1,2,3,4,5\}}   \\
    \cline{1-1}
    numpy &  \\
    \cline{1-1}
    python random module &  \\
    \hline
    the training environment & 1234 \\
    the evaluation environment & 0 \\
    \bottomrule
  \end{tabular}
  \vspace{-5pt}
\normalsize
\end{table}

\subsection{Hyperparameters}
\label{ap:hyperparam}

\begin{table}[h]
 \footnotesize
  \caption{Hyperparameters of SAC}
  \label{tb:sac-param}
 \centering
  \begin{tabular}{ll}
    \toprule
    Parameter & Value \\
    \midrule
    Actor learning rate & 0.0001 \\
    Critic learning rate & 0.0001 \\
    Actor network architecture & [$|S|$, 64, 64, 2$\times$ $|A|$]\\
    Critic network architecture & [$|S|+|A|$, 64, 64, 1]\\
    \bottomrule
  \end{tabular}
\normalsize
 \vspace{5pt}
 \footnotesize
  \caption{Hyperparameters of PETS}
  \label{tb:pets-param}
  \begin{tabular}{ll}
    \toprule
    Parameter & Value \\
    \midrule
    Model learning rate & 0.001 \\
    Warmup episodes & 1 \\
    Planning Horizon & 30 \\
    CEM population size & 500 (400 in PETS-reacher3D) \\
    CEM proportion of elites & $10\%$ \\
    CEM initial distribution variance & 0.25 \\
    CEM max \# of internal iterations & 5 \\
    \bottomrule
  \end{tabular}
  \normalsize
 \vspace{5pt}
 \footnotesize
  \caption{Hyperparameters of POPLINA and POPLINP}
  \label{tb:poplin-param}
  \begin{tabular}{ll}
    \toprule
    Parameter & Value \\
    \midrule
    Model learning rate & 0.001 \\
    Warmup episodes & 1 \\
    Planning Horizon & 30 \\
    CEM population size & 500 (400 in PETS-reacher3D) \\
    CEM proportion of elites & $10\%$ \\
    CEM initial distribution variance & 0.25 \\
    CEM max \# of internal iterations & 5 \\
    Policy network architecture (A) & [$|S|$, 64, 64, $|A|$]\\
    Policy network architecture (P) & [$|S|$, 32, $|A|$]\\
    Policy network learning rate & 0.001 \\
    Policy network activation function & tanh \\
    \bottomrule
  \end{tabular}
\vspace{-20pt}
\end{table}
\normalsize

This section includes the details of the key hyperparameters used in the baseline algorithms \textit{PETS} (Table \ref{tb:pets-param}), \textit{POPLIN-A/P} (Table \ref{tb:poplin-param}) and \textit{SAC}\footnote{Our SAC implementation used network architectures that are similar to the policy network in our method. The results of our implementation either matches or surpasses the ones reported in \citep{PETS,poplin} and \citep{mbbl}} (Table \ref{tb:sac-param}).
The proposed \textit{DecentCEM} algorithms (\emph{DecentPETS, DecentCEM-A, DecentCEM-P}) have identical hyperparameters as their corresponding baselines (\emph{PETS, POPLIN-A, POPLIN-P}) except for an additional ensemble size parameter.
The hyperparameter search for the ensemble size is performed by sweeping through the set $\{2,3,4,5,6\}$ for each environment. 
For the neural network architecture for the dynamics model, the \textit{DecentCEM} methods exactly follow the original one in \textit{PETS} and \textit{POPLIN} for a fair comparison, which is an ensemble of fully connected networks.

\section{Full Results}
\label{ap:results}
\subsection{Detailed visualization of the iterative updates in the one-dimensional optimization task}
Figure \ref{fig:motivation_1D_full_compare} is a version of Figure \ref{fig:vis1d} with more iterations.
It shows the iterative sampling process of CEM methods in the 1D optimization task and how the sampling distributions evolve over time.

\begin{figure*}[tbp]
    \centering
    \includegraphics[width=0.9\textwidth]{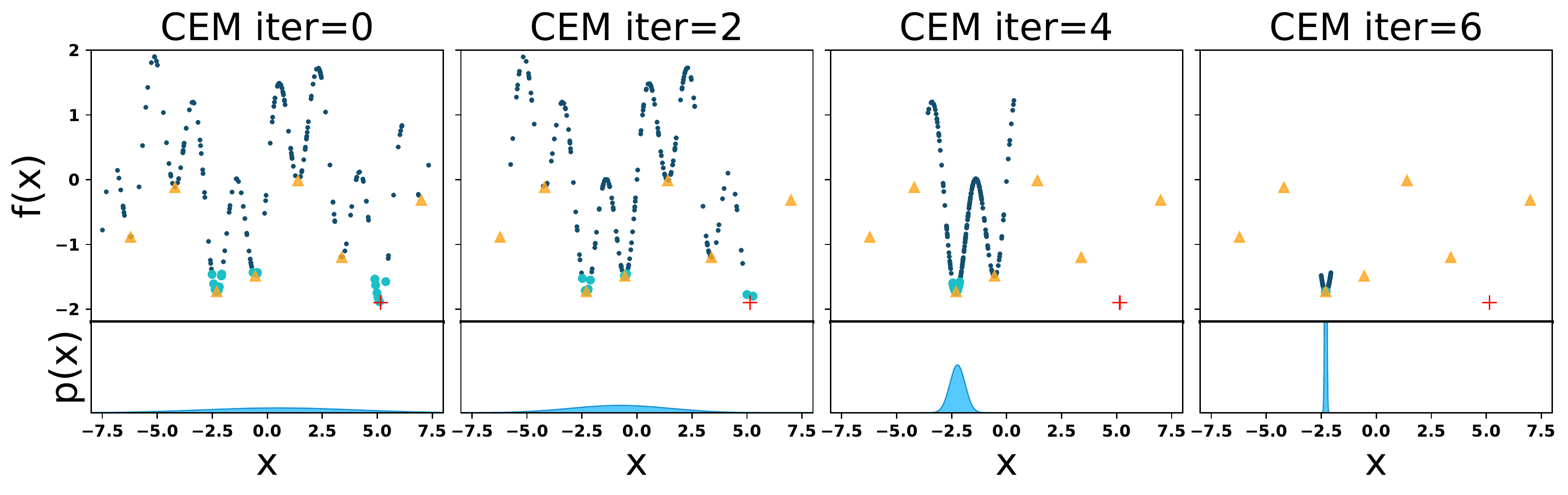}\\
    \includegraphics[width=0.9\textwidth]{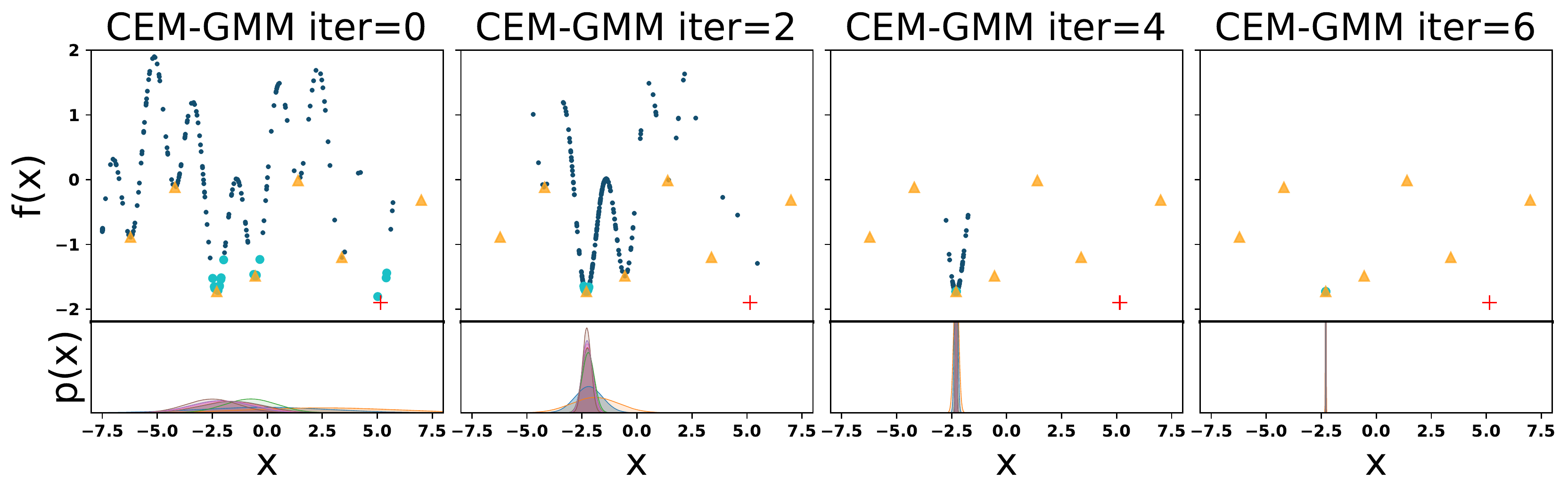}\\
    \includegraphics[width=0.9\textwidth]{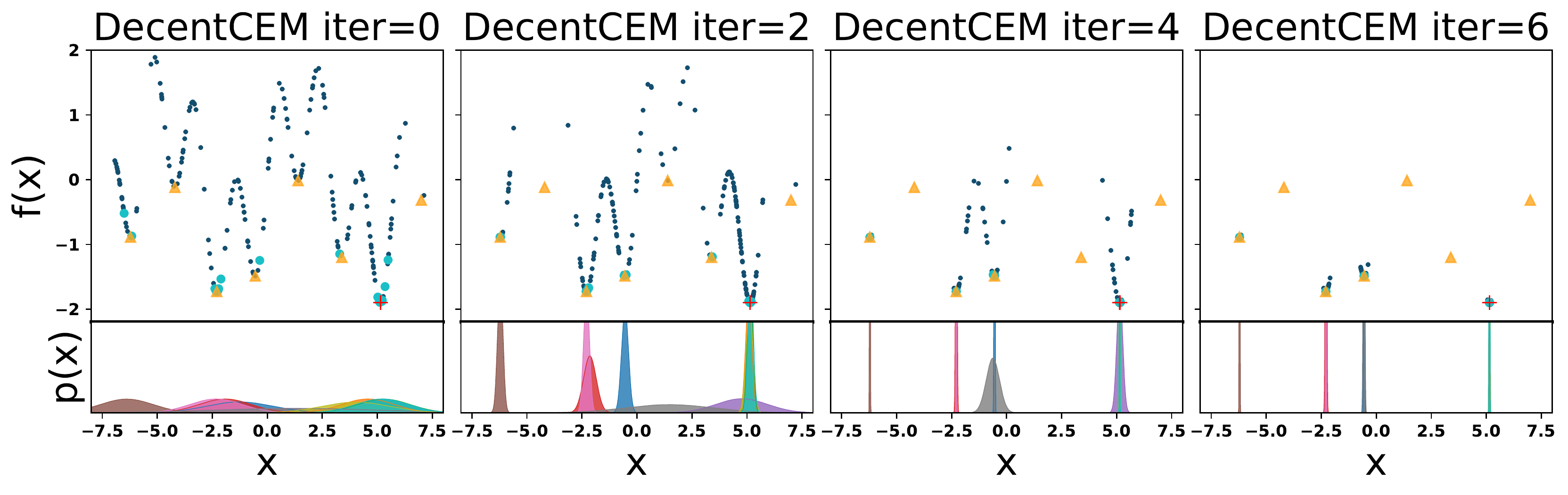}
    \caption{The iterative sampling process in the one-dimensional optimization task.}
    \label{fig:motivation_1D_full_compare}
\end{figure*}

\begin{figure*}[tbp]
	\centering
	\begin{tabular}{@{}p{47mm}@{}p{47mm}@{}p{47mm}}
		{\includegraphics[height=3.7cm]{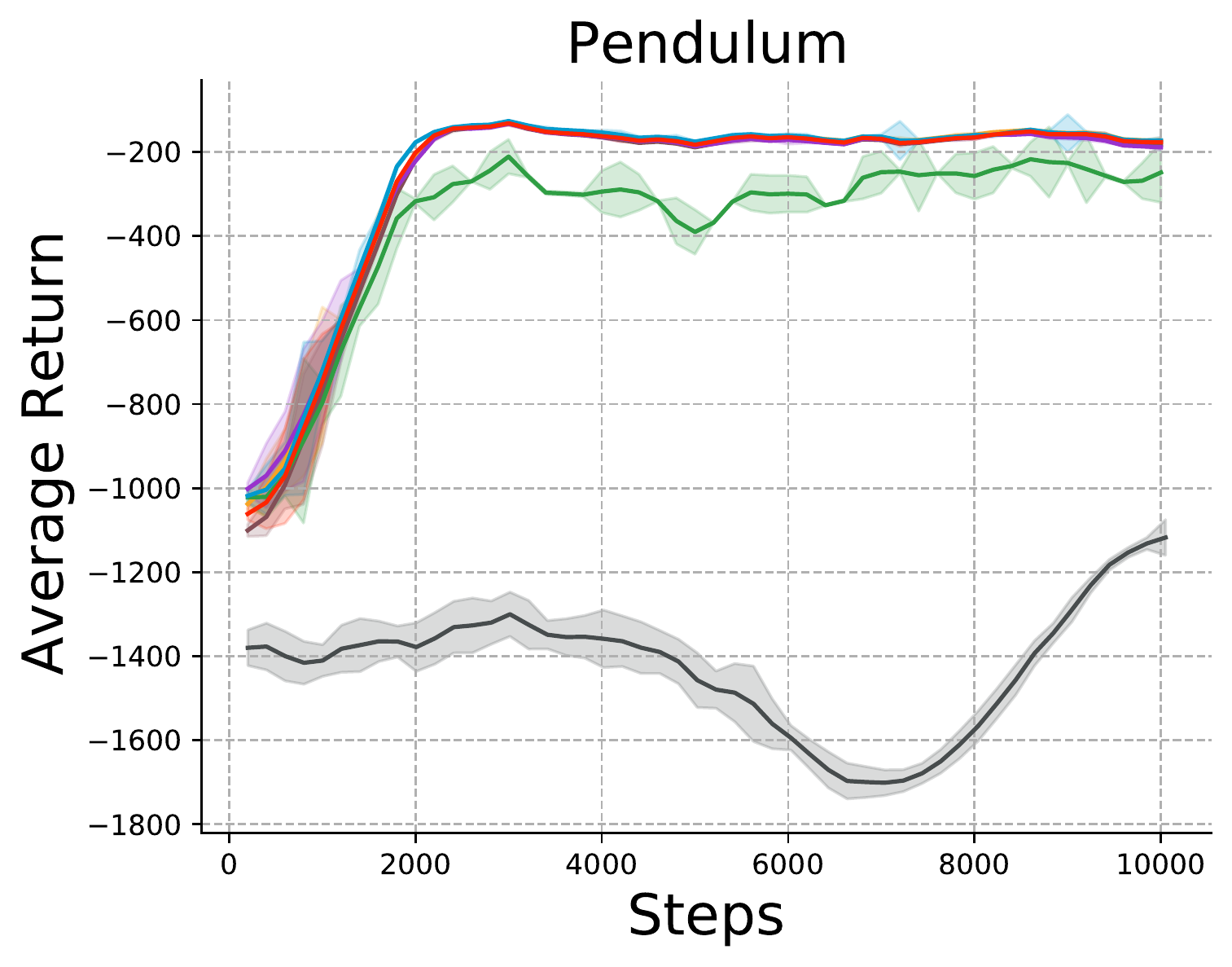}}
		&{\includegraphics[height=3.7cm]{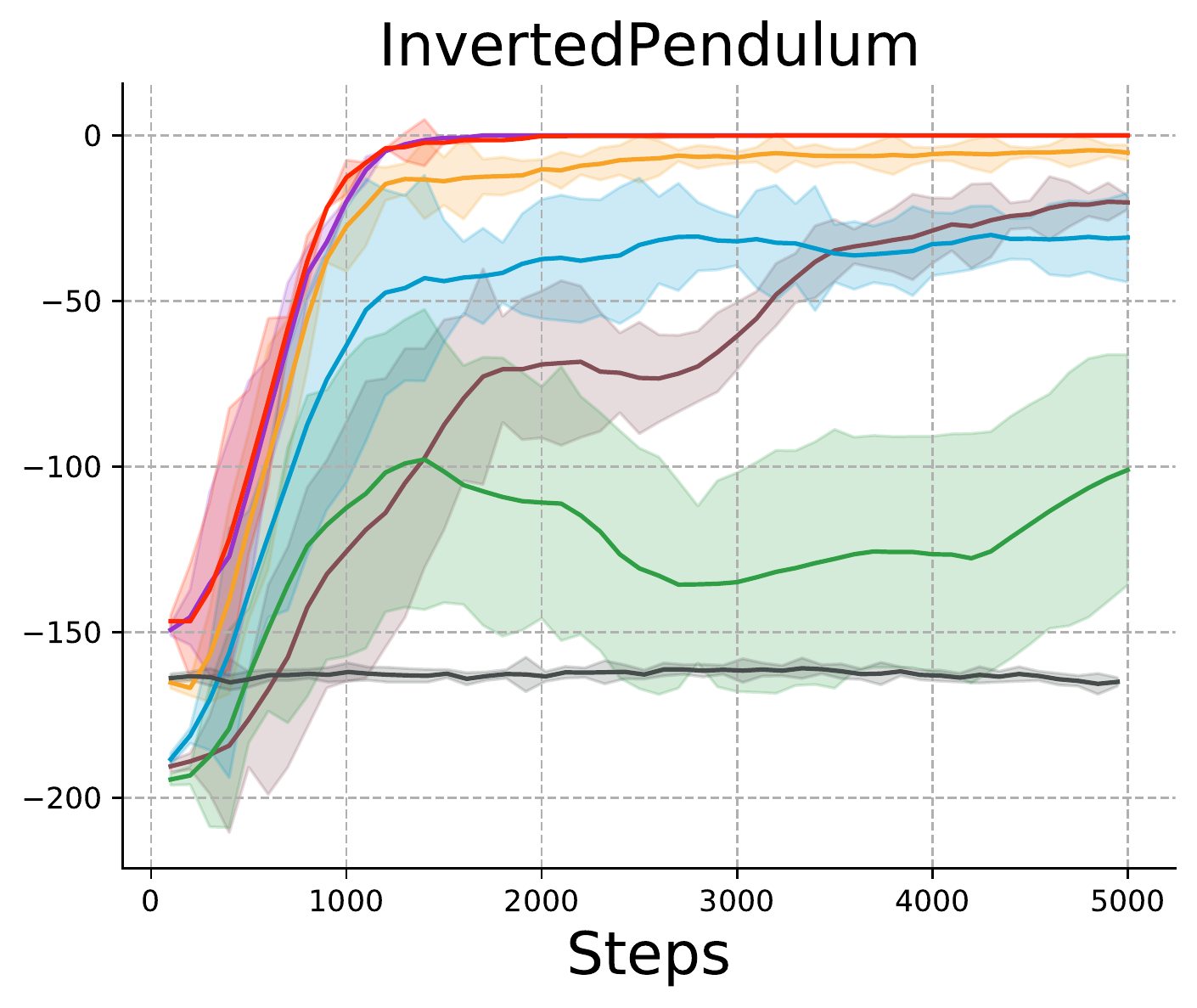}}
		&{\includegraphics[height=3.7cm]{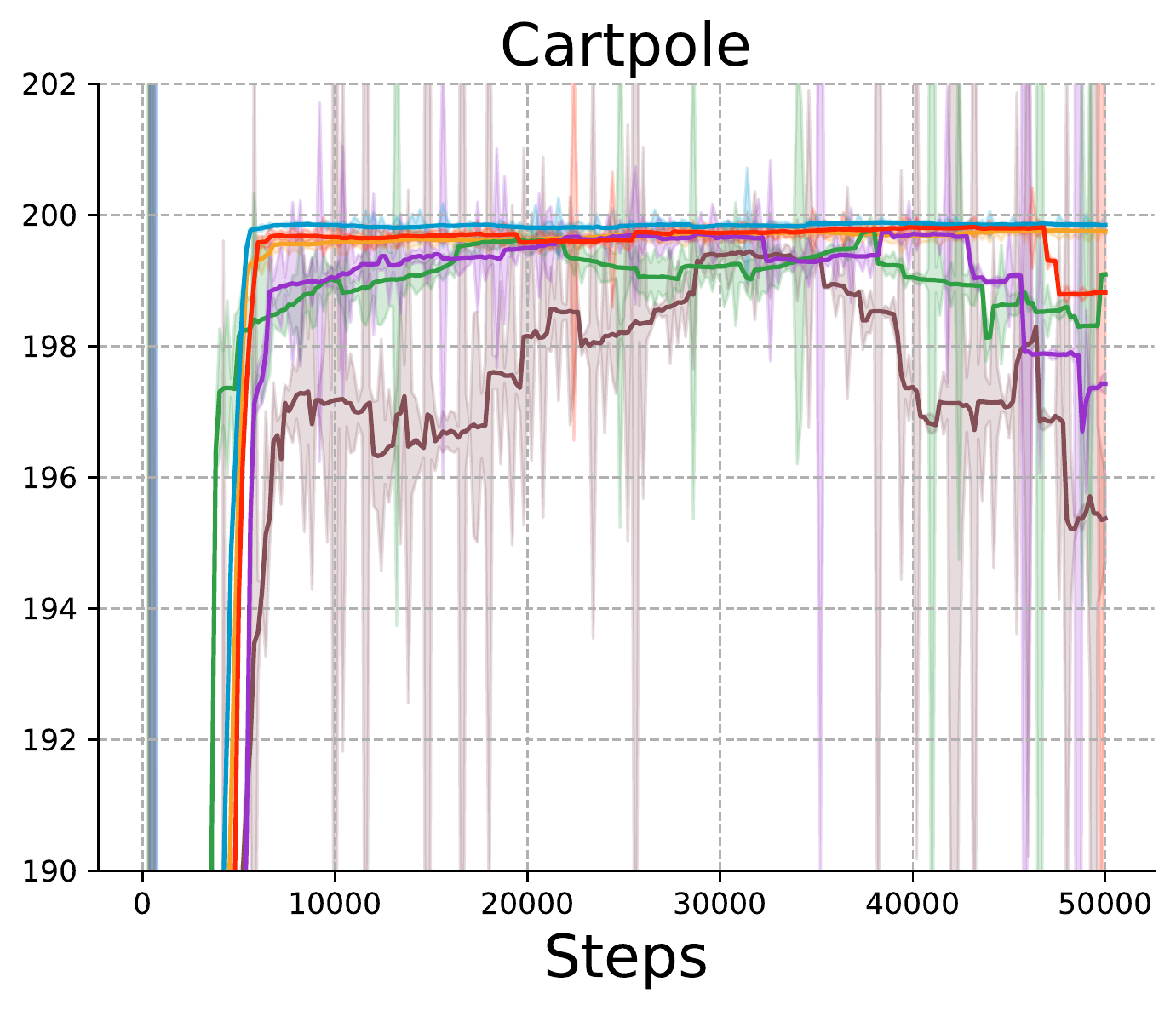}}
	\end{tabular}
	\begin{tabular}{@{}p{47mm}@{}p{47mm}@{}p{47mm}}
		{\includegraphics[height=3.7cm]{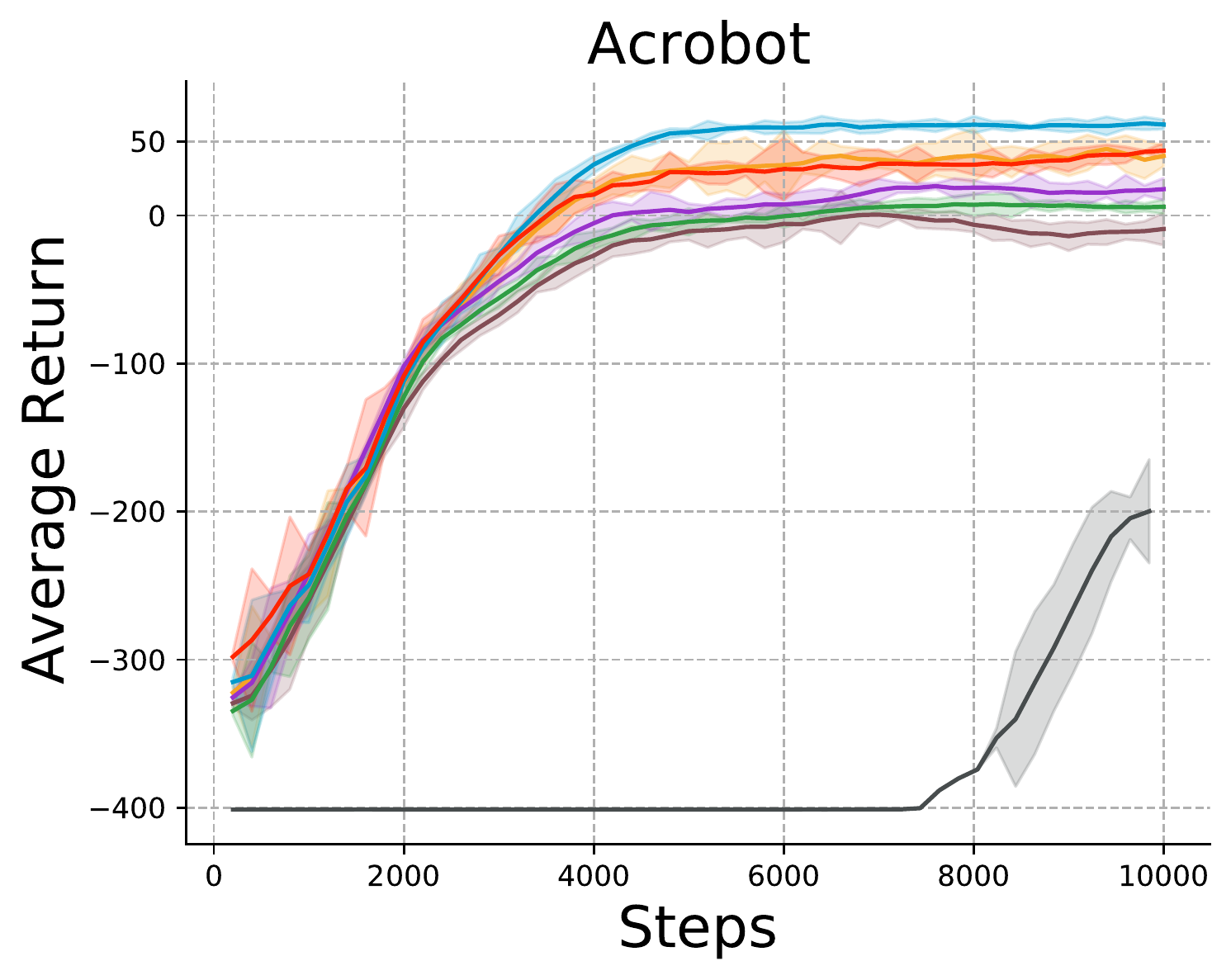}}
		&{\includegraphics[height=3.7cm]{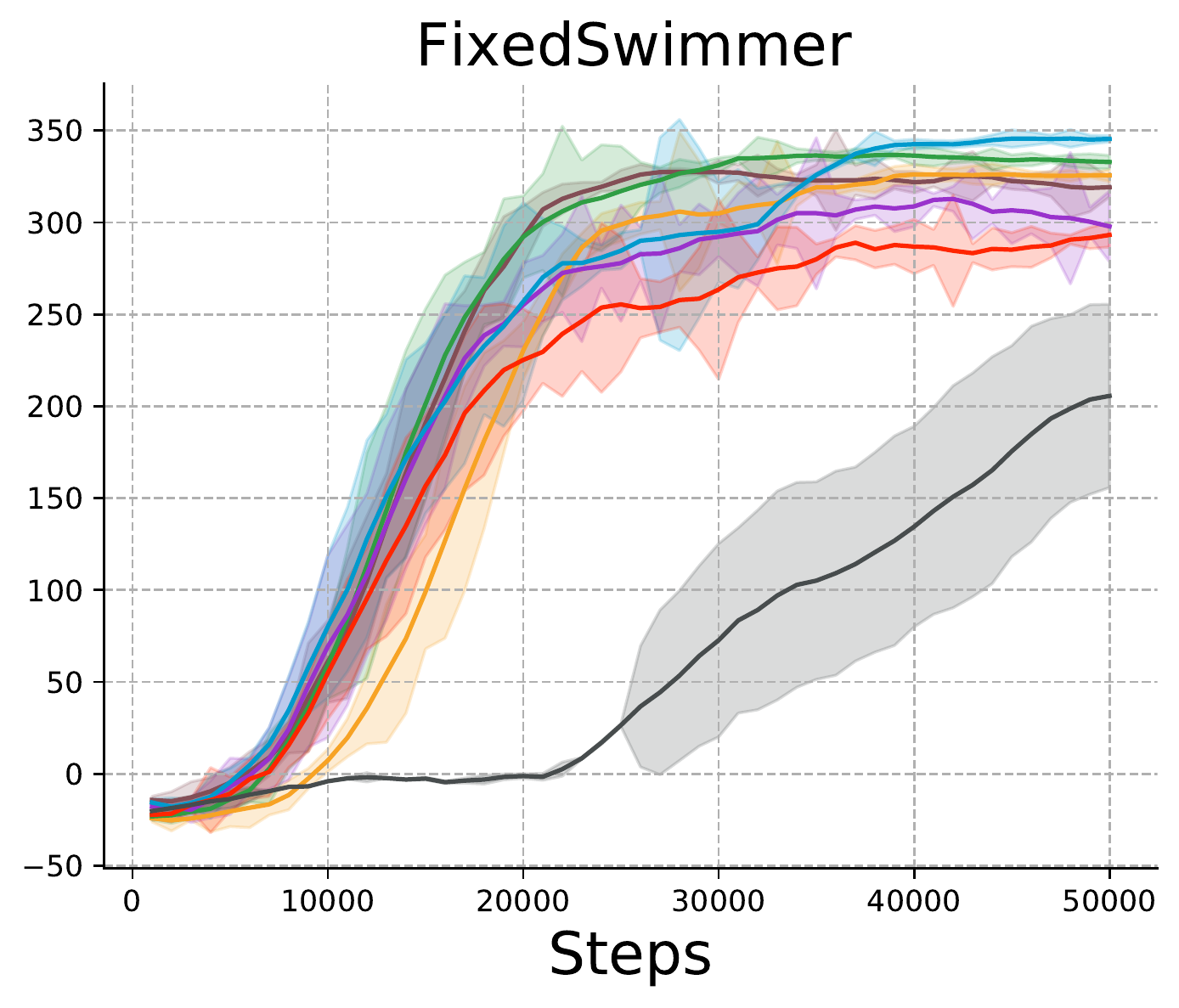}}
		&{\includegraphics[height=3.7cm]{figs/gym_reacher_20_smooth_noylabel.pdf}}
	\end{tabular}
	\begin{tabular}{@{}p{47mm}@{}p{47mm}@{}p{47mm}}
		{\includegraphics[height=3.7cm]{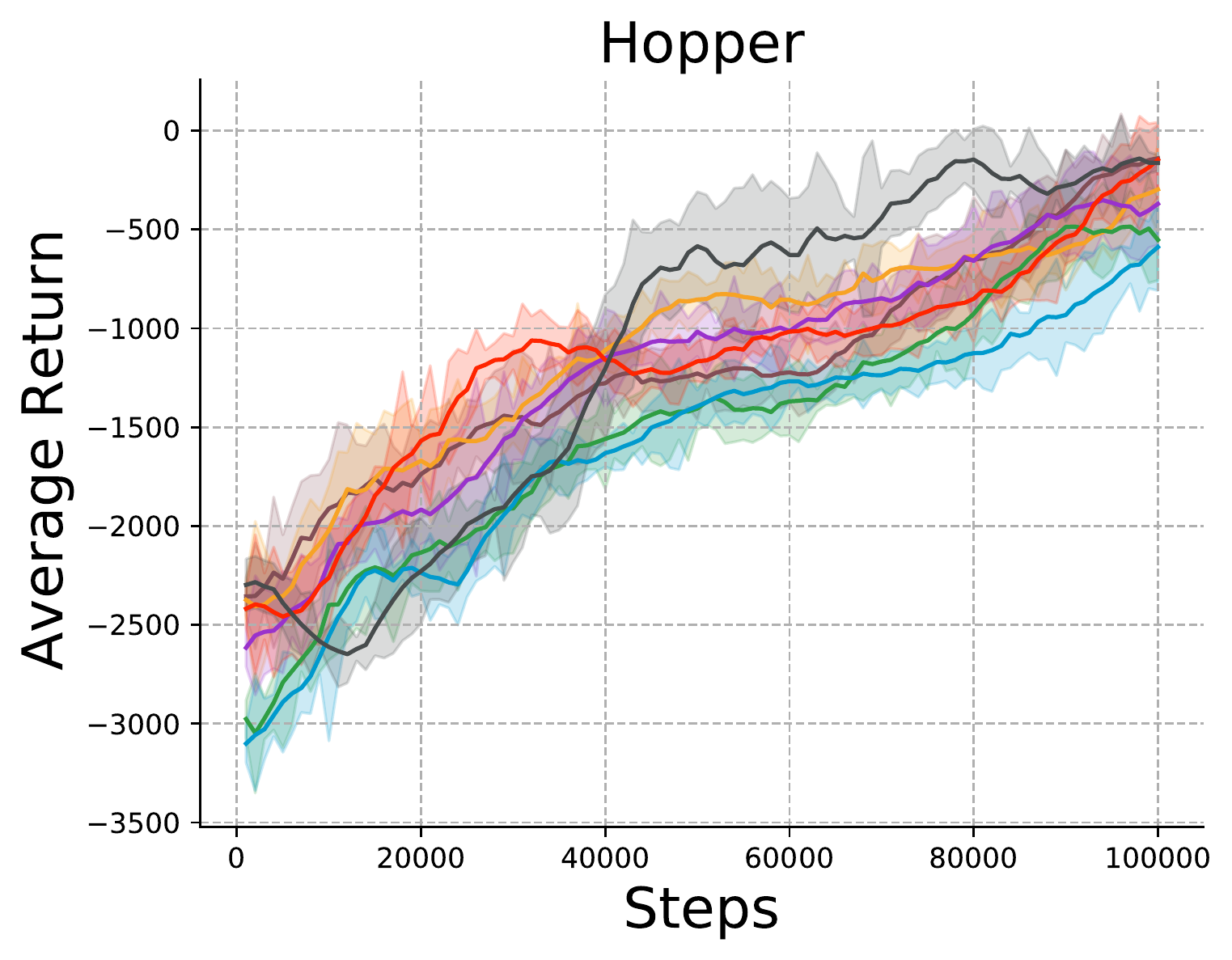}}
		&{\includegraphics[height=3.7cm]{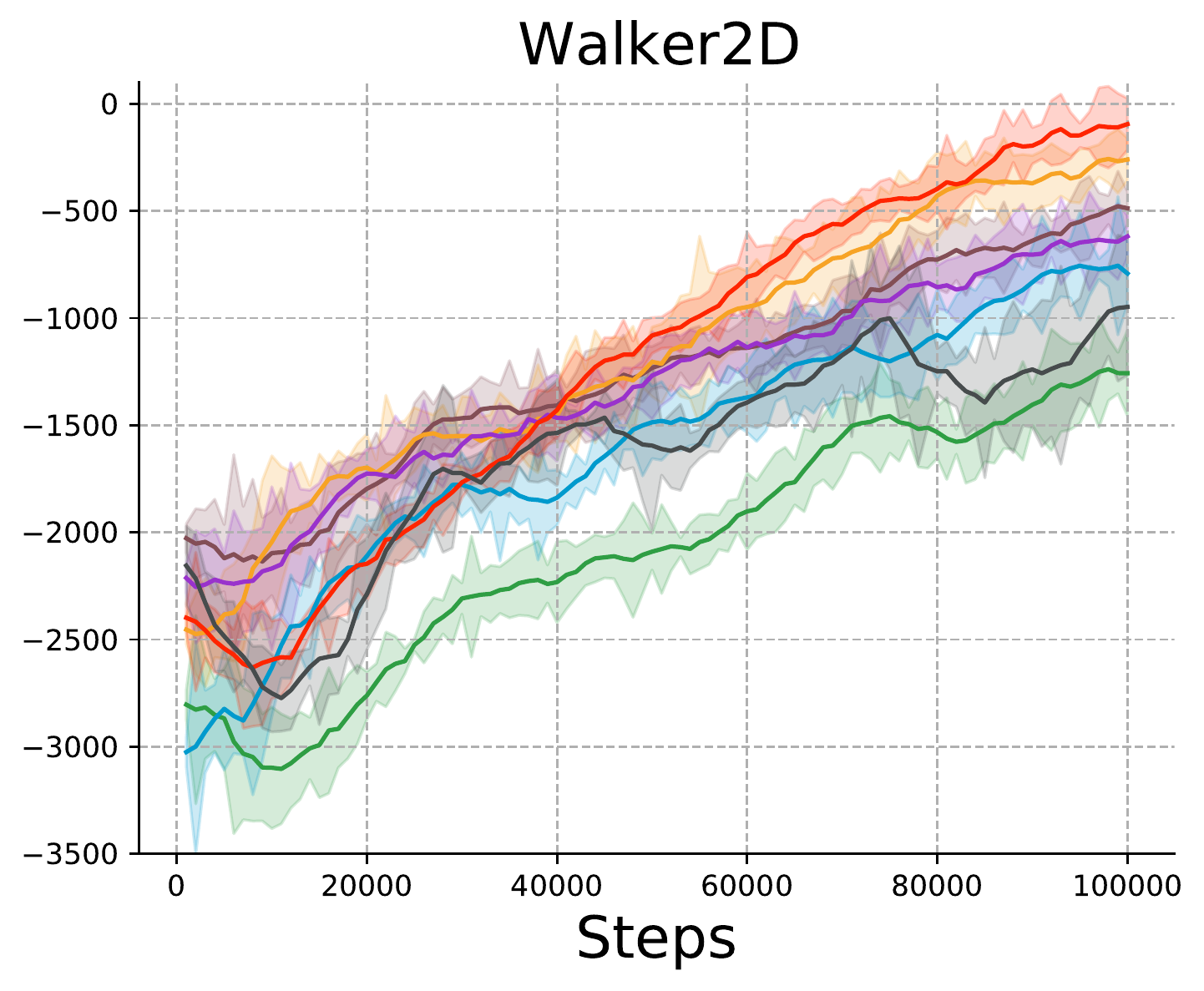}}
		&{\includegraphics[height=3.7cm]{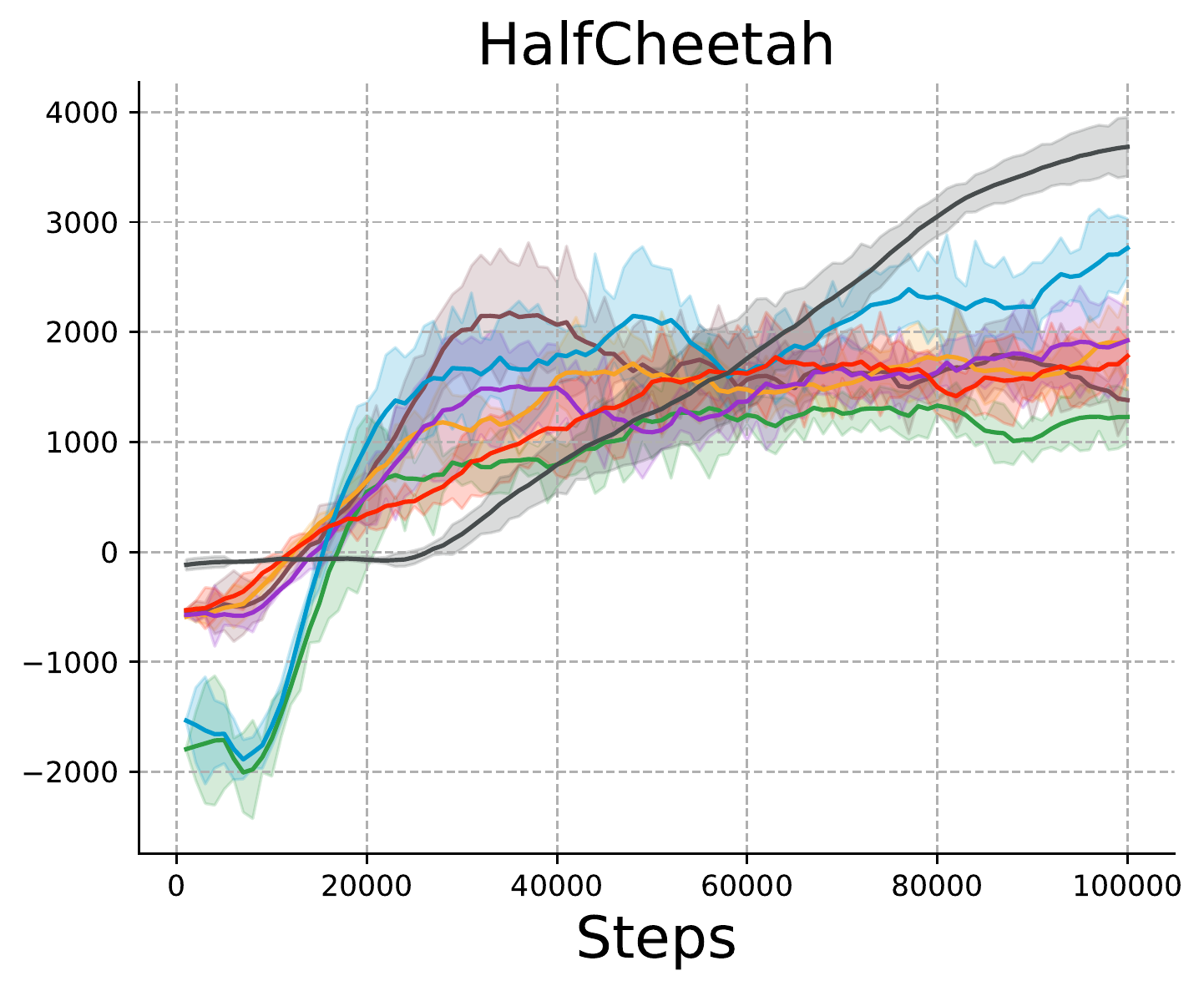}}
	\end{tabular}
	\begin{tabular}{@{}p{47mm}@{}p{47mm}@{}p{47mm}}
		{\includegraphics[height=3.7cm]{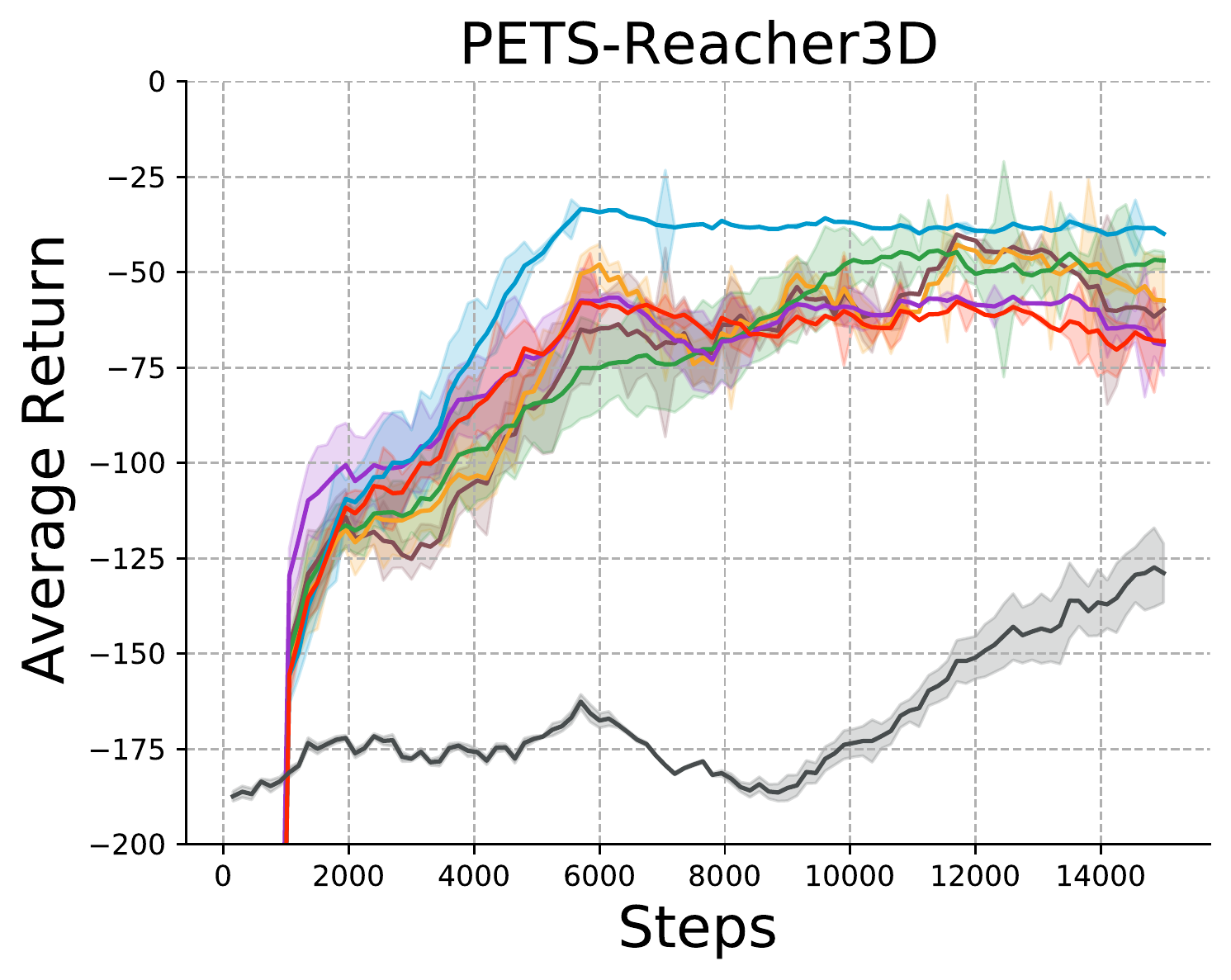}}
		&{\includegraphics[height=3.7cm]{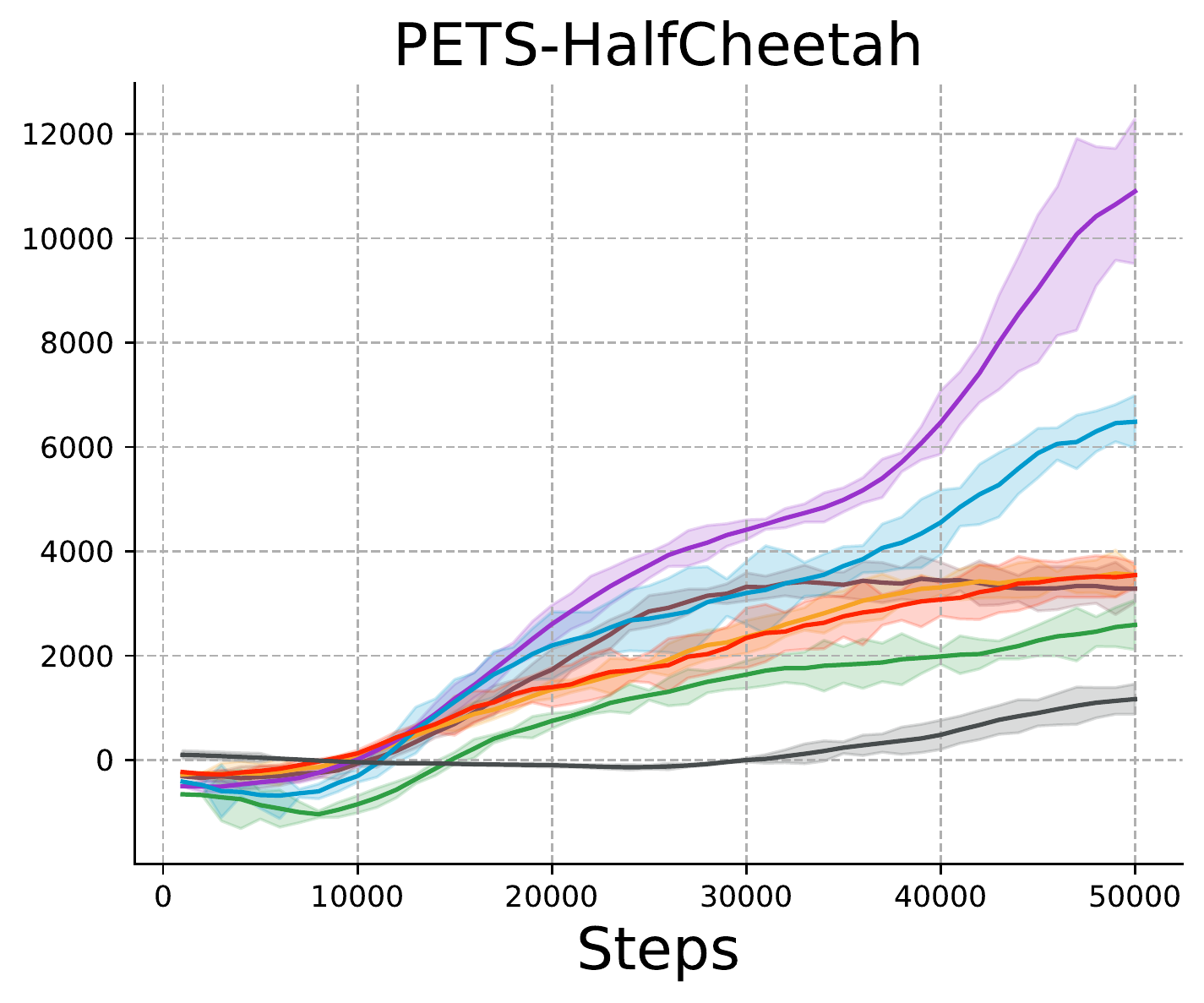}}
		&{\includegraphics[height=3.7cm]{figs/pusher_20_smooth_noylabel.pdf}}
	\end{tabular}
	\begin{tabular}{@{}p{47mm}@{}p{47mm}@{}p{47mm}}
		&
        {\includegraphics[height=3.7cm]{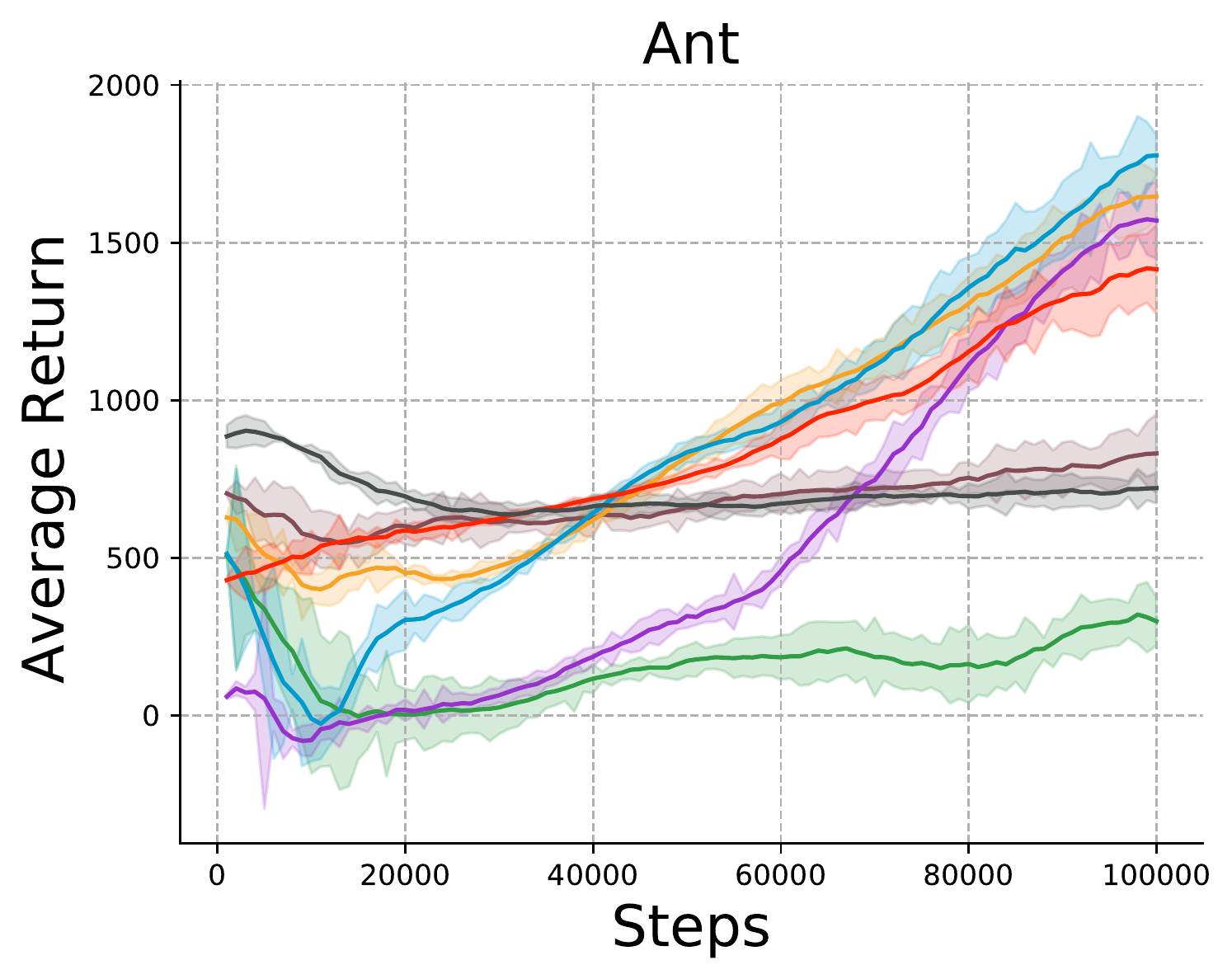}}
		&
	\end{tabular}
	\begin{tabular}{@{}c@{}}
		\includegraphics[width=0.9\textwidth]{figs/main_legend.pdf}
	\end{tabular}
	
	\caption{The learning curves of the proposed \textit{DecentCEM} methods and the baseline methods on continuous control environments. The line and shaded region shows the mean and standard error of evaluation results from 5 training runs using different random seeds. Each run is evaluated in an environment independent from training and reports average return of 5 episodes at every training episode. To ensure that the evaluation environments are the same across different methods and multiple runs, we set a fixed random seed in the evaluation environment of each task.}
    \label{fig:full-lc}
\end{figure*}

\subsection{Full Learning Curves}
In Figure \ref{fig:full-lc}, we report the learning curves in all environments listed in Appendix \ref{ap:env}.

The algorithms evaluated in the benchmark are: PETS, POPLIN-A, POPLIN-P and the proposed methods \emph{DecentPETS, DecentCEM-A} and \emph{DecentCEM-P}. 
We also included a model-free algorithm SAC as a baseline.
\emph{DecentCEM} subsumes POPLIN and they are equivalent when the ensemble size is one. The same applies to \emph{DecentPETS} and PETS. To distinguish them in the learning curves and discussions, we show the \emph{DecentCEM} results from an ensemble size larger than one.

The learning curves in some environments can be noisy. We apply smoothing with 1D uniform filter.
The window size of the filter was 10 for all but Cartpole, where 30 was used due to its large noise for all algorithms. 

Note that the performance of the baseline methods may be different from the results reported in their original paper. 
Specifically, in the paper by \cite{poplin}, PETS, POPLIN-A and  POPLIN-P have been evaluated in a number of environments that we use for the benchmark.
Our benchmark results may not be consistent with theirs due to differences in the implementation and evaluation protocol.
For example, our results of PETS, POPLIN-A and POPLIN-P in the Acrobot environment are all better than the results in \cite{poplin}.
We have identified a bug in the POPLIN code base that causes the evaluation results to be on a wrong timescale that is much slower than what it actually is. Hence the results of our implementation look far better, reaching a return of 0 at about 4k steps as opposed to 20k steps reported in \cite{poplin}. %

\subsection{Analysis}
Let's group the environments into two categories based on how well the \textit{decentralized} methods perform in them:
\begin{enumerate}
    \item Pendulum, InvertedPendulum, Acrobot, Cartpole, FixedSwimmer, Reacher, Walker2D, PETS-Pusher,  PETS-Reacher3D, Ant
    \item Hopper, HalfCheetah, PETS-HalfCheetah
\end{enumerate}

The first category is where the best performing method is one of the proposed \textit{decentralized} algorithms: \emph{DecentPETS}, \textit{DecentCEM-A} or \textit{DecentCEM-P}.
In environments where the baseline \emph{PETS}, \textit{POPLIN-A} or \textit{POPLIN-P} could reach near-optimal performance such as pendulum and invertedPendulum, applying the ensemble method would yield similar performance as before. It is also evident from InvertedPendulum results that the decentralized version significantly improves over the centralized algorithm where the performance of the latter is poor.
In Cartpole, Acrobot, Reacher, Walker2D and PETS-Pusher, applying the decentralized approach increases the performance in both action space planning (``A'') and parameter space planning (``P'').
In Pendulum, InvertedPendulum, FixedSwimmer, PETS-Reacher3D and Ant, ensemble helps in the action space planning but either has no impact or negative impact on the parameter space planning.

The second category is where it is better not to use a \textit{decentralized} approach with multiple instances (note that the decentralized methods with one instance fall back to one of \emph{PETS}, \textit{POPLIN-A}, \textit{POPLIN-P}).
In Hopper and HalfCheetah, the issue might lie in the model rather than planning since all MBRL baselines performed worse than the model-free baseline SAC.
In HalfCheetah, \textit{DecentCEM-A} in fact performs the best out of all model-based methods but it falls behind SAC.
One possible issue is that the true dynamics is difficult to approximate with our model learning approach. Another possibility is that it may be necessary to learn the variance of the sampling distribution, which none of these model-based approaches do. To be clear, the variance is \emph{adapted} online by CEM but it is not \emph{learned}.
PETS-HalfCheetah is slightly different in that the ensemble does improve the performance significantly when used for action space planning. However,\textit{ POPLIN-P} performs significantly better than all other algorithms. This suggests that the parameter space planning has been able to successfully find a high return region using a single Gaussian distribution. In this case, distributing the population size would not be able to trade the estimation accuracy for a better global search.

One interesting phenomenon is that DecentPETS performs better than or comparably as PETS in \emph{all} environments in both categories except in Hopper (where they are quite close as well). This suggests that when not using a learned neural network to initialize the distribution in CEM, decentralizing the samples is an effective technique to achieve an improvement of the optimization performance.

\section{More Ablations}
\label{ap:ab}

\begin{figure}[bht]
	\centering
	\includegraphics[trim=0.2cm 0 0.3cm 0, clip,height=5cm]{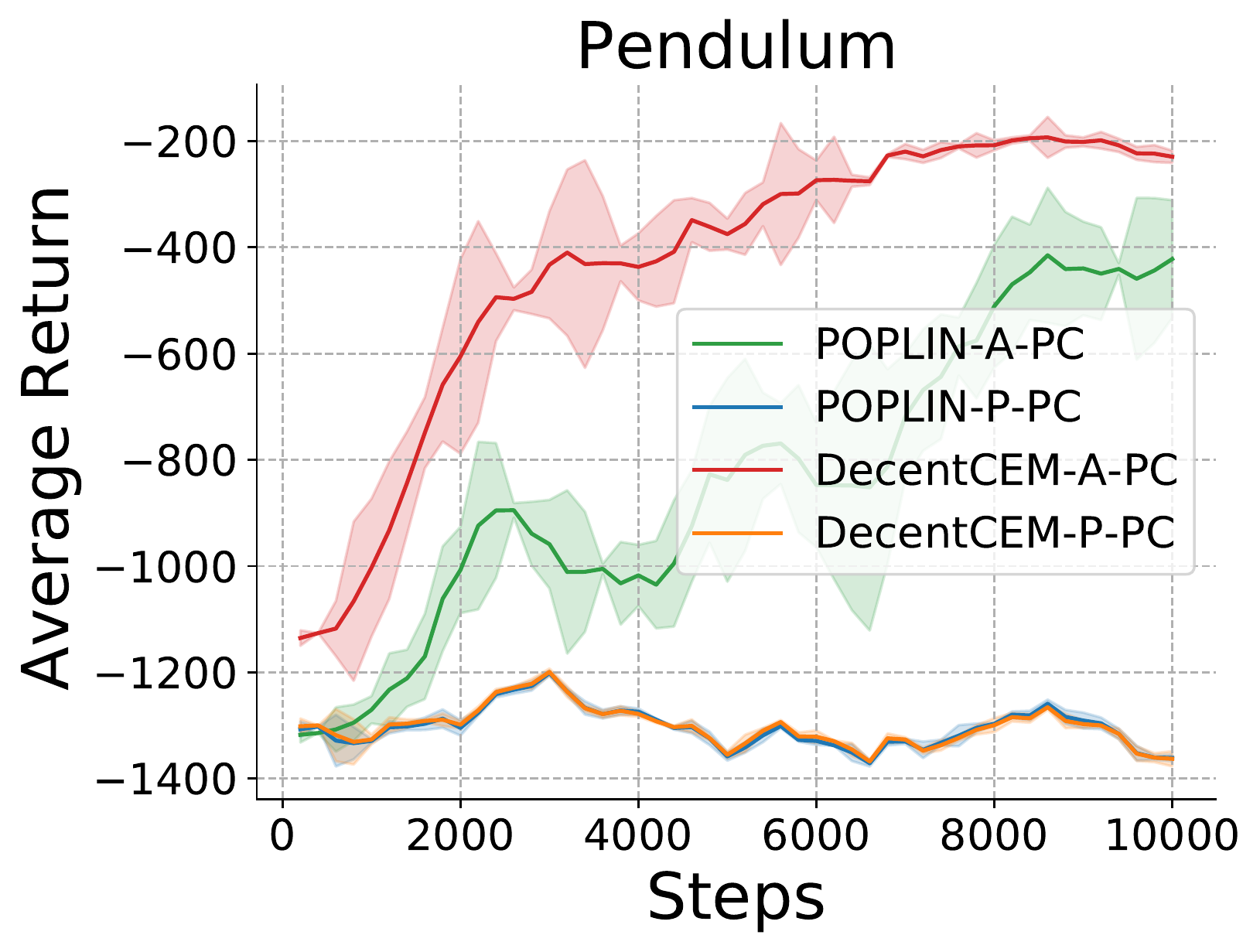} 
	\caption{Policy control ablation 
   where the policy network is directly used for control without CEM policy improvement}
	\label{fig:pendulum-pc}
\end{figure}

We study the performance of policy control where the policy network is directly used for control without the CEM step, denoted by the extra suffix ``-PC''.
The result is shown in Fig.~\ref{fig:pendulum-pc}.
Without the policy improvement from CEM, 
all algorithms perform worse than their counter-part of using CEM.
\textit{POPLIN-P-PC} and \textit{DecentCEM-P-PC} both get stuck in local optima and do not perform very well. This makes sense since the premise of planning in parameter space is that CEM can search more efficiently there. The policy network is not designed to be used directly as a policy.
Interestingly, \textit{DecentCEM-A-PC} achieves a high performance from about 7k steps (35 episodes) of training.
The ensemble of policy networks seems to add more robustness to control compared to using a single one.

\begin{figure}[h]
  \centering
\subfigure[Selection Ratios]{\includegraphics[height=3.6cm]{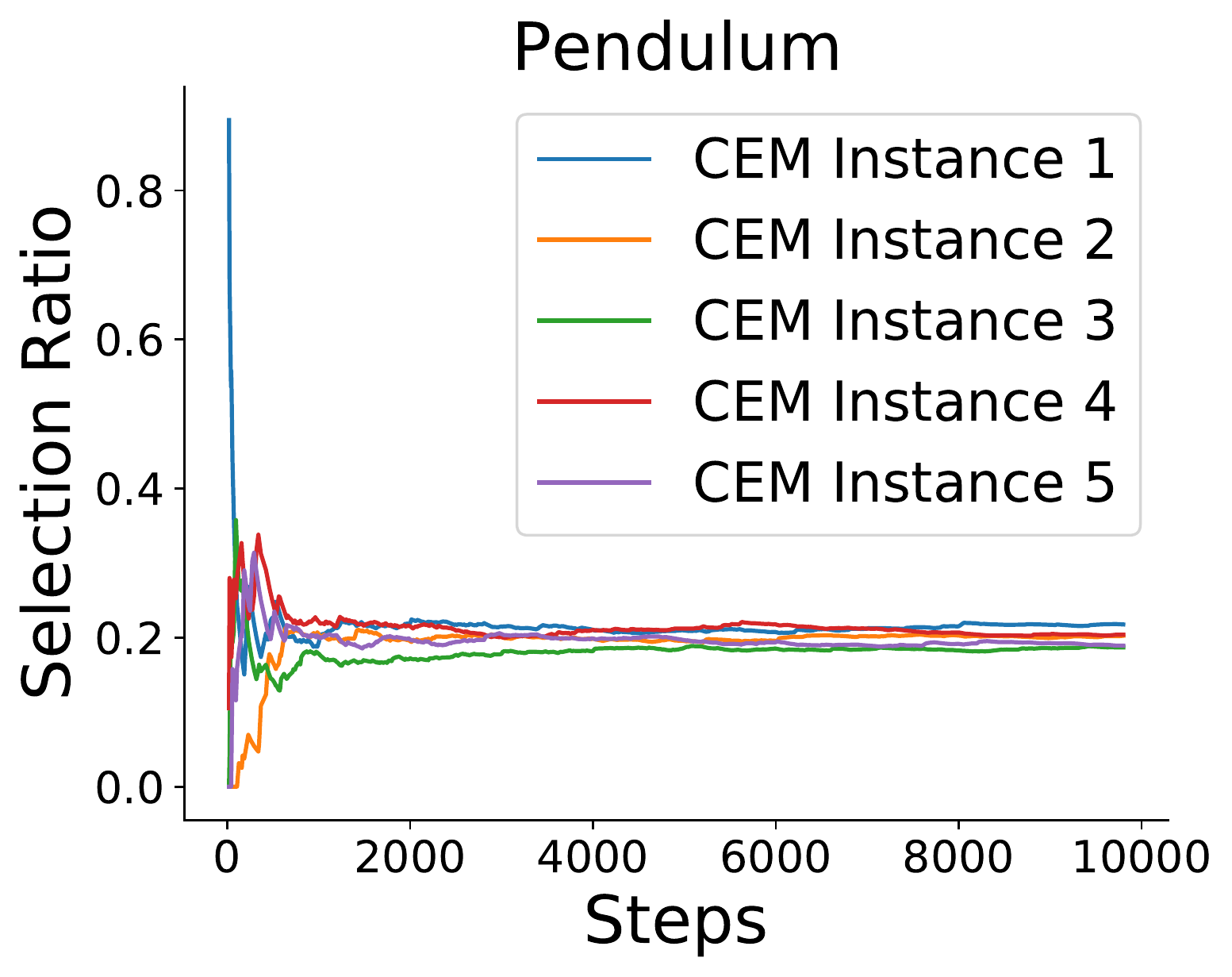}\label{fig:selection-ratio-pets}}    
\subfigure[Action Statistics]{\includegraphics[height=3.6cm]{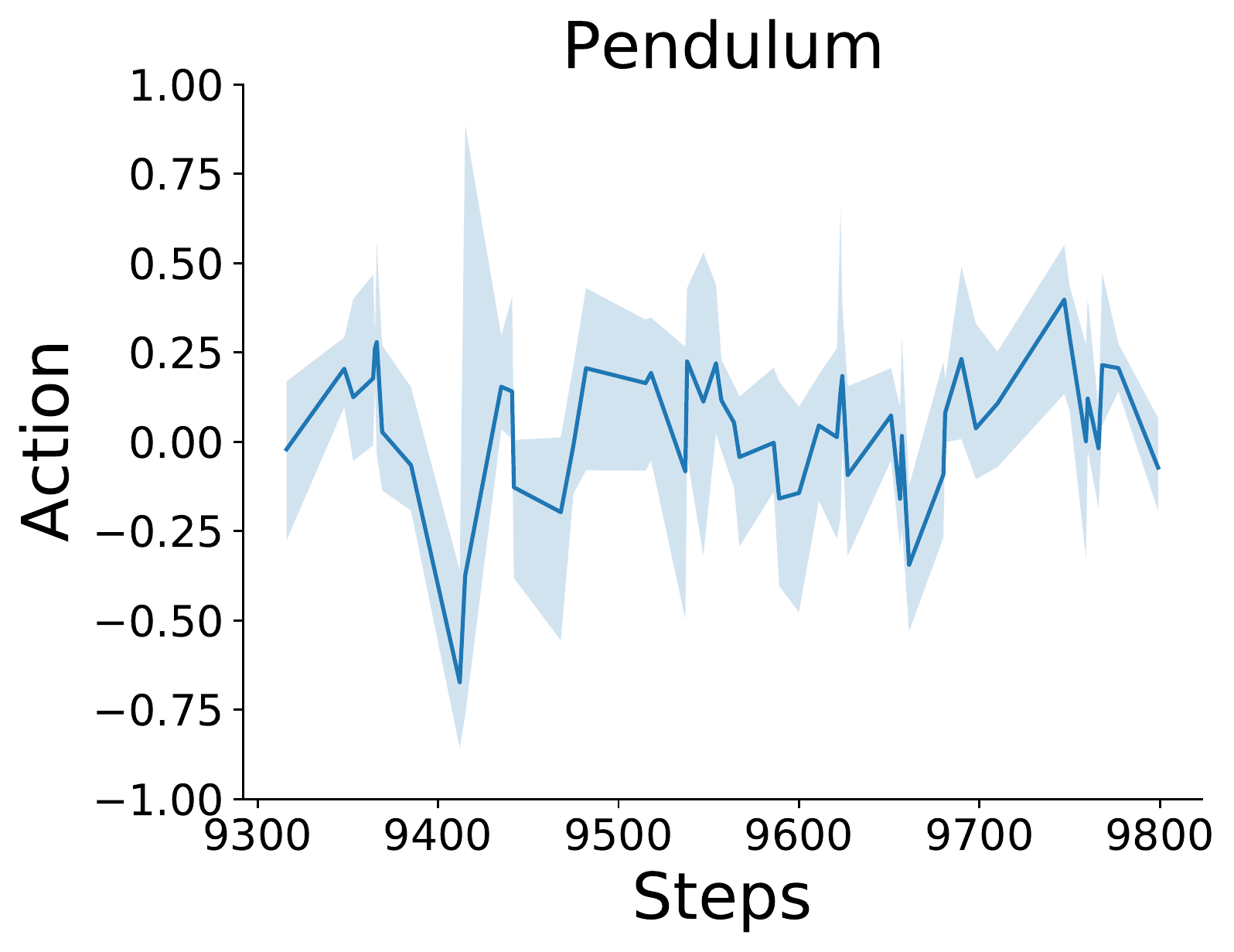}\label{fig:action-stats-pets}}
\subfigure[Action Distance Statistics]{\includegraphics[height=3.6cm]{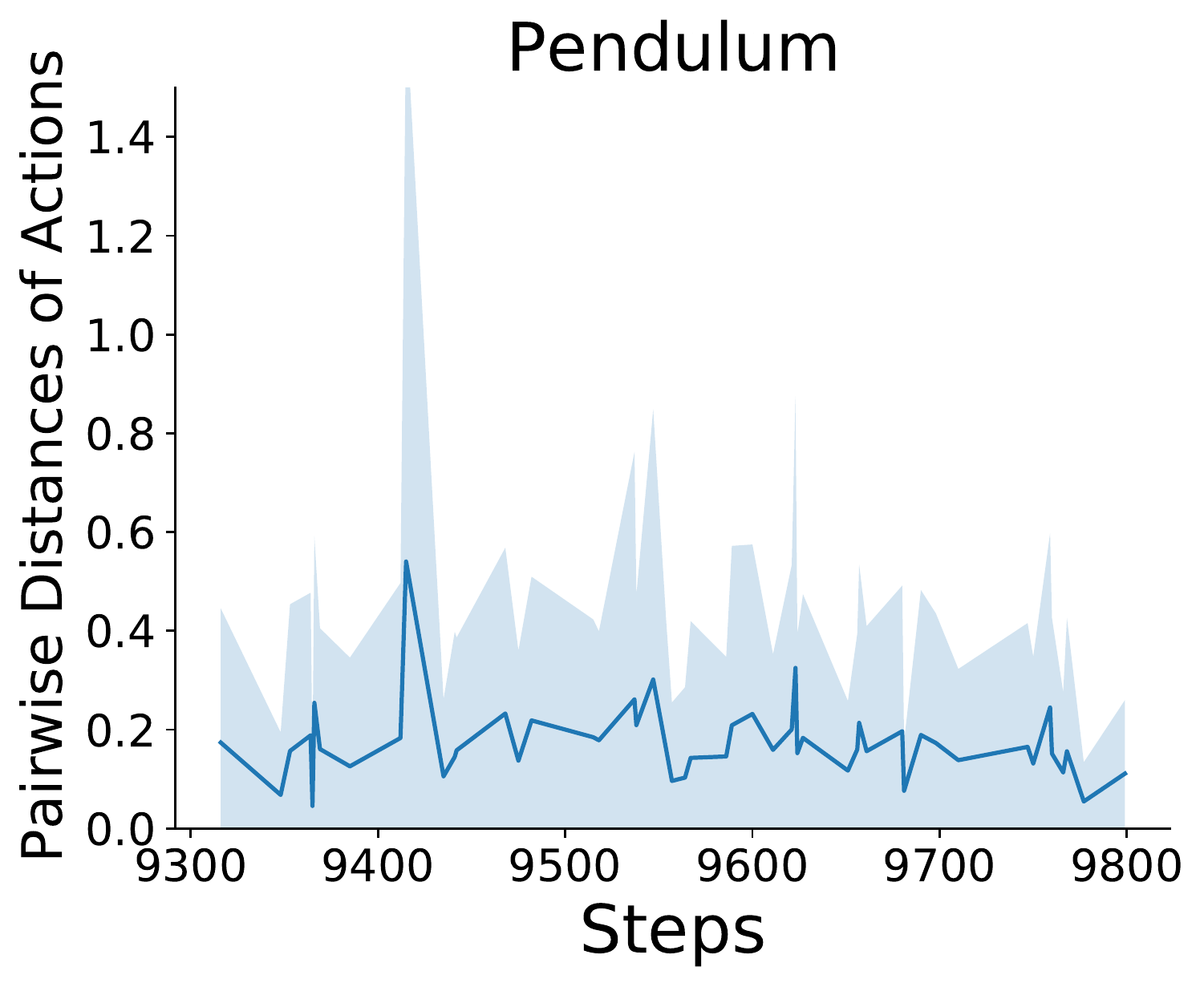} \label{fig:action-pair-pets}}
\caption{\footnotesize Ablation of ensemble diversity in Pendulum during training of \textit{DecentPETS} with 5 instances. (a) Cumulative selection ratio of each CEM instance. (b)(c) Statistics of the actions and pairwise action distances of the instances, respectively. The line and shaded region represent the mean and min/max.\normalsize}
 \label{fig:ensemble-diversity-pets}
\end{figure}

\begin{figure}[h]
  \centering
\subfigure[Selection Ratios]{\includegraphics[height=3.6cm]{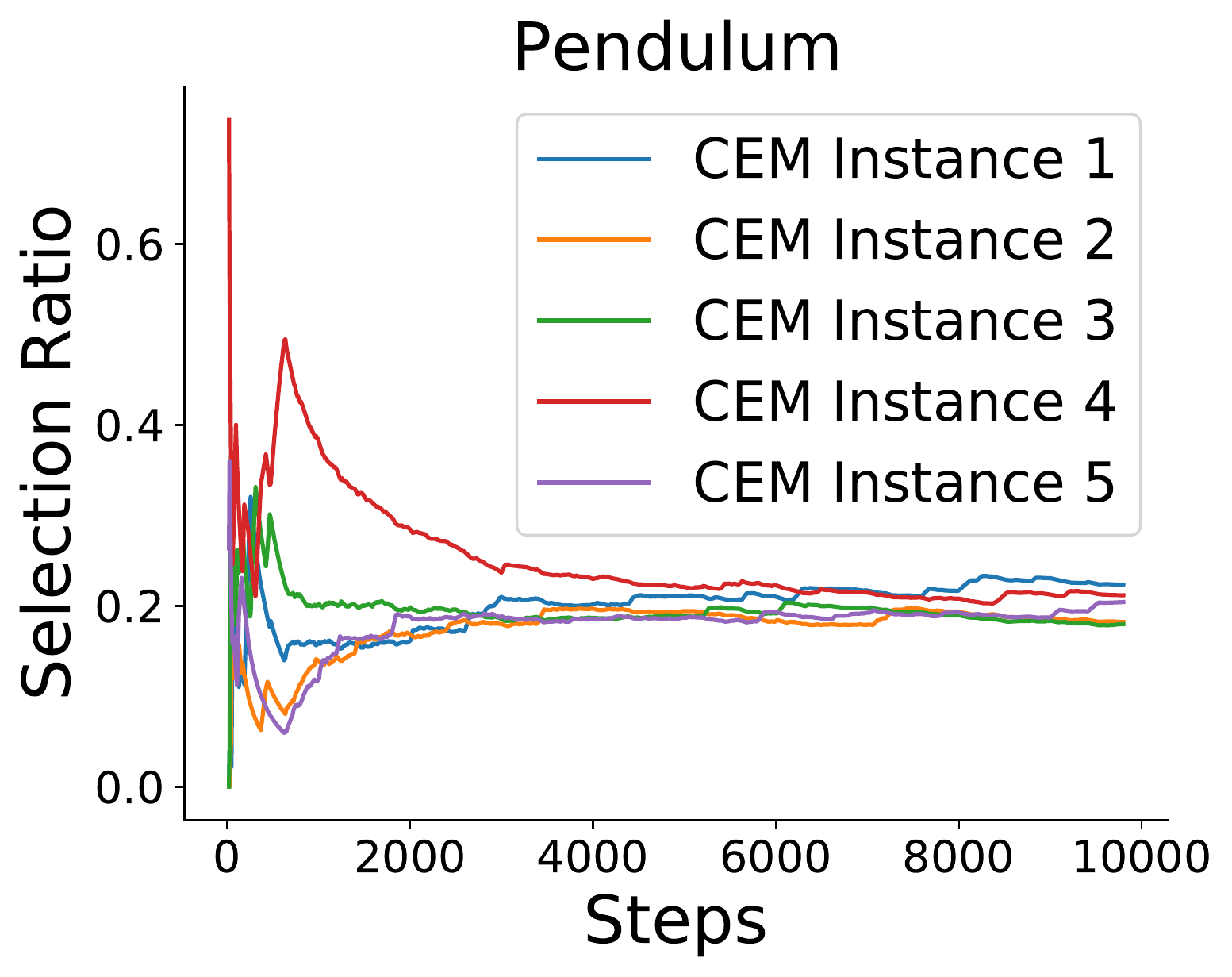}\label{fig:selection-ratio-p}}    
\subfigure[Action Statistics]{\includegraphics[height=3.6cm]{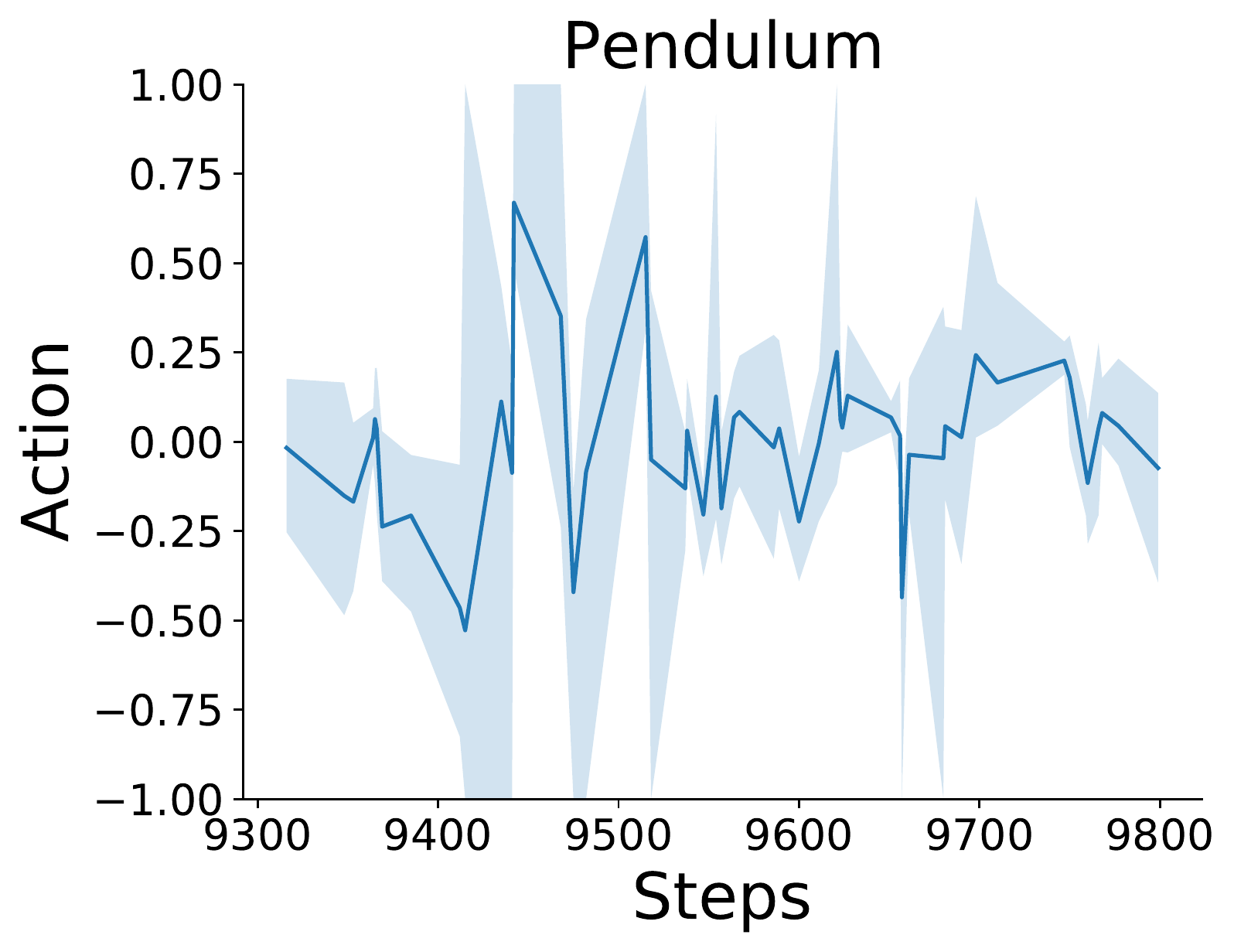}\label{fig:action-stats-p}}
\subfigure[Action Distance Statistics]{\includegraphics[height=3.6cm]{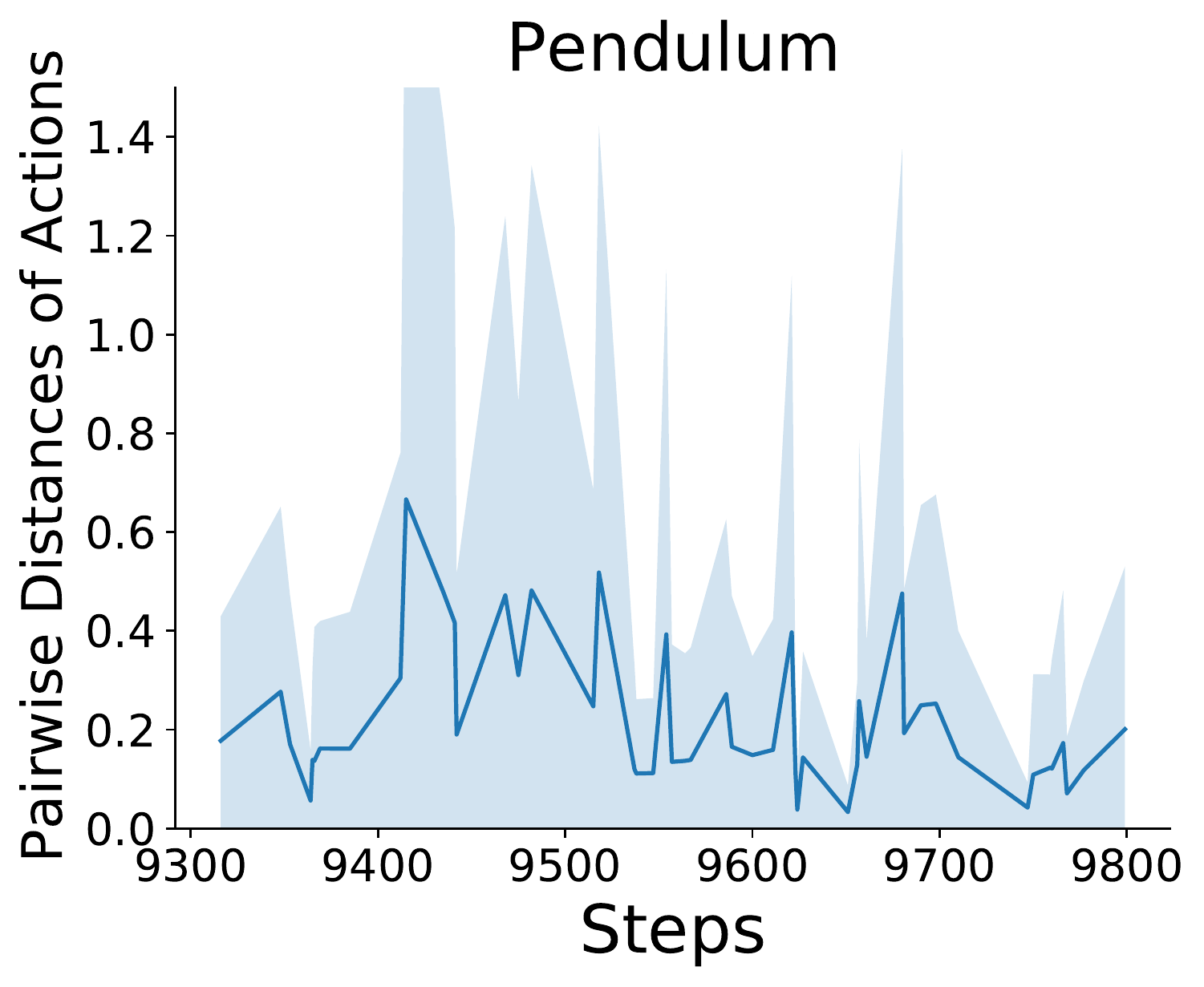} \label{fig:action-pair-p}}
\caption{\footnotesize Ablation of ensemble diversity in Pendulum during training of \textit{DecentCEM-P} with 5 instances. (a) Cumulative selection ratio of each CEM instance. (b)(c) Statistics of the actions and pairwise action distances of the instances, respectively. The line and shaded region represent the mean and min/max.\normalsize}
 \label{fig:ensemble-diversity-p}
\end{figure}

Figure~\ref{fig:ensemble-diversity-pets} and \ref{fig:ensemble-diversity-p} are additional plots for the ensemble diversity ablation. 
They show the results for DecentPETS and DecentCEM-P, respectively.
The same as in Fig.~\ref{fig:ensemble-diversity}~(b)(c), we only show a time window toward the end of training for visual clarity and the line and shaded region represent the mean and min/max.
Comparing the action and action distances statistics of the three algorithms shown in Fig.~\ref{fig:ensemble-diversity}~(b)(c),~\ref{fig:ensemble-diversity-pets}~(b)(c),~\ref{fig:ensemble-diversity-p}~(b)(c),
the actions from \emph{DecentPETS} cover a smaller range of values compared with those from \emph{DecentCEM-A/P}. This suggests that the use of policy networks in the multiple instances add more exploration without sacrificing performance thanks to the $\argmax$.

\section{Overhead of the Ensemble}
\label{ap:overhead}

The sample efficiency is not impaired when going from one policy network to the multiple policy networks used in \textit{DecentCEM-A }and \textit{DecentCEM-P}.
This is because that the generation of the training data only involve taking imaginary rollouts with the model, rather than interacting with the real environment, as discussed in Section \ref{sec:train-policy}.

In terms of the population size (number of samples drawn in CEM), the \textit{DecentCEM} methods do not impose additional cost.
We show in both the motivational example (Sec.~\ref{sec:optimization-1d}) and the benchmark experiments (Appendix \ref{ap:results}) that the proposed methods work better than \emph{CEM} under the same total population size.

The additional computational cost is reasonable in \textit{DecentCEM} compared to \textit{POPLIN}.
Each branch of policy network and CEM instance runs independently from the others, allowing for a parallel implementation.
The instances have to be synchronized ($\argmax$) but its additional cost is minimal. 
One caveat with our current implementation though is that the instances run serially, which slows down the speed. 
This is not a limitation of the method itself though and the speed loss can be alleviated by a parallel implementation.

\section{Convergence Analysis of Decentralized CEM}
\label{ap:convergence-analysis}
This section analyzes the convergence of the proposed DecentCEM algorithm in optimization.

Consider the following optimization problem:
\begin{equation}
x^* \in \argmax_{x\in \mathcal{X}} V(x)   
\label{eq:obj}
\end{equation}
where $\mathcal{X} \subset{\mathbb{R}^n}$ is a non-empty compact set and $V(\cdot)$ is a bounded, deterministic value function to be maximized. We assume that this problem has a unique global optimal solution $x^*$ but the objective function $V(\cdot)$ may have multiple local optimum and may not be continuous.

We will show that the existing convergence result of CEM in continuous optimization established in \citep{hu2011stochastic} also applies to DecentCEM.
It assumes that the sampling distribution $g_\phi(x)$ in CEM is in the natural exponential families (NEFs) which subsumes Gaussian distribution (with known covariance). We restate the definition of NEFs for completeness (definition 2.1 in \citep{hu2011stochastic}):
\begin{definition}[Natural Exponential Family]
A family of parameterized distributions $\{g_\phi(\cdot), \phi \in \Phi \subset \mathbb{R}^d\}$ on $\mathcal{X} \subset{\mathbb{R}^n}$ is called a Natural Exponential Family (NEF) if there exists continuous mappings $\Gamma: \mathbb{R}^n \rightarrow \mathbb{R}^d$, $h: \mathbb{R}^n \rightarrow \mathbb{R}$ and $K: \mathbb{R}^d \rightarrow \mathbb{R}$ such that $g_{\phi}(x)  = \text{exp}(\phi^\top \Gamma(x) - K(\phi)) h(x)$, where the parameter space $\Phi = \{ \phi \in \mathbb{R}^d : | K(\phi)| < \infty \}$, $K(\phi) = \ln \int_\mathcal{X}\text{exp}(\phi^\top \Gamma(x)) h(x) \nu(dx)$ and $\nu$ is the Lebesgue measure of $\mathcal{X}$. 
\label{def:NEF}
\end{definition}
The mean vector function
\begin{equation}
    m(\phi) = \mathbb{E}_\phi[\Gamma(x)]
    \label{eq:mean-vector}
\end{equation}
where the expectation $\mathbb{E}_\phi$ is with respect to $g_\phi$ and $\Gamma$ is the mapping in Def.~\ref{def:NEF}. It can be shown that $m(\cdot)$ is invertible.
Note that the expression of the densities can be simplified when restricted to a multivariate Gaussian distribution (with known diagonal covariance) where the natural sufficient statistics $\Gamma(x) = x$. 

We then present the CEM algorithm below to fix notations.
It follows Algorithm 2 in \citep{hu2011stochastic} but is modified to align with some notations introduced in previous sections in our paper.

\begin{algorithm}[h]
\DontPrintSemicolon
\SetNoFillComment
Choose the family of distributions $g_\phi(x), x\in \mathcal{X}$ from NEFs defined in \ref{def:NEF}
and the initial parameter $\phi_0 \in int(\Phi)$ where 
$int$ denotes the interior of the parameter space $\Phi$. \\
Specify elite ratio $\rho \in (0,1)$ and 
step size sequence $\{\alpha_k\}$ and $\{\lambda_k\}$ where $k$ denotes the time step.
Set $k=0$.
Specify $\epsilon > 0$ which is the parameter in the  thresholding function defined in \Eqref{eq:threholding}.

\begin{equation}
    \mathbbm{1}(x,\gamma) = 
    \begin{cases}
    1,& \text{if } x\geq \gamma\\
    \frac{x-\gamma+\epsilon} {\epsilon},& \text{if } \gamma - \epsilon < x < 1\\
    0 ,& \text{if } x \leq \gamma - \epsilon
    \end{cases}
    \label{eq:threholding}
\end{equation}

  \Repeat{a stopping condition is reached} 
  {
  Step 1: Draw $N_k$ i.i.d samples $\Lambda_k = \{x_1, x_2, ..., x_{N_k}\}$ from the distribution $g_{\phi_k}(x)$ \\
  Step 2: Calculate the sample $(1-\rho)$-quantile  $\hat{\gamma}_k = V_{(\lceil(1-\rho) N_k\rceil)}$ where $\lceil a \rceil$ is the ceiling function that gives the smallest integer greater than $a$ and $V_{({i})}$ is the $i$th-order statistics of the sequence $\{V(x_j)\}_{j=1}^{N_k}$ where $V(\cdot)$ is the objective function to be maximized. \\
  Step 3: Compute the new parameter $\phi_{k+1} = m^{-1}(\eta_{k+1})$, where 
  $\eta_0 = m(\phi_0) = \mathbb{E}_{\phi_0}(\Gamma(x))$
  and
  \begin{align}
\eta_{k+1} &= \alpha_k \frac{\sum_{x\in \Lambda_k} \mathbbm{1}(V(x),\hat{\gamma}_k) \Gamma(x) }{\sum_{x\in \Lambda_k} \mathbbm{1}(V(x),\hat{\gamma}_k)} + \nonumber \\
&\ (1-\alpha_k) \left(\frac{\lambda_k} {N_k} \sum_{x\in\Lambda_k} \Gamma(x) + (1-\lambda_k) \eta_k\right)
  \label{eq:param-update}
  \end{align} \\
Step 4: $k = k+1 $ %
 }
 \Return $\phi_k$
 \caption{CEM}
\label{alg:CEM}
\end{algorithm}

The convergence results will require the following assumptions from \citet{hu2011stochastic}:

\begin{assumption}
The parameter $\phi_{k+1}$ computed at step 3 of Algorithm \ref{alg:CEM} satisfies $\phi_{k+1} \in int(\Phi)$ for all $k$.
\label{ap:a2}
\end{assumption}

\begin{assumption}
The step size sequence $\{\alpha_k\}$ satisfies: $\alpha_k > 0 \ \forall\,k$ , $\lim_{k\rightarrow \infty} \alpha_k =0$ and $\sum_{k=0}^\infty \alpha_k = \infty$.
\label{ap:a3}
\end{assumption}

\begin{assumption}
The sequence $\{\lambda_k\}$ satisfies
$\lambda_k = O(k^{-\lambda})$ for some constant 
$\lambda \geq 0$ and the sample size $N_k=\Theta(k^\beta)$ where $\beta > \max\{0, 1-2\lambda\}$.
\label{assumption:sample}
\end{assumption}

\begin{assumption}
The $(1-\rho)$-quantile of $\{V(x), x \sim g_\phi(x) \}$ is unique for each $\phi \in \Phi$.
\label{ap:a4}
\end{assumption} 

We know from \cite{hu2011stochastic} that the sequence $\{\eta_k\}_{k=0}^\infty$ from \eqref{eq:param-update} asymptotically approaches the solution set of the ODE:
\begin{equation}
    \frac{d \eta(t)}{dt} = L(\eta)
    \label{eq:ODE}
\end{equation}
\begin{equation}
  L(\eta) = \nabla_\phi \ln \mathbb{E}_\phi [\mathbbm{1} (V(x), \gamma(m^{-1}(\eta)))] \, \rvert_{\phi={m^{-1}(\eta)}}
  \label{eq:L}
\end{equation}
where $\gamma(m^{-1}(\eta))$ is the true  $(1-\rho)$-quantile of $V(x)$ under $g_{m^{-1}(\eta)}$.

\begin{assumption}
The function $L(\eta)$ defined in \eqref{eq:L} has a unique integral curve for a given initial condition.
\label{ap:a5}
\end{assumption}

The above assumptions 2-6 are the assumptions required by the previous convergence result of CEM. 
To show the convergence of DecentCEM, we only require one additional mild condition in the assumption \ref{ap:a1} (note that the sample size requirement is included here only for completion since it is already part of Assumption \ref{assumption:sample}).

Now we restate the convergence result of DecentCEM from the main text and show the proof:
\convergence*
\vspace{-8pt}
\begin{proof}
Each individual CEM instance has a sample size of $\frac{N_k}{M}$ and $N_k =\Theta(k^\beta)$.
Since Assumption 1 holds, $M$ is constant and gets absorbed into the $\Theta$ and we have $\frac{N_k}{M} =\Theta(k^\beta)$. 
Hence the conditions of Theorem 3.1 in \citep{hu2011stochastic} holds for each CEM instance indexed by $i$ and can be directly applied to show the almost sure convergence of their solutions $\{\eta_{i,k}\}$ to an internally chain recurrent set of \Eqref{eq:ODE}. 
If the recurrent sets are isolated equilibrium points, then $\{\eta_{i,k}\}$ converges almost surely to a unique equilibrium point.

Due to the fact that the instances in DecentCEM run independently from each other, their solutions $\{\eta_{i,k}\}_{i=1}^M$ (or equivalently $\{\phi_{i,k}\}_{i=1}^M = \{m^{-1}(\eta_{i,k})\}_{i=1}^M$) might converge to identical or different solutions denoted as $\{\eta_i^*\}_{i=1}^M$.
DecentCEM computes the final solution by applying an $\argmax$ over all individual solutions:
$\eta_{o,k} = \argmax_{\eta \in \{\eta_{i,k}\}_{i=1}^M} \mathbb{E}_{m^{-1}(\eta)} [V(x)]$ (equivalent to \Eqref{eq:decentcem-argmax}).
Here the expectation is approximated by the sample mean with respect to the distribution $g_{m^{-1}(\eta)}$: $\frac{1}{N_k} \sum_{j=1}^{N_k} V(x_j)$, 
which converges almost surely to the true expectation according to the strong law of large numbers.
Hence we have that $\eta_{o,k}$ converges almost surely to the best solution in the set $\{\eta_i^*\}_{i=1}^M$ found by the individual CEM instances, in terms of the expected value of $\mathbb{E}_{m^{-1}(\eta)} [V(x)]$.
\end{proof}

Note that the theorem implies that the solution of CEM / DecentCEM assigns the maximum probability to a locally optimal solution to \Eqref{eq:obj}. 
It does not guarantee whether this local optimum is a global optimum or not.
To the best of our knowledge, almost sure convergence to a local optimum is the only convergence result that has been established about CEM in continuous optimization.

\section{Visualization of DecentCEM Planning}
\label{ap:visplan}

\begin{figure}[bth]
  \centering
     \includegraphics[width=0.32\linewidth]{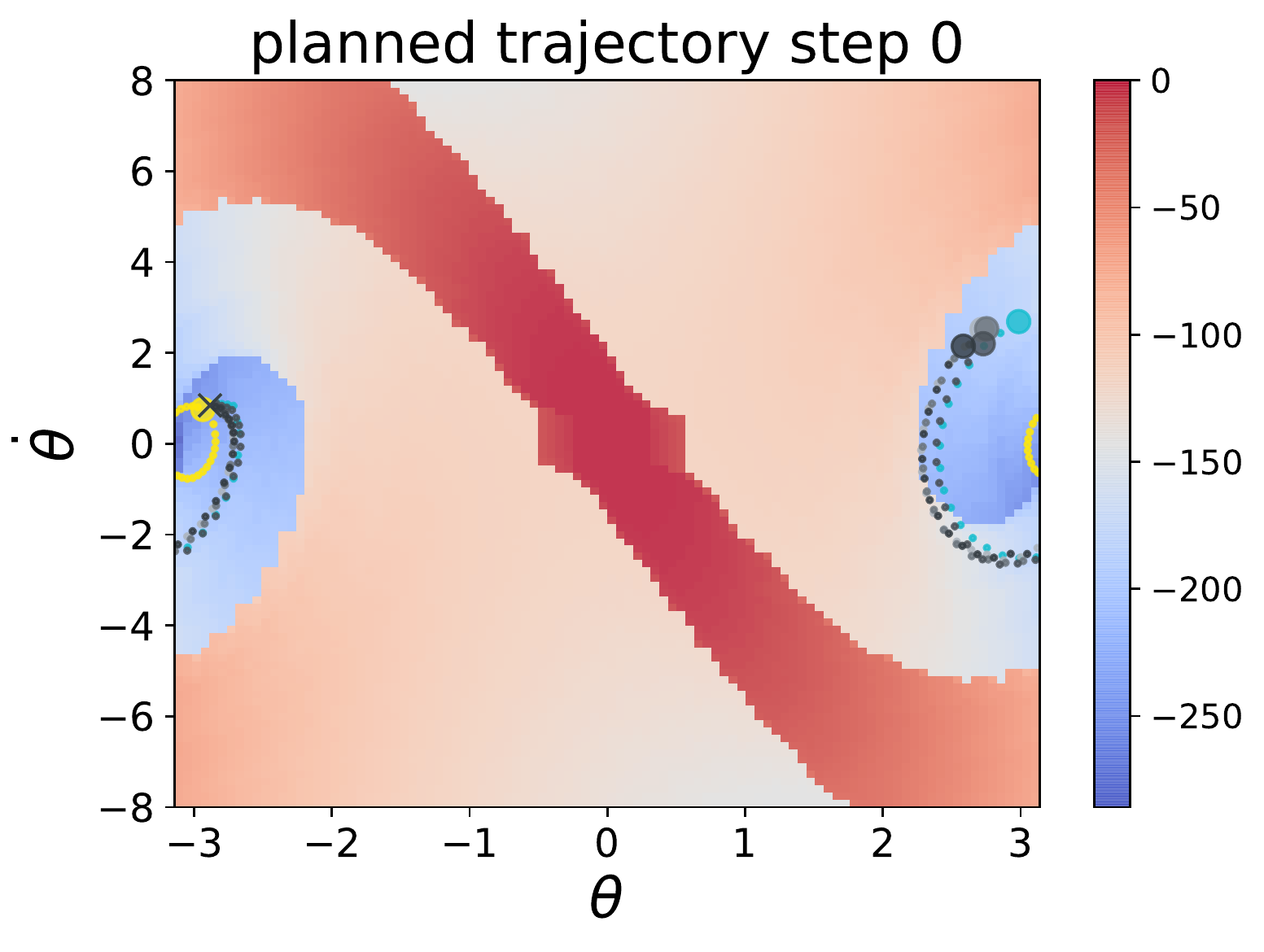}
    \includegraphics[width=0.32\linewidth]{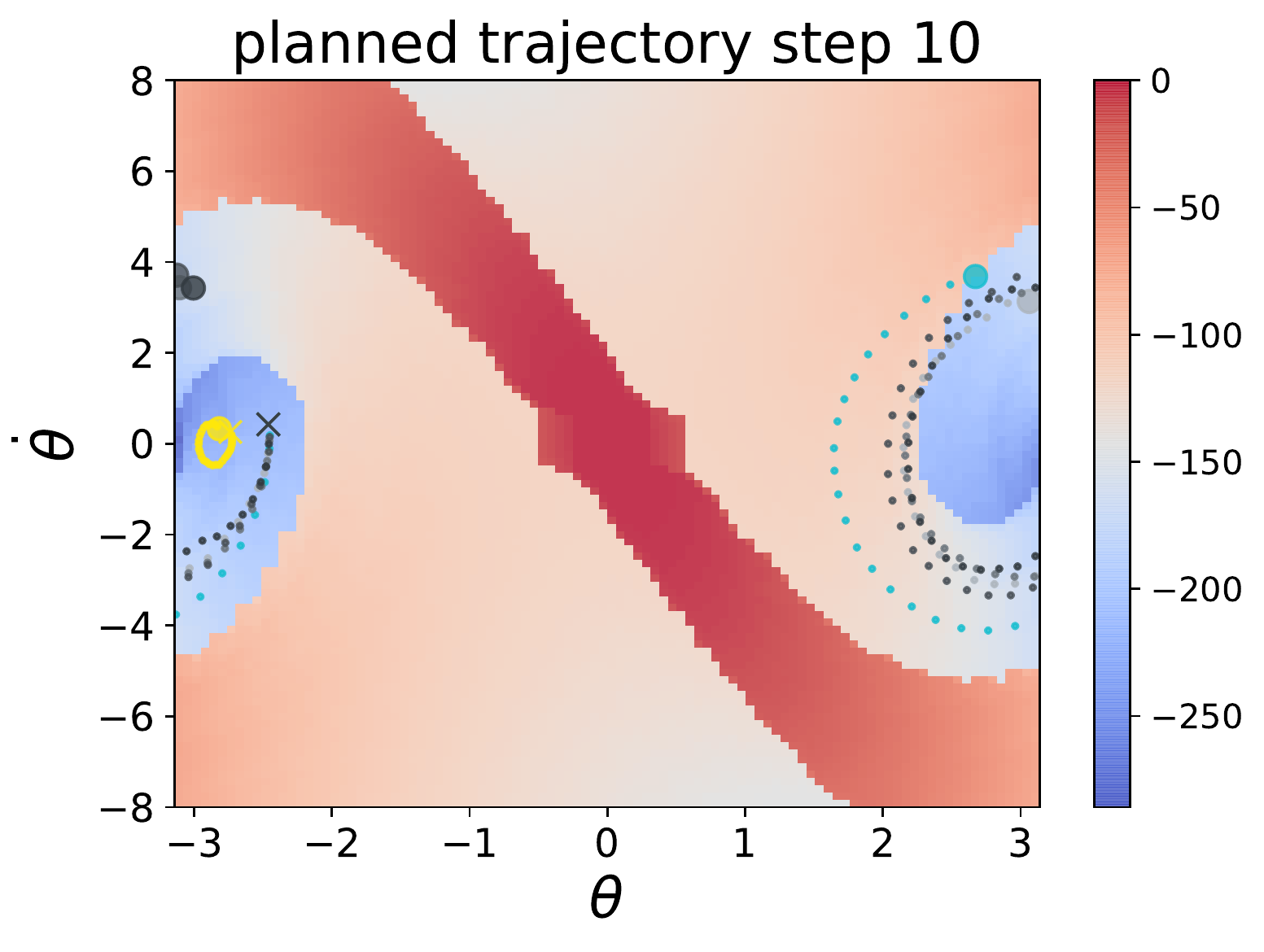}
    \includegraphics[width=0.32\linewidth]{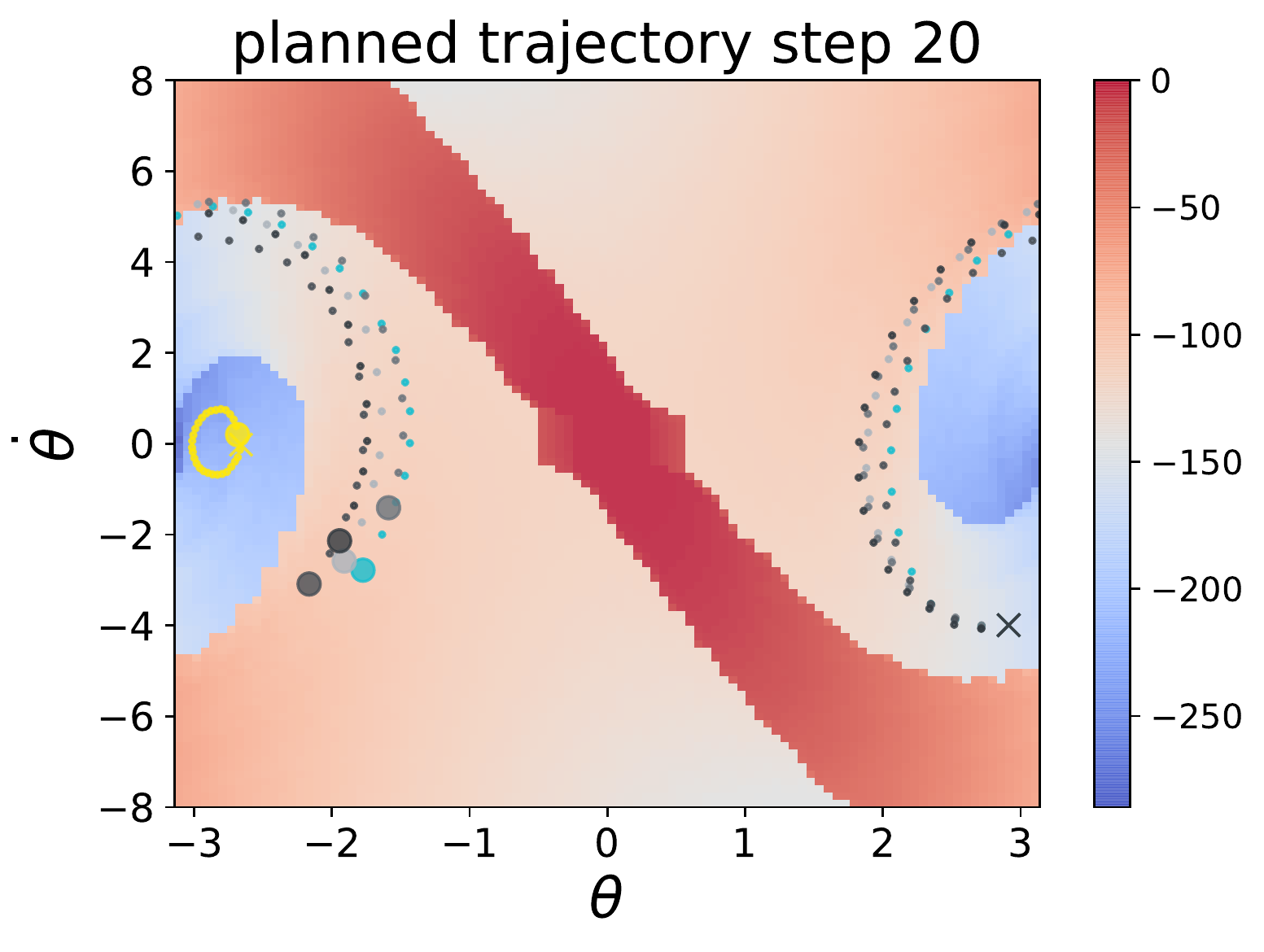} \\
         \includegraphics[width=0.33\linewidth]{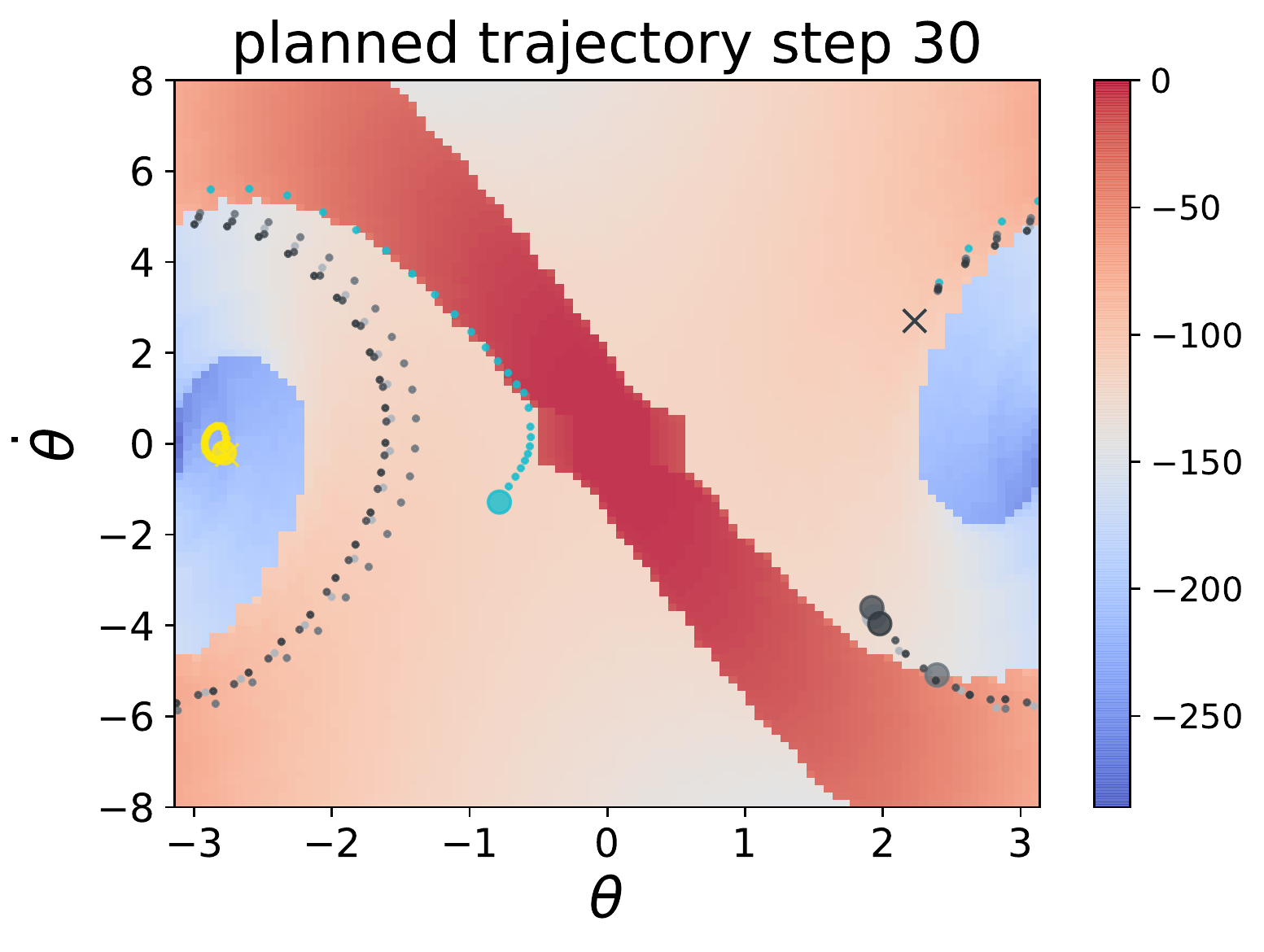}
    \includegraphics[width=0.33\linewidth]{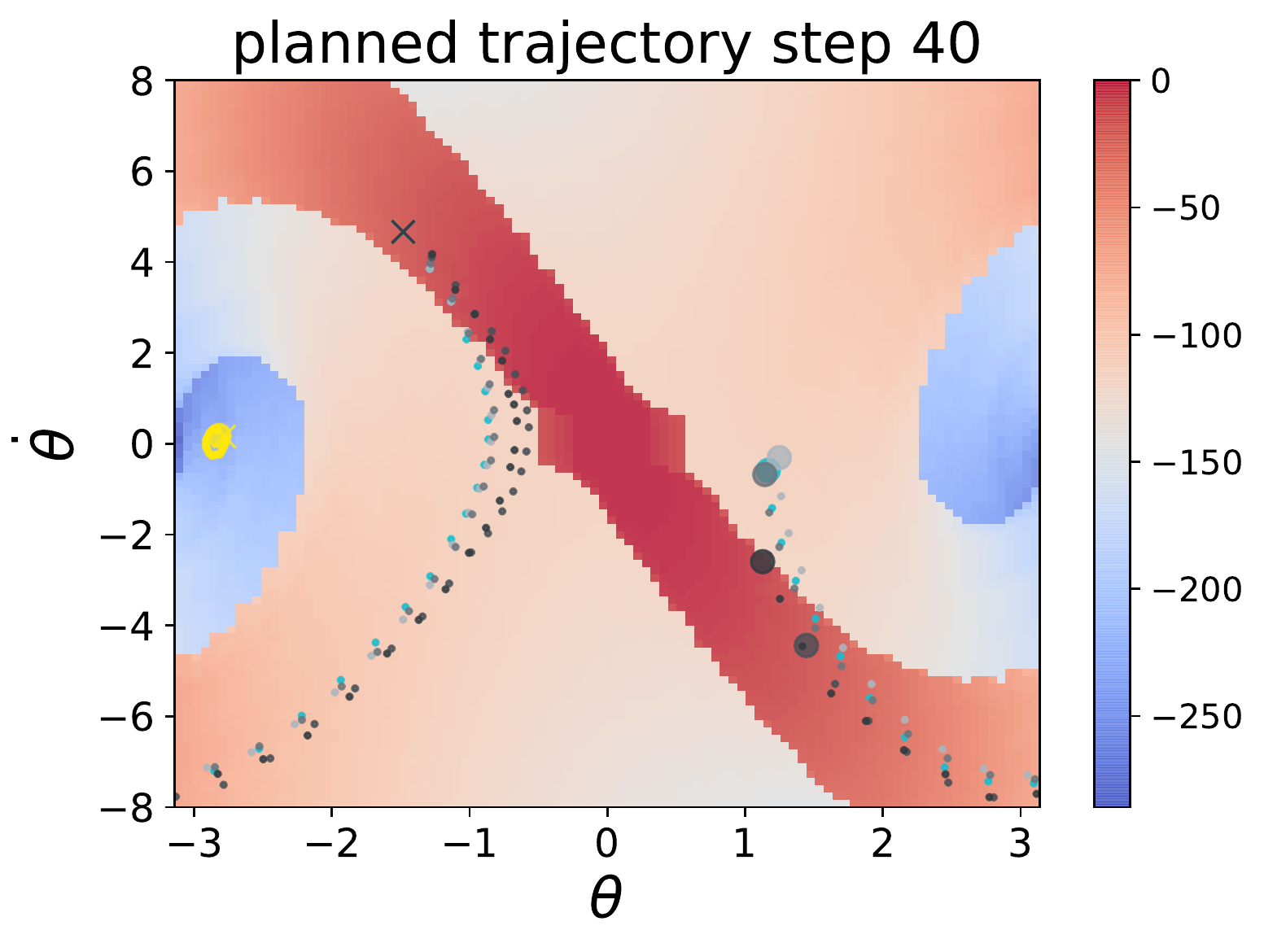}\\
    \includegraphics[width=0.5\textwidth]{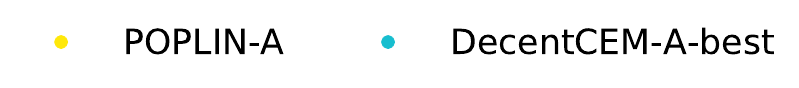}
    \caption{ Planning trajectories of DecentCEM-A (ensemble size 5) and POPLIN-A. It's overlaid on the heatmap showing the true optimal state value (red means higher value).
    Title of each plot shows the steps in the evaluation environment.
  Each planned trajectory has 30 steps, denoted by a sequence of dots. 
 The starting state is denoted by a cross \protect\markerstart    
\ and the ending (planned) state is denoted by a big solid dot \protect\markerelite\  (with different colors corresponding to different trajectories).
 The color for POPLIN-A and the best solution from DecentCEM-A is indicated in the legend in the plot and the rest of DecentCEM-A trajectories are colored with different shades of gray.
 Note that $\theta$ of $2\pi$ and $-2 \pi$ are identical in the angular position but appears as ``disconnected'' on the plots.
 }
  \label{fig:visplan}
\end{figure}

 To better understand the planning process of DecentCEM, we visualized the planning trajectories of POPLIN-A and DecentCEM-A in Fig.\ref{fig:visplan}. 
 The planned state trajectories are denoted by sequences of dots. Each plot show the planned trajectories at different steps from running both algorithms (at 2k training steps) on the same evaluation environment such that the comparison is fair.
 The heatmap shows the optimal state value 
  solved by value iteration on the discretized pendulum environment. 
  The discretization is performed by discretizing the state space and action space of the original pendulum environment into 100 and 50 intervals respectively.
The best trajectory from the multiple instances in DecentCEM-A is colored in cyan. Note that it is not ranked based on the true value, but on the imaginary value during planning. We could observe that this solution may not always be the true best solution among the trajectories due to the model inaccuracy at 2000 training steps. However, these trajectories are able to explore the space better than using a single instance as in POPLIN-A which can easily get stuck in the state regions with low values.
 
\end{document}